\documentclass[onecolumn, 12pt]{IEEEtran}
\usepackage{cite}

\usepackage{amsmath,amssymb,amsfonts}
\usepackage{graphicx}
\usepackage{textcomp}
\usepackage{algpseudocode}
\usepackage{algorithm}
\usepackage{mathtools}
\usepackage{subcaption}

\usepackage{soul}
\usepackage{comment}

\usepackage[utf8]{inputenc} %
\usepackage[T1]{fontenc}    %
\usepackage{hyperref}       %
\usepackage{url}            %
\usepackage{nicefrac}       %
\usepackage{microtype}      %
\usepackage{xcolor}         %

\newtheorem{lemma}{Lemma}

\newtheorem{definition}{Definition}

\newtheorem{example}{Example}
\newtheorem{approach}{Approach}
\DeclareMathOperator*{\argmax}{arg\,max} %

\newenvironment{proof}
{\textit{Proof:} }
{$\square$}

\usepackage{thm-restate}

\def\BibTeX{{\rm B\kern-.05em{\sc i\kern-.025em b}\kern-.08em
    T\kern-.1667em\lower.7ex\hbox{E}\kern-.125emX}}

\newcommand{\twocol}[1]{}

\newcommand{\revision}[1]{#1}
\newcommand{\secondrev}[1]{{#1}}
\newcommand{\by}{\mathbf{y}}
\newcommand{\bx}{\mathbf{x}}
\newcommand{\bw}{\mathbf{w}} 
\newcommand{\audit}{\mathcal{D}_{\text{audit}}}
\newcommand{\bv}{\mathbf{v}}

\newcommand{\EE}{\mathbb{E}}

\newcommand{\cS}{\mathcal{S}}
\newcommand{\cZ}{\mathcal{Z}}

\newcommand{\gp}{\mathcal{G}}
\newcommand{\fairness}{\texttt{F\_MaxGap}}
\newcommand{\genericfair}{\texttt{F}}
\newcommand{\weiupperbound}{O\left( \frac{ 2^{\frac{1}{2} H_{2/3}(w)} }{\epsilon^2(1 - \alpha)} +\frac{ 2^{\frac{1}{3} H_{2/3}(w)} }{\epsilon^4(1 - \alpha)^2} \right)}
\newcommand{\attupperbound}{O\left( \frac{1}{\epsilon^4 (1 - \alpha)^2}\right)}

\newcommand{\norm}[1]{ \left|\left| #1 \right|\right|}
\newcommand{\gtg}[1]{{#1}}
\newcommand{\averagegtg}[1]{ \overline{#1}(\bw) }
\newcommand{\cvarrange}{[0, 1)}

\newcommand{\hypothesis}{\text{H}}
\newcommand{\hypregions}{\mathcal{P}}

\AddToHook{cmd/appendix/before}{%
={\thesection.\arabic{equation}}%
  \setcounter{theorem}{0}%
  \setcounter{proposition}{0}%
  \setcounter{lemma}{0}%
  \setcounter{corollary}{0}%
}

\begin{document}
\title{Multi-Group Fairness Evaluation via Conditional Value-at-Risk Testing}
\author{Lucas Monteiro Paes, Ananda Theertha Suresh, Alex Beutel, Flavio P.  Calmon, and Ahmad Beirami 
\thanks{The work of Lucas Monteiro Paes and Flavio P. Calmon was supported in part by the National Science Foundation under grants CAREER 1845852, CIF 2312667, FAI 2040880, and also in part by a gift from Google.

Lucas Monteiro Paes and Flavio P. Calmon are with the School of Engineering
and Applied Sciences, Harvard University, Cambridge, MA 02134, USA
(e-mail: \texttt{lucaspaes@g.harvard.edu}; \texttt{flavio@seas.harvard.edu}).

Ananda Theertha Suresh and Ahmad Beirami are with Google Research, New York, NY 10011, USA (e-mail: \texttt{theertha@google.com}; \texttt{beirami@google.com}).

Alex Beutel contributed to this research while at Google Research, New York, NY 10011, USA. He is currently with OpenAI (e-mail: \texttt{alexb@openai.com})
. 
}
}

\maketitle

\begin{abstract}
    Machine learning (ML) models used in prediction and classification tasks may display performance disparities across population groups determined by sensitive attributes (e.g., race, sex, age). We consider the problem of evaluating the performance of a fixed ML model across population groups defined by multiple sensitive attributes (e.g., race \textbf{and} sex \textbf{and} age). Here, the sample complexity for estimating the worst-case performance gap across groups (e.g., the largest difference in error rates) increases exponentially with the number of group-denoting sensitive attributes. To address this issue, we propose an approach to test for performance disparities based on Conditional Value-at-Risk (CVaR). By allowing a small probabilistic slack on the groups over which a model has approximately equal performance, we show that the sample complexity required for discovering performance violations is reduced exponentially to be at most upper bounded by the square root of the number of groups. As a byproduct of our analysis, when the groups are weighted by a specific prior distribution, we show that R\'enyi entropy of order $2/3$ of the prior distribution captures the sample complexity of the proposed CVaR test algorithm. Finally, we also show that there exists a non-i.i.d. data collection strategy that results in a sample complexity independent of the number of groups.
\end{abstract}


\section{Introduction}
\label{sec:introduction}
Machine learning (ML) algorithms are increasingly used in domains of consequence such as hiring \cite{Roy_hiring}, lending \cite{Zhou_lending, Brotcke_lending}, policing \cite{Wang_crime}, and healthcare \cite{csilag_health, Dritsas_health}.
These algorithms may display performance disparities across groups determined by legally protected attributes \cite{LAW, Piano_ethics} (e.g., sex and race). %
Discovering and reducing performance disparities across population groups is a recognized desideratum in ML. Over the past decade, hundreds of fairness interventions have been proposed to close performance gaps in a range of prediction and classification tasks while preserving overall accuracy \cite{hort2022bias,lowy2022a, alghamdi2022beyond, baharlouei2019renyi, hardt2016equality, Agarwal18, Calmon17, zemel13_learning_fair}.
While there are many notions of fairness in ML~\cite{mehrabi2021survey},  we focus on {\em group fairness}~\cite{Faisal10, Dwork12, zemel2013learning, hardt2016equality}.
Group fairness is usually concerned with a small number (e.g., two) of pre-determined demographic groups $\gp$ and requires a certain performance metric to be approximately equal across the groups in $\gp$ --- different choices of performance metric lead to different notions of fairness.
For example, {\em equal opportunity} \cite{hardt2016equality}, a popular variation of group fairness,  requires that false positive rates (FPR) be approximately equal across groups in $\gp$. 
A widely accepted mathematical formalization of group fairness violation is the \emph{largest}  gap between a model's performance for a group in $\gp$  and its average performance across the entire population. We refer to such worst-case group fairness metrics as  \emph{max-gap fairness metrics}.

{\em Multi-group fairness} notions aim to increase the granularity of groups in $\gp$ for which model performance is measured and compared  \cite{kearns2018preventing, Johnson18-multicalibration, wang2022towards, foulds2020intersectional, kong2022intersectionally}. Multi-group fairness metrics answer recent calls for testing for performance disparities across groups determined by intersectional demographic attributes \cite{wang2022towards, foulds2020intersectional, kong2022intersectionally} (e.g., groups determined by sex \textbf{and} race instead of separate evaluation for each attribute). A fine-grained characterization of groups is critical for discovering performance disparities that may lead to systematic disadvantages
across intersecting dimensions \cite{foulds2020intersectional}. %

Though socially critical, the need for multi-group fairness and the discovery of max-gap fairness violations pose competing statistical objectives. Intuitively, as the number of groups increases (i.e., $|\gp|$ becomes large), so does the sample complexity of evaluating max-gap fairness metrics. In such cases, even testing for multi-group fairness is challenging \cite{monteiro22-epistemic, Johnson18-multicalibration, kearns2018preventing}. For instance, Kearns {\em et al.} \cite{kearns2018preventing} show that  a model can be reliably tested for multi-group fairness with a test dataset of size $O(|\gp|)$. Testing for multi-group fairness may be infeasible when $|\gp|$ is large, such as when $\gp$ is a product set of several group-denoting attributes, such as sex, age, and race.

In this paper, we derive fundamental performance limits for testing for multi-group fairness and propose a notion of multi-group fairness based on conditional value-at-risk (CVaR) \cite{rockafellar2000optimization} that allows practitioners to relax fairness guarantees to decrease the sample complexity for reliably testing for multi-group fairness --- we call it CVaR fairness.
CVaR is a popular metric that is a convex relaxation of the {\em maximum}, making it a suitable objective for optimization as well. As we shall see, CVaR fairness also leads to a much improved sample complexity as compared to the max-gap fairness allowing us to scale the number of groups in a tractable manner.
Remarkably, for non-i.i.d. datasets, we propose an algorithm for reliably testing for CVaR fairness violations with a number of samples independent of the number of demographic groups.

Formally, each individual has sensitive attributes denoted by $(s_1, ..., s_k)$ for some $k \in \mathbb{N}$, and $s_i \in \cS^i$ --- the combination of which defines a protected group.
Intersectionality argues that all the population groups defined by the combinations $(s_1, ..., s_k) \in \cS^1 \times ... \times \cS^k$ have unique sources of discrimination and privilege \cite{Hooks92-feminism}.
Therefore, it is important to discover and treat disparities in performance across groups $g \in \gp = \cS^1\times ... \times \cS^k$ %
\cite{foulds2020intersectional, Kearns19FAT, wang2022towards, Lett_translating}. %
In this case, the number of groups that we aim to ensure that the ML algorithm treats similarly is $|\gp| = |\cS^1| \cdot |\cS^2|... \cdot |\cS^k|$.
When each $\cS^i$ is at least binary for all $i \in [k]$, $\gp$ has an exponential number of groups $|\gp| \geq 2^k$.

Let $\bw = (w_1, ..., w_{|\gp|})$ be a stochastic vector that defines a distribution over groups. 
Broadly speaking, we can think of $\bw$ as a vector of fairness aware importance scores for every group $g \in \gp$; for example, $\bw$ can be given by the percentage of each group in the population, an importance weight for each group, %
or  uniform $\left(\text{i.e., } w_g = \frac{1}{|\gp|}\right)$ when assuming a uniform prior on the group distributions.
\revision{The theoretical contribution of this work (Theorem \ref{thm:weighted_upper_general}) makes a connection between the sample complexity for testing  multi-group fairness and the R\'enyi entropy of order~$2/3$, which is defined by $H_{2/3}(\bw) \triangleq 3 \log_2 \sum_g w_g^{2/3}$.}
Our \textbf{main contributions} are: 
\begin{itemize}
    \item We provide an impossibility result, showing that \emph{any hypothesis test} that uses i.i.d. samples from the data distribution needs at least $n = \Omega(|\gp|)$ points to reliably test if max-gap fairness metrics are greater than a fixed threshold --- Proposition \ref{prop:hyptest_lower_bound}. 
    Our result complements the achievability result in \cite{kearns2018preventing}, where it was shown that multi-group fairness can be tested with $n = O(|\gp|)$ samples.
    Hence,  scaling $\gp$ to intersectional groups determined by a combination of features can be infeasible due to the cost of acquiring an exponential number of samples.

    \item Motivated by the first negative result, we propose a multi-group fairness metric that is a relaxation for max-gap fairness metrics based on the conditional value-at-risk (CVaR) --- namely CVaR fairness in Definition \ref{def:CVaRFairness}.
    We show that it is possible to recover max-gap fairness metrics from CVaR fairness (Proposition \ref{prop:recover_fairness_knwing_g}) and that CVaR fairness lower bounds max-gap fairness metrics (Proposition \ref{prop:CVaRLowerBound}).

    \item {We propose two natural sampling techniques for data collection: \emph{weighted sampling} and \emph{attribute specific sampling} thus providing flexibility for the auditor to collect datasets.}

    \item We propose an algorithm to test if the CVaR fairness metric is greater than a fixed threshold $\epsilon$  (Algorithm \ref{alg:hypothesis_test}).
    To do so, we use an estimator for the CVaR fairness metric, use data samples to compute the estimator {based on either sampling technique}, and perform a threshold test that outputs if the CVaR fairness metric is bigger than the fixed threshold ($\epsilon$) or equal to zero. 

    \item 
    We prove that the proposed Algorithm \ref{alg:hypothesis_test}, with weighted sampling, can reliably test for CVaR fairness using at most $n = O\left(2^{\frac{1}{2}H_{2/3}(\bw)}\right)$ samples, where $H_{2/3}(\cdot)$ is the R\'enyi entropy of order $\frac{2}{3}$, which is formally defined in ~\eqref{eq:renyi} (Theorem \ref{thm:weighted_upper_general}). 
    Our result shows that using Algorithm \ref{alg:hypothesis_test} for the CVaR fairness metric allows us to decrease the sample complexity of testing, achieving a complexity of $2^{\frac{1}{2}H_{2/3}(\bw)} \leq \sqrt{|\gp|} < |\gp|$.

     \item {We also show a lower bound for the sample complexity for testing CVaR fairness under weighted sampling when $\bw$ is a uniform distribution over all groups.    
    We show that when $\bw$ is a uniform distribution over all groups we need
    $n = \Omega \left(\sqrt{|\gp|}\right)$ samples to reliably test for CVaR fairness (Theorem \ref{thm:converse_CVaR}). When $\bw$ is the uniform distribution, this bound matches the upper bound of the proposed algorithm under weighted sampling tehcnique.}

    \item Finally, we prove that if we have access to $2$ samples of any group $g \in \gp$ (we will not necessarily use them), it is possible to use Algorithm \ref{alg:hypothesis_test} with attribute specific sampling to reliably test for CVaR fairness utilizing a number of samples ($n$) that is independent of the number of groups $|\gp|$.
    Specifically, when testing if the CVaR fairness metric is bigger than a threshold $\epsilon$, we only need $O\left(\frac{1}{\epsilon^4}\right)$ samples to test reliably, making the number of samples to test for multi-group fairness independent of the number of groups.
\end{itemize}

The paper is organized as follows.
In Section \ref{sec:background}, we introduce the notation used in the paper and the hypothesis test framework for auditing a model for  multi-group fairness.
In Section \ref{sec:max-gap_fairness}, we define the max-gap fairness metric and provide a converse result for auditing a model for max-gap fairness --- this result indicates when auditing a model for  max-gap fairness is infeasible.
In Section \ref{sec:CVaR}, we define the conditional value-at-risk fairness metric and describe its properties.
In Section \ref{sec:Testiong_for_CVAR_fairness_main}, we define an algorithm along with sampling strategies for efficiently testing whether the CVaR fairness metric is greater than a threshold or is equal to zero.
In Section \ref{sec:ConverseCVaR} we show a converse result for auditing a model for CVaR fairness under certain assumptions, showing that the proposed testing algorithm achieves optimal order, and providing an operational meaning for the R\'enyi entropy of order $2/3$. 
\revision{Finally, in Section \ref{sec:experiments} we study the practical consequences of our theoretical contributions by (i) empirically analyze the impact of the R\'enyi entropy on the multi-group fairness auditing accuracy, (ii) computing the sample complexity for reliably auditing using CVaR fairness in comparison with max-gap fairness, and (iii) using our converse bounds to compute the maximun number of groups one model may use for multi-group fairness auditing to be feasible with reasonable performance.}

\section{Background \& Notation}
\label{sec:background}
We start by describing each individual as a tuple $(x, g, y)$ where $x \in \mathcal{X}$ is the attribute vector, $g \in \gp$ is the \emph{protected} group variable (e.g., $(\texttt{sex, race, age}) = (\texttt{male, asian, 20--30})$), 
and $y \in \{0, 1\}$ is a binary label to be predicted --- note that $x \in \mathcal{X}$ may contain $g \in \gp$.
Let $\widehat{y} = f_{\theta}(x) \in \{0, 1\}$ denote a model that aims to predict $y$ and where the model is parameterized by $\theta$, e.g., weights of a neural network.
We denote the data of each individual by $z = (x, g, y, \widehat{y}) \in \cZ$, when using a random variable instead of the data samples, we write $Z = (X, G, Y, \widehat{Y})$, and assume that the data is drawn from an unknown distribution $\widetilde{P}_Z$.
We sample from $\widetilde{P}_Z$ to create a model auditing dataset with $n \in \mathbb{N}$ samples $\audit = \{z_i\}_{i = 1}^n$ --- the data points are not necessarily i.i.d. sampled. Note that the samples $z = (x, g, y, \widehat{y})$ are realizations of the random variable $Z = (X, G, Y, \widehat{Y})$. 

We also define $\bw \triangleq (w_1, ..., w_{|\gp|}) \in \Delta^{|\gp|}$ to be a prior distribution on the groups of the audit dataset --- the model auditor has full control over $\bw$.
For example, when testing for fairness, the model auditor may choose $\bw$ to be given by a uniform prior, i.e., $w_g = \frac{1}{|\gp|}$.
This way, all groups would be equally represented in the audit dataset and would receive equal fairness guarantees. 
The model auditor may choose $\bw$ to increase the weight of such groups, reflecting some desired policy.

Group fairness notions such as equal opportunity~\cite{hardt2016equality}, equal odds~\cite{hardt2016equality}, and representation disparity~\cite{hashimoto2018fairness} are widely used in practice.
These notions quantify {\em fairness} in terms of differences (or lack thereof) of a given performance metric (e.g., FPR) across population groups. In general, group fairness notions can be represented in terms of a \emph{fairness violation metric}, which will be zero if there is no performance disparity across groups, and positive if some performance disparity exists (e.g., one group has a higher FPR than the rest).
We use fairness violation metrics to test if a given machine learning model follows a specific fairness notion.
Let $\genericfair$ be a fairness violation metric that takes the test dataset distribution $\widetilde{P}_Z$ as input and outputs a real number  --- $\genericfair$ will depend on the group distribution vector $\bw$.

We are interested in testing if most groups $g \in \gp$ are treated similarly accordingly to the metric (i.e., $\genericfair(\widetilde{P}_Z) = 0$) or if there exists a group that is being treated differently than the others (i.e., $\genericfair(\widetilde{P}_Z) \geq \epsilon$).
Hence, we propose to use a hypothesis test framework for testing if multi-group fairness metrics are bigger than a fixed threshold.
This is equivalent to the hypothesis test in Definition \ref{def:Hypothesis_Test}.

\begin{definition}[$\epsilon$-Test] Given a fairness violation metric $\genericfair$
for the dataset distribution {$\widetilde{P}_Z$ over $\cZ$}, we call the following hypothesis test an $\epsilon$-test.
\begin{equation}
  \begin{cases}
    \hypothesis_0:& \genericfair(\widetilde{P}_Z) = 0\\
    \hypothesis_1:& \genericfair(\widetilde{P}_Z) \geq \epsilon.
  \end{cases}
\end{equation}
We denote by {$\psi_{\epsilon}: \cZ^n \rightarrow \{0, 1\}$} the decision function that takes the dataset $ \audit = \{z_i\}_{0 \leq i \leq n}$ such that $z_i \in \cZ$ and predicts
\begin{equation}
    \psi_{\epsilon}\left( \audit \right) = 
    \begin{cases}
        0 & \text{if $\hypothesis_0$ is true}\\
        1 & \text{if $\hypothesis_1$ is true}.
    \end{cases}
\end{equation}
\label{def:Hypothesis_Test}
\end{definition}

The decision function in Definition \ref{def:Hypothesis_Test} aims to differentiate between distributions from two separate decision regions ($ \hypregions_0$ and $ \hypregions_1(\epsilon)$) in the set of all distributions in $\cZ$.
Specifically, the decision regions $ \hypregions_0$ and $ \hypregions_1(\epsilon)$ are given respectively by \eqref{eq:hyp_reg_0} and \eqref{eq:hyp_reg_1}. 
\begin{align}
    \label{eq:hyp_reg_0}
    \hypregions_0 &= \left\{\left. \widetilde{P}_Z \sim \cZ \right| \ \genericfair(\widetilde{P}_Z) = 0 \right\} \\
        \label{eq:hyp_reg_1}
    \hypregions_1(\epsilon) &= \left\{\left. \widetilde{P}_Z \sim \cZ \right| \ \genericfair(\widetilde{P}_Z) \geq \epsilon \right\}
\end{align}
where $\widetilde{P}_Z \sim \cZ$ denotes that $\widetilde{P}_Z$ is the dataset probability distribution with support on $\cZ$.
Assuming a uniform prior on the hypothesis $\hypothesis_0$ and $\hypothesis_1$, the decision function has a probability of error equal to:
\begin{equation}
    P_{\text{error}} = \frac{\Pr(\psi_{\epsilon} = 1 | \hypothesis_0) + \Pr(\psi_{\epsilon} =  0 | \hypothesis_1)}{2}.
    \label{eq:error_prob}
\end{equation}
Since we are distinguishing a single distribution from a set of distributions, the abovementioned test is a composite hypothesis test. Composite hypothesis tests have been recently popularized under \emph{distribution property testing}, see \cite{canonne2020survey} for a detailed survey on the topic.

In this section, we defined fairness violation metrics ($\genericfair$) and how to use them to audit a model for multi-group fairness.
In the next section, we study a particular case of fairness violation metrics that we call max-gap fairness (Definition \ref{def:multi_group_fairness}) and show that performing a reliable $\epsilon$-test for max-gap fairness require at least $O(|\gp|/\epsilon^2)$ samples.

\section{Max-Gap Fairness}
\label{sec:max-gap_fairness}

Several fairness violation metrics $\genericfair$ measure the magnitude of the difference between model performance for a group and the average performance across all the groups ~\cite{Johnson18-multicalibration, kearns2018preventing, hardt2016equality, hashimoto2018fairness, monteiro22-epistemic}.
We call these metrics the max-gap fairness and denote it by $\fairness$.
In this paper, we call the function that measures the performance across different demographic groups the \emph{quality of service} function and the average performance across all groups the \emph{average quality of service}.
We give examples of the average quality of service function for equal opportunity in Example \ref{eg:equal_opportunity} and statistical parity in Example \ref{eg:statistical_parity}. 
We use the quality of service function to formally define the max-gap fairness metrics as follows.

\begin{definition}[Max-Gap Fairness] 
\label{def:multi_group_fairness}
Let $L: \cZ \rightarrow \{0, 1\}$ be the quality of service function, $\bw \in \Delta^{|\gp|}$ be the group weight vector, $P$ be a distribution in $\cZ$ that measures the average quality of service, and define the average quality of service to be 
given by
\begin{equation}
    \averagegtg{L} \triangleq \sum_{g \in \gp} w_g \EE_P[L(Z) | G = g].
    \label{eq:globaltg}
\end{equation}

The fairness gap per group is denoted as $\Delta(g; P, L)$ and given by
\begin{equation}
    \Delta(g; P, L, \bw) \triangleq \left| \EE_P[L(Z) | G = g] - \averagegtg{L} \right|.
    \label{eq:garppergroup}
\end{equation}

The \emph{max-gap} fairness metric $\fairness$ is the maximum fairness gap per group $\Delta(g; P, L, \bw) $ for $g \in \gp$, i.e.,
\begin{align}
    \fairness(P, L, \bw) & \triangleq \max_{g \in \gp} \Delta(g; P, L, \bw).
    \label{eq:gapintersec}
\end{align}
\end{definition}
Note that when $L$, $P$, and $\bw$ are clear from the context we drop them and just write $\fairness \triangleq \fairness(P, L, \bw)$. 

We demonstrate next how group fairness metrics such as Equal Opportunity (EO) and Statistical Parity (SP) can be expressed in terms of the above max-gap fairness definition.

\begin{example}[Equal Opportunity]
Recall that the equal opportunity violation metric is given by 
\begin{align}
    \genericfair_{EO} &= \max_{g \in \gp} \left| \Pr(\widehat{Y} = 1 | Y = 0, G = g) - \Pr(\widehat{Y} = 1 |  Y = 0) \right| \\
    &= \max_{g \in \gp} \left| \Pr(\widehat{Y} = 1 | Y = 0, G = g) - \sum_{g \in \gp} w_g \Pr(\widehat{Y} = 1 | Y = 0, G = g)) \right|.
\end{align}

Note that we can recover $\genericfair_{EO}$ from the max-gap fairness definition by taking
\begin{align}
          L_{EO}(z) &=  \mathbb{I}_{\widehat{y} = 1}(z), \\
          P_{EO}(z) &= \widetilde{P}_Z(z | Y = 0).
\end{align}

We have that the fairness gap per group is given by
\begin{align}
         \Delta(g; P_{EO},  L_{EO}, \bw) = \left| \Pr(\widehat{Y} = 1 | Y = 0, G = g) - \sum_{g \in \gp} w_g \Pr(\widehat{Y} = 1 | Y = 0, G = g) \right|.
\end{align}

Hence, we recover the equal opportunity violation metric
\begin{align}
    \genericfair_{EO} &= \max_{g \in \gp} \left| \Pr(\widehat{Y} = 1 | Y = 0, G = g) - \sum_{g \in \gp} w_g \Pr(\widehat{Y} = 1 | Y = 0, G = g)) \right|\\
    &= \max_{g \in \gp}  \Delta(g; P_{EO},  L_{EO}, \bw) = \fairness(P_{EO},  L_{EO}, \bw).
\end{align}

\label{eg:equal_opportunity}
\end{example}

\begin{example}[Statistical Parity]
Recall that the statistical parity violation metric is given by 
\begin{align}
    \genericfair_{SP} &= \max_{g \in \gp} \left| \Pr(\widehat{Y} = 1 | G = g) - \Pr(\widehat{Y} = 1) \right| \\
    &= \max_{g \in \gp} \left| \Pr(\widehat{Y} = 1 | G = g) - \sum_{g \in \gp} w_g \Pr(\widehat{Y} = 1 | G = g)) \right|.
\end{align}

Note that we can recover $\genericfair_{SP}$ from the max-gap fairness definition by taking
\begin{align}
          L_{SP}(z) &= \mathbb{I}_{\widehat{y} = 1}(z), \\
          P_{SP}(z) &= \widetilde{P}_Z(z ).
\end{align}
We have that the fairness gap per group is given by
\begin{align}
         \Delta(g; P_{SP}(z),  L_{SP}, \bw) = \left| \Pr(\widehat{Y} = 1 | G = g) - \sum_{g \in \gp} w_g \Pr(\widehat{Y} = 1 | G = g)) \right|.
\end{align}
Hence, we recover the statistical parity violation metric
\begin{align}
    \genericfair_{SP} &= \max_{g \in \gp} \left| \Pr(\widehat{Y} = 1 | G = g) - \sum_{g \in \gp} w_g \Pr(\widehat{Y} = 1 | G = g)) \right|\\
    &= \max_{g \in \gp}  \Delta(g; P_{SP}(z),  L_{SP}, \bw) = \fairness(P_{SP}(z),  L_{SP}, \bw).
\end{align}
\label{eg:statistical_parity}
\end{example}
Using a similar process, we can recover other fairness metrics such as calibration \cite{Pleiss17}.

Since max-gap fairness measures the violation in fairness metrics, we aim to ensure that the gap $\fairness(P, L, \bw) \approx 0$.
Achieving this condition implies that the model $f_{\theta}$ is {\em treating} all groups $\gp$ in the population similarly per the chosen fairness notion.
However, in general, we don't have access to the distribution $P$; therefore, it is not possible to compute $\EE_P[L(Z) | G = g]$ for all $g \in \gp$ exactly. 
Hence, we can't exactly compute $\fairness(P, L, \bw)$ the fairness metric we are interested in.
Regardless, we can estimate $\fairness(P, L, \bw)$ using $n$ i.i.d. samples from $P$.
When we are interested in ensuring that the model is treating only two groups similarly, i.e., when $\gp = \{0, 1\}$, 
We can estimate $\fairness(P, L, \bw)$ using $n = O(1/\epsilon^2)$ i.i.d. samples from $P$ and then reliably test if $\fairness(P, L, \bw) \approx 0$ --- this result is an immediate application of Chebyshev inequality~\cite{Tchébychef1867}.
Also, in \cite[Theorem 2.11]{kearns2018preventing}, it was shown that using the same procedure we can estimate $\fairness(P, L, \bw)$ for multi-group fairness using $n = O(|\gp|/\epsilon^2)$ i.i.d. samples from $P$.

In this work, we are interested in the case where $|\gp|$ is large ---  for example, when it contains exponentially many combinations of sensitive attributes (protected groups), and intersectionality is considered.
Here, the fairness definitions of interest will be the maximum gap fairness as in \eqref{eq:gapintersec}, but $|\gp|$ can grow exponentially.
In this regime, estimating the fairness metrics given by $\fairness$ \eqref{eq:gapintersec} is challenging. 
It is necessary to estimate $\EE_P[L(Z) | G = g] $ for all $g \in \gp$; hence, if $|\gp|$ is sufficiently large, performing all these estimations is impractical \cite{kearns2018preventing, monteiro22-epistemic, Johnson18-multicalibration}. In fact, we show that even \emph{testing} if $\fairness$ is above a given threshold is challenging, requiring further assumptions on the hypothesis test procedure to reduce sample complexity.

Next, we show that the max-gap fairness metrics discussed in Definition \ref{def:multi_group_fairness} \emph{require} at least $\Omega(|\gp|/\epsilon^2)$ samples for the $\epsilon$-test to have a small probability of error.  
In Proposition \ref{prop:hyptest_lower_bound}, we use the fact that there exist two distributions $P_0 \in \hypregions_0$ and $P_1 \in \hypregions_1(\epsilon)$ that are close in the space of distributions in terms of  Hellinger distance \cite{Hellinger}.
Intuitively, the fact that the hypothesis regions are close indicates that it is difficult to distinguish between certain  distributions in them. 
In Proposition \ref{prop:hyptest_lower_bound}, we show that 
the probability of error of the $\epsilon$-test is bounded away from $0$.

\begin{restatable}{proposition}{HypTestLowerBound}
(Converse for Max-Gap Fairness).
For max-gap fairness metrics $\fairness$ (e.g., equal opportunity in Example \ref{eg:equal_opportunity} and statistical parity in Example \ref{eg:statistical_parity}) with quality of service function $L$ and measuring the average quality of service using the distribution $P$ such that $(L, P) \in \{(L_{EO}, P_{EO}), (L_{SP}, P_{SP})\}$.
If $\epsilon \in [0, 0.5]$ and $w_g = P(G = g ) = \frac{1}{|\gp|}$ for all $g \in \gp$, using $n$ i.i.d. samples from $P$ the minimum probability of error of performing an $\epsilon$-test for $\fairness$ in Definition \ref{def:Hypothesis_Test} is lower bounded by 
\begin{equation}
    2\inf_{\psi_{\epsilon}} P_{\text{error}} \geq 1 - \left[2\left(1 - \left(1 - \frac{2\epsilon^2}{|\gp|}\right)^{n}\right)\right]^{1/2}.
\end{equation}
Furthermore, it is necessary to have access to $n$ given in \eqref{eq:max_gap_lower_bd} i.i.d. samples from $P$ to test if $\fairness = 0$ or $\fairness \geq \epsilon$ correctly with probability $0.99$.
\begin{equation}
    n = \Omega \left( \frac{|\gp|}{\epsilon^2}\right).
    \label{eq:max_gap_lower_bd}
\end{equation}
\label{prop:hyptest_lower_bound}
\end{restatable}

We provide the proof for this result in Appendix \ref{sec:lower_bd_max_gap}.

In Proposition \ref{prop:hyptest_lower_bound}, we have shown that the probability of error in the hypothesis test is lower bounded by a function of $n$ the number of samples from $P_{|Y = 0}$, $|\gp|$ the number of groups, and $\epsilon$ and 
it is necessary to have at least  $n = \Omega \left( {|\gp|}/{\epsilon^2}\right)$ samples to audit a model for max-gap fairness. Thus, a hypothesis test approach cannot decrease the sample complexity of auditing a model for max-gap fairness.
Next, we define a fairness metric based on the conditional value-at-risk, aiming to decrease the sample complexity of testing for multi-group fairness, albeit with a weaker notion of fairness violation guarantee.

\section{CVaR Multi-Group Fairness Metrics}
\label{sec:CVaR}
In this section, we define the conditional value-at-risk (CVaR) 
fairness and prove the properties of this multi-group fairness metric.

\secondrev{We adopt the CVaR fairness definition of~\cite{williamson2019fairness}, which is motivated by conditional value-at-risk \cite{rockafellar2000optimization}. Similarly to~\cite{williamson2019fairness}, we are concerned with applying CVaR to the fairness gap $\Delta(g)$. 
Following~\cite{williamson2019fairness}, Soen {\em et al.}~\cite{Soen2022} developed a post-processing technique to decrease CVaR fairness.
Meng and Gower~\cite{Meng2023} applied CVaR to the model loss with the objective of improving worst-case model performance -- not applying it to produce a fairer model.
Differently from \cite{williamson2019fairness, Soen2022, Meng2023}, our main contribution is the definition of CVaR fairness as a relaxation of the max-gap fairness metrics and applying CVaR to quantities beyond the model loss.
Additionally, we provide a sample-efficient test for CVaR, which is the main algorithmic contribution of this paper.
}

The CVaR fairness is an extension of the usual definition of fairness given in Definition \ref{def:multi_group_fairness}, i.e., we will show that it is possible to recover Definition \ref{def:multi_group_fairness} from the CVaR fairness in Definition \ref{def:CVaRFairness}.

The definition of CVaR for a random variable $W$ starts by computing $q_{\alpha}(W)$ the $\alpha$-quantile of $W$ for $\alpha \in \cvarrange$.
The $\alpha$-quantile of a random variable $W$ is defined to be a number $q_{\alpha}(W)$ such that
\begin{align}
    \Pr[W \geq q_{\alpha}(W)] &\geq 1- \alpha\\
    \Pr[W < q_{\alpha}(W)] &\geq \alpha.
\end{align}
Then, the conditional value-at-risk at level $\alpha$ is denoted by $\text{CVaR}_{\alpha}(W)$ and given by \eqref{eq:CVaROG} -- CVaR was proposed by \cite{rockafellar2000optimization}:
\begin{equation}
    \text{CVaR}_{\alpha}(W) \triangleq \EE[W | W \geq q_{\alpha}(W)].
    \label{eq:CVaROG}
\end{equation}
Intuitively, $\text{CVaR}_{\alpha}(W)$ measures the tail behavior of $W$.
For example, when $\alpha = 1$ the conditional value-at-risk becomes the maximum over all possible values of $W$ an event that may have vanishing probability.
The case of interest in multi-group fairness (when testing is hard) is when only a small group is treated differently from all others, i.e., when $g \in \gp$ is small ($w_g$ is small) but $\Delta(g) \geq \epsilon$. 
Hence, we are interested in the tail behavior of $\Delta(g)$.
For this reason, we next define CVaR fairness to capture the desired tail behavior. 

\begin{definition}[CVaR Fairness] Let $L: \cZ \rightarrow \{0, 1\}$ be the quality of service function, $\bw$ be the group importance weight vector defined by the auditor, and $P$ be the distribution that measures the average quality of service $\averagegtg{L}$. 
Recall that 
\begin{equation*}
    \Delta(g) = \left| \EE_P[L(Z) | G = g] - \averagegtg{L} \right|.
\end{equation*}

For all $\alpha \in \cvarrange$, 
we denote CVaR fairness as $\texttt{F\_CVaR}_{\alpha}(\bw; P, L, \averagegtg{L})$ and define it by 
\begin{align}
    Q_{\alpha} &\triangleq \argmax_{\substack{\sum_{g \in Q} w_g \leq 1 - \alpha \\ Q \subset \gp }} \sum_{g \in Q} w_g \Delta(g)\\
    \texttt{F\_CVaR}_{\alpha}(\bw; P, L, \averagegtg{L}) &\triangleq \frac{1}{1 - \alpha}  \sum_{g \in Q_{\alpha}} w_g \Delta(g).
    \label{eq:CVaRDefinition}
\end{align}

We can also write $\texttt{F\_CVaR}_{\alpha}(\bw; P, L, \averagegtg{L})$ as
\begin{equation*}
    \texttt{F\_CVaR}_{\alpha}(\bw; P, L, \averagegtg{L}) = \frac{1}{1 - \alpha}  \max_{\substack{\sum_{g \in Q} w_g \leq 1 - \alpha \\  Q \subset \gp}} \sum_{g \in Q} w_g \Delta(g).
\end{equation*}
\label{def:CVaRFairness}
\end{definition}
When $P, L, \averagegtg{L}$ are clear from the context we denote $\texttt{F\_CVaR}_{\alpha}(\bw; P, L, \averagegtg{L})$ by $\texttt{F\_CVaR}_{\alpha}(\bw)$. \revision{Note that when the group importance weights $\bw$ are uniform, then the set $Q_{\alpha}$ contains $(1 - \alpha) |\gp|$ groups; hence, the number of groups decreases linearly with alpha, and for $\alpha = (1-\frac{1}{|\gp|})$, the set only contains one of the groups with the highest $\Delta(g),$ thus CVaR recovers max-gap fairness in this case.}

\revision{
Connecting CVaR fairness back to the classical definition of CVaR \cite{rockafellar2000optimization}, it can be considered as an application of the classical CVaR for $\Delta(g)$.
 We claim that, under mild assumptions\footnote{\secondrev{The mild assumptions are (1) $\sum_{g \in Q_{\alpha}} w_g = 1 - \alpha$ and (2) if $g \neq g'$ then $\Delta(g) \neq \Delta(g')$. Lemma \ref{apx_cvar_fcvar} shows this result.}}, 
\begin{equation*}
     \texttt{F\_CVaR}_{\alpha}(\bw) = \EE[\Delta(g) | \Delta(g) \geq \min_{g \in Q_{\alpha}} \Delta(g)] = \EE[\Delta(g) | \Delta(g) \geq q_{\alpha}(\Delta(g))] = \text{CVaR}_{\alpha}(\Delta(g)), 
\end{equation*}
CVaR fairness captures the behavior of $\Delta(g)$ while accounting for the influence of the group weight on the tail.
}

Our definition for the CVaR fairness is more lenient than the usual multi-group fairness in Definition \ref{def:multi_group_fairness}.
Particularly, CVaR fairness gives stronger fairness guarantees when $\alpha$ is close to $1$.
However, it gives practitioners the flexibility of allowing $\alpha$ to be smaller and make $\texttt{F\_CVaR}_{\alpha}(\bw)$ easier to estimate --- i.e., decrease the sample complexity to perform an $\epsilon$-test for $\texttt{F\_CVaR}_{\alpha}(\bw)$.

Note that when $\alpha = 0$,  $\texttt{F\_CVaR}_{0}(\bw)$ corresponds to the average fairness gap of all groups. In this case, we can estimate  $\texttt{F\_CVaR}_{0}(\bw)$
with precision $\epsilon$ using $n = O\left(\frac{1}{\epsilon^2}\right)$ samples from $P$. Thus, it is possible to perform the $\epsilon$-test reliably  using $O\left(\frac{1}{\epsilon^2}\right)$ samples from $P$ 
--- we conclude that by a simple application of Chebyshev's inequality \cite{Tchébychef1867}. 
\revision{Additionally, $\alpha$ should be chosen to be a fixed constant (i.e., independent of $\gp$ and $\bw$), we next exemplify the impact of $\alpha$ in the sample complexity of fairness auditing and in Section \ref{sec::testing} we show that, tuning $\alpha$, it is possible to get a more favorable sample complexity for reliable fairness auditing ($\epsilon$-test for CVaR fairness). 
}

As a middle ground example between easy testing ($\alpha = 0$) and the max-gap fairness in Definition \ref{def:multi_group_fairness}, consider the following example \revision{where the parameter $\alpha$ in CVaR fairness allows trading stronger fairness guarantees (e.g., all groups, including the ones with almost zero probability of occurrence are equally treated) for better sample complexity by increasing the probability mass of the statistics being estimated --- i.e., $\max_{g} \Delta(g)$ may occur in a demographic group with vanishing probability while Example \ref{eg:equal_CVaR} illustrate how CVaR is estimated over a set of mass $(1 - \alpha)$}.

\begin{example}[CVaR Equal Opportunity] Let $\alpha = 0.75$, the quality of service function be $L_{EO}$, and $P_{EO}$ be given by the equal opportunity fairness metric in Example \ref{eg:equal_opportunity}.
The fairness gap per group is given by
\begin{equation}
    \Delta(g) = \left| \Pr(\hat{Y}= 1 | Y = 0, G = g) - \sum_{g \in \gp} w_g \Pr(\widehat{Y} = 1 | Y = 0, G = g)) \right|.
\end{equation}

Therefore, the CVaR Equal Opportunity is given by
\begin{equation}
    \texttt{F\_CVaR}_{0.75}(\bw) = \frac{1}{0.25}  \max_{\substack{\sum_{g \in Q} w_g \leq 0.25 \\ Q \subset \gp}}\sum_{g \in Q} w_g \Delta(g).
\end{equation}
Therefore, we need to estimate a quantity with probability mass $0.25$, in contrast with Example~\ref{eg:equal_opportunity}, where estimating a quantity with potentially vanishing probability was necessary.
\revision{
This example highlights the advantage of CVaR fairness in comparison with max-gap fairness, CVaR fairness gives flexibility for the model auditor to select the minimum probability mass ($1 - \alpha$) they can reliably estimate using their limited number of samples.
}
\label{eg:equal_CVaR}
\end{example}

In Proposition \ref{prop:recover_fairness_knwing_g}, we show that we can recover max-gap fairness (Definition \ref{def:multi_group_fairness}) by choosing an appropriate value for $\alpha$ --- max-gap fairness is hard to estimate and may demand $O\left(|\gp|\right)$ samples from $P$ to perform an $\epsilon$-test reliably as saw in Proposition \ref{prop:hyptest_lower_bound}.

\begin{restatable}{proposition}{RecoverFairnessKnwingG}($\texttt{F\_CVaR}_{\alpha}$ recovers $\fairness$). For all $\bw \in \Delta^{|\gp|}$, $L: \cZ \rightarrow \{0, 1\}$, and any distribution $P$ with support on $\cZ$, there exists $\alpha^* \in \cvarrange$ such that 
\begin{equation}
    \fairness(P, L, \averagegtg{L}) = \texttt{F\_CVaR}_{\alpha^*}(\bw; P, L, \averagegtg{L}).
\end{equation}
When $\bw = \left( \frac{1}{|\gp|}, ..., \frac{1}{|\gp|}\right)$ we have that $\alpha^* = 1 - \frac{1}{|\gp|}$.
\label{prop:recover_fairness_knwing_g}
\end{restatable}
We show this result in Appendix \ref{sec:properties_CVaR}.
Proposition \ref{prop:recover_fairness_knwing_g} indicates that by choosing a suitable $\alpha \in \cvarrange$ we can obtain a reasonable trade-off between sample complexity and fairness guarantees.
In Sections \ref{sec::testing}, we show how to leverage the flexibility provided by $\alpha$ to efficiently test for multi-group fairness. 

For all choices of $\alpha \in \cvarrange$ we prove that $\fairness(P, L, \averagegtg{L})$ is lower bounded by $\texttt{F\_CVaR}_{\alpha}(\bw)$.
This shows why CVaR fairness is a lenient version of the standard multi-group fairness in Definition \ref{def:multi_group_fairness}.
In Proposition \ref{prop:CVaRLowerBound}, we show that  $\texttt{F\_CVaR}_{\alpha}(\bw)$ lower bounds $\fairness(P, L, \averagegtg{L})$.

\begin{restatable}{proposition}{CVarLowerbound}($\texttt{F\_CVaR}_{\alpha}$ is a lower bound for $\fairness$). For all $L: \cZ \rightarrow \{0, 1\}$, distribution $P$ with support on $\cZ$ and distribution $P$ with support on $\cZ$, $\bw \in \Delta^{|\gp|}$, and $\alpha \in \cvarrange$ we have that
\begin{equation}
    0 \leq  \texttt{F\_CVaR}_{\alpha}(\bw; P, L, \averagegtg{L}) \leq \fairness(P, L, \averagegtg{L}) \leq 1.
\end{equation}
\label{prop:CVaRLowerBound}
\end{restatable}
We provide a proof for this result in Appendix \ref{sec:properties_CVaR}.

Proposition \ref{prop:CVaRLowerBound} shows that  $\fairness(P, L, \averagegtg{L}) \geq \epsilon$ does not necessarily imply $\texttt{F\_CVaR}_{\alpha}(\bw) \geq \epsilon$.
For this reason, performing an $\epsilon$-test using $\texttt{F\_CVaR}_{\alpha}(\bw)$ is more lenient than performing the same test for $\fairness(L, \averagegtg{L})$.
However, Proposition \ref{prop:CVaRLowerBound} ensures that if $\texttt{F\_CVaR}_{\alpha}(\bw) \geq \epsilon$ then $\fairness(L, \averagegtg{L}) \geq \epsilon$ which is what we will  test in Section \ref{sec::testing}. 

In this section, we have defined the CVaR fairness and showed that (i) it is an extension of the standard multi-group fairness metrics given in Definition \ref{def:multi_group_fairness} (Proposition \ref{prop:recover_fairness_knwing_g}), (ii) CVaR fairness is a more lenient metric than the max-gap fairness metrics in Definition \ref{def:multi_group_fairness}, and  (iii) CVaR fairness is a value in the square $[0, 1]$ as max-gap fairness (Proposition \ref{prop:CVaRLowerBound}).
Moreover, we exemplified that by carefully choosing $\alpha$ we can leverage the trade-off between sample complexity to perform a $\epsilon$-testing and fairness -- Example \ref{eg:equal_CVaR} and Proposition \ref{prop:CVaRLowerBound}. 

In the next section, we propose a testing algorithm that leverages the trade-off between sample complexity and fairness guarantees from CVaR fairness to efficiently perform an $\epsilon$-test.

\section{Testing for CVaR Fairness}
\label{sec:Testiong_for_CVAR_fairness_main}
\subsection{Data collection process}
In order to test for  $\texttt{F\_CVaR}_{\alpha}(\bw)$, we need to collect an audit dataset $\audit$. 
We propose two ways of collecting this dataset --- weighted sampling (Definition \ref{def:weightedsampling}) and attribute specific sampling (Definition \ref{def:attributesampling}).
We denote $M_g$ as the number of samples from group $g$ in the audit dataset. In general, $M_g$ can be a random variable. 

Due to the cost of acquiring new samples, there is a limit on the dataset size one can collect.
For this reason, next, we define the notion of a data budget that a practitioner may have when testing for CVaR fairness.

\begin{definition}[Data Budget] A model is audited with a data budget equal to $n \in \mathbb{N}$ when the expected number of samples used to audit the model is equal to $n$ --- i.e., \eqref{eq:databudget} holds.
\begin{equation}
 n \triangleq   \sum_{g \in \gp} \EE[M_g].
    \label{eq:databudget}
\end{equation}
\label{def:databudget}
\end{definition}

We now propose two ways of collecting data. 
We start with the weighted sampling, when the audit dataset $\audit$ is constructed using i.i.d. samples.

\begin{definition}[Weighted Sampling] Let $M_g$ be the number of samples from group $g \in \gp$.
Under a data budget $n$, we define the number of samples per group as
\begin{equation*}
    M_g \triangleq \text{Bin}(n, v_g) \ \ \forall g \in \gp,
\end{equation*}
where $\text{Bin}(n, v_g)$ represents the Binomial distribution with parameters $n \in \mathbb{N}$.
Moreover, $\bv = (v_1, ..., v_{|\gp|}) \in \Delta^{|\gp|}$ is the group marginal distribution from the audit dataset $\audit$, i.e., $v_g = P_Z(G = g)$, and $ \sum_{g \in \gp} \EE[M_g] = n$.
\label{def:weightedsampling}
\end{definition}

Note that the group marginal $v_g$ may be different from $\bw$ (the underlying or prior group distribution).
To provide additional flexibility for model auditor, we set 
\begin{equation}
v_g = \frac{w^{\eta}_g}{\sum_{g'} w^{\eta}_{g'} },
\end{equation}
where $\eta \in [0, \infty)$. Of these, two special cases are $\eta = 0$  (uniform distribution) and $\eta = 1$ (when  $v_g = w_g$).
\revision{In practice, the model auditor selecting the group marginals $\bv$ means that they can control the frequency at which a sample from a group is acquired.}

The model auditor may also decide to collect the data without ensuring that samples are i.i.d. from some distribution to optimize the sample complexity of performing an $\epsilon$-test. Next, we define attribute specific sampling, a non-i.i.d. sampling strategy that will allow the model auditor to achieve a sample complexity that is independent of the number of groups $|\gp|$ --- Corollary \ref{cor:samplecomplexityalg_att}.

\begin{definition}[Attribute Specific Sampling] Let $M_g$ be the number of samples from group $g \in \gp$, and $\bw = (w_1, ..., w_{|\gp|})$ be the group weight vector.
Under a data budget $n$, we define the number of samples per group as
\begin{equation*}
    M_g \triangleq \text{Ber}(\min\{\gamma w_g, 1\}) \frac{n}{\gamma} \ \ \forall g \in \gp.
\end{equation*}
Where $\text{Ber}(\min\{\gamma w_g, 1\})$ represents the Bernoulli distribution with parameter $\min\{\gamma w_g, 1\}$,  $\gamma \in [0, + \infty)$, and $ \sum_{g \in \gp} \EE[M_g] = n$.
\label{def:attributesampling}
\end{definition}

In the next section we propose an estimator $\widehat{F}(\bw)$ that we will use to audit the model for CVaR fairness.

\subsection{Proposed estimator}
\label{sec::Approximating}

Estimating the value of the CVaR fairness metric ($\texttt{F\_CVaR}_{\alpha}$) may still be difficult and require a high sample complexity.
However, in this work, we are only interested in performing an $\epsilon$-test, i.e., testing if the CVaR fairness metric is bigger than $\epsilon$ or equal to zero.
For this reason, instead of defining an estimator that tries to approximate $\texttt{F\_CVaR}_{\alpha}$, we define it with the sole objective of testing if $\texttt{F\_CVaR}_{\alpha} > \epsilon$.
\revision{To do so, we introduce the quantities $\widehat{F}_1$ and $\widehat{F}_2$ in Definition \ref{def:estimator} that we show are unbiased estimators for $\sum_{g \in \gp} w_g\EE[L(Z) | G = g]^2$ and $\averagegtg{L}$ respectively in Lemma~\ref{lem::apx_avg-first-order} and Lemma \ref{lem::apx_avg-second-order}.
Then, we use these estimators to approximate an upper bound for the CVaR fairness that we show in Lemma \ref{lem:apx_CVaRUpperBound}, as follows
\begin{equation*}
    \widehat{F} \approx \sum_{g \in \gp} w_g\EE[L(Z) | G = g]^2 - \averagegtg{L}^2 \geq (1 - \alpha)(\texttt{F\_CVaR}_{\alpha}(\bw))^2.
\end{equation*}

}
Definition \ref{def:estimator} formalizes the definition of the estimator for testing $\texttt{F\_CVaR}_{\alpha}$.

\begin{definition}[Estimator for testing $\texttt{F\_CVaR}_{\alpha}$] Let $\texttt{F\_CVaR}_{\alpha}$ be the CVaR fairness metric that uses the quality of service function $L: \cZ \rightarrow \{0, 1\}$, the average quality of service $\averagegtg{L}$, and uses the distribution $P$ for $\cZ$.
Define $M_g$ to be a random variable in $\mathbb{N}$, and assume we have access to $M_g$ samples from $P(Z| G = g)$ for each group $g \in \gp$ --- note that some $M_g$ may be equal to $0$.
If $M_g = 0,$ we define:
\begin{equation*}
    \frac{\sum_{j = 1}^{M_g} L(z_j)}{M_{ {g}}} \triangleq 0,
    \label{eq:first-order-estimator}
\end{equation*}
and if $M_g \in \{0, 1\}$ we define
\begin{equation*}
    \frac{\sum_{i = 1}^{M_g} L(z_i)}{M_g}\frac{(\sum_{i = 1}^{M_g} L(z_i)-1)}{M_g-1} \triangleq 0.
    \label{eq:second-order-estimator}
\end{equation*}
Then, we define the estimators $\widehat{F}_1$ and $\widehat{F}_2$ as follows: 
 \begin{align*}
    \widehat{F}_1(\bw) &\triangleq \sum_{ {g} \in \gp} \frac{w_g}{P\left[M_{ {g}} \geq 2\right]} \frac{\sum_{i = 1}^{M_g} L(z_i)}{M_g}\frac{(\sum_{i = 1}^{M_g} L(z_i)-1)}{M_g-1} ,\\
    \widehat{F}_2(\bw) &\triangleq  \sum_{ {g} \in \gp} \frac{w_g}{P\left[M_{ {g}} \geq 1\right]} \frac{\sum_{j = 1}^{M_g} L(z_j)}{M_{ {g}}}.
\end{align*}
Finally, our CVaR estimator for $\texttt{F\_CVaR}_{\alpha}$ is defined as:
\begin{equation}
    \widehat{F}(\bw) \triangleq \widehat{F}_1(\bw) - \left( \widehat{F}_2(\bw)\right)^2.
\end{equation}
\label{def:estimator}
\end{definition}
In Appendix \ref{sec:upper_bound_CVaR} we derive some of the properties, such as the mean and the variance, of the proposed estimator $\widehat{F}(\bw)$.

The estimator in Definition \ref{def:estimator} gives practitioners flexibility by allowing them to sample the number of data points $M_g$ they will use from each group $g \in \gp$.
Moreover, allowing the number of samples for a group to be zero ($M_g = 0$) may allow a better sample complexity than $O(|\gp|)$ --- which is why we use this strategy.
In the next section we use the proposed estimator along with the introduced sampling strategies to efficiently perform the $\epsilon$-test for the CVaR fairness metric.

In the next section, we propose a specific decision function $\psi_{\epsilon}$ that uses the auditing dataset $\audit$ to compute the estimator in Definition \ref{def:estimator} and output if the model is CVaR multi-group fair.
\subsection{An Efficient Test For Multi-Group Fairness}
\label{sec::testing}

In this section, we provide an algorithm to perform the $\epsilon$-test for the CVaR fairness metric.
We show that our algorithm can perform $\epsilon$-test reliably with at most $n = \weiupperbound$ samples when using weight sampling --- where $H_{2/3}$ denotes the R\'enyi entropy of order $2/3$.
Additionally, we show that when using attribute specific sampling, we can perform $\epsilon$-test for CVaR fairness reliably with at most $n = \attupperbound$.

We propose Algorithm \ref{alg:hypothesis_test} to perform an $\epsilon$-test by using the estimator for the CVaR fairness metric proposed in Section \ref{sec::Approximating}.
The algorithm takes the group importance weights $\bw$, the group sampling distributions $\{Q_g \}_{g \in \gp}$ (e.g., attribute specific sampling or weighted sampling), and returns which hypothesis is true (i.e., $\hypothesis_0$ is true or $\hypothesis_1$ is true).

The proposed algorithm starts by defining the number of samples per group $M_g$ using attribute specific or weighted sampling.
When the dataset is fixed, $M_g$ is a prefixed constant. However, if the model auditor has control of the dataset construction, they can use weighted sampling (Definition \ref{def:weightedsampling}) or attribute specific sampling (Definition \ref{def:attributesampling}) to define $M_g$.
Next, it computes the CVaR fairness metric estimator $\widehat{F}(\bw)$ in Definition \ref{def:estimator}.
Finally, it performs a threshold test comparing $\widehat{F}(\bw)$ with $\frac{(1 - \alpha) \epsilon^2}{2}$.
\revision{We compare $\widehat{F}(\bw)$ with the threshold $\frac{(1 - \alpha) \epsilon^2}{2}$ because in Lemma \ref{lem:apx_CVaRUpperBound} we show that $
    \widehat{F} \approx \sum_{g \in \gp} w_g\EE[L(Z) | G = g]^2 - \averagegtg{L}^2 \geq (1 - \alpha)(\texttt{F\_CVaR}_{\alpha}(\bw))^2.
$
}
We formalize this process in Algorithm \ref{alg:hypothesis_test}.

\begin{algorithm}[t]
\caption{CVaR Test}
\label{alg:hypothesis_test}
\begin{algorithmic}
\State \textbf{Input:} \revision{$\bw \in \Delta^{|\gp|}$,  $\{Q_g\}_{g \in \gp}$}
\State $M_g \sim Q_g$ \Comment{\revision{Sample $M_g$ s.t. $\Pr(M_g = k) = Q_g(k)$ as in Def. \ref{def:weightedsampling}, Def. \ref{def:attributesampling}, and Thm. \ref{thm:weighted_upper_general}. }}
\State Sample $M_g$ i.i.d. points from $P_{(\bx, \by)|G = g} \ \ \forall g \in \gp$ \Comment{\revision{Sample $M_g$ from $P_{(\bx, \by)|G = g}$ for each group.}}
\State Compute $\widehat{F}(\bw)$ as in Def. \ref{def:estimator} \Comment{\revision{Compute using the samples generate in previous step}}
\If{$\widehat{F}(\bw) \geq \frac{(1 - \alpha)\epsilon^2}{2}$}
\State \Return $\texttt{F\_CVaR}_{\alpha}(\bw) \geq \epsilon$
\ElsIf{$\widehat{F}(\bw) < \frac{(1 - \alpha)\epsilon^2}{2}$}
\State \Return $\texttt{F\_CVaR}_{\alpha}(\bw) = 0$
\EndIf
\end{algorithmic}
\end{algorithm} 

We show that the sample complexity for an $\epsilon$-test to differentiate between $ \texttt{F\_CVaR}_{\alpha}(\bw) = 0$ and $ \texttt{F\_CVaR}_{\alpha}(\bw) > \epsilon$ using Algorithm \ref{alg:hypothesis_test} is {more favorable} than $O(|\gp|/\epsilon^2)$ --- differently than the standard multi-group fairness metrics in Definition \ref{def:multi_group_fairness} which needs at least $\Omega(|\gp|/\epsilon^2)$ samples to perform a reliable $\epsilon$-test.

We first prove that the probability of error in \eqref{eq:error_prob} is bounded when using Algorithm \ref{alg:hypothesis_test} as the decision function.
Theorem \ref{thm:weighted_upper_new} shows that the probability of error is bounded when using weighted sampling, and it depends on the test data group distribution $\bv = (v_1, ..., v_{|\gp|})$.

\begin{restatable}{theorem}{WeightedUpperGeneral}(Achievability with weighted sampling).
\label{thm:weighted_upper_general}
Let $\bw \in \Delta^{|\gp|}$ be the group weight vector \revision{such that $\max_{g} w_{g} \leq 1 - \alpha$} , 
$L: \cZ \rightarrow \{0, 1\}$ be the quality of service function, and $\averagegtg{L}$ is the average quality of service.
Suppose we have access to $n$ i.i.d. samples using, i.e., we used weighted sampling in Algorithm \ref{alg:hypothesis_test}. Then,
\begin{equation*}
    Q_g(k) = \Pr(\text{Bin}(n, v_g) = k),
\end{equation*}
where $\bv = (v_1, ..., v_{|\gp|})$ is the group marginal from $P_Z$ that we sampled i.i.d. to generate $\audit$.

The probability of error ($P_{\text{error}}$ in \eqref{eq:error_prob}) for Algorithm \ref{alg:hypothesis_test} to differentiate $ \texttt{F\_CVaR}_{\alpha}(\bw) = 0$ from $ \texttt{F\_CVaR}_{\alpha}(\bw) > \epsilon$ is such that
\begin{equation}
     P_{\text{error}}  =  O\left( \frac{1}{(1 - \alpha)^2 \epsilon^4 n^2}\sum_g \frac{w_g^{2}}{v_g^2} + \frac{1}{n (1 - \alpha)^2 \epsilon^4}\sum_g \frac{w_g^2}{v_g}  \right).
\end{equation}
\label{thm:weighted_upper_new}
\end{restatable}
We show this result in Appendix \ref{sec:upper_bound_CVaR}.  Next, we show how large $n$ needs to be to ensure that the upper bound in Theorem \ref{thm:weighted_upper_new} is smaller than a fixed $\delta > 0$.
This result gives an upper bound for the sample complexity of testing for CVaR fairness using Algorithm \ref{alg:hypothesis_test}.

Recall that the R\'enyi entropy of order $\rho$ for $\rho \geq -1$ is defined as\footnote{We extend the definition at $\rho = 1$ via continuous extension.}
\begin{equation}
\label{eq:renyi}
    H_\rho(\bw) := \frac{1}{1-\rho} {\log_2} \sum_{s \in \cS} (w_s)^{\rho}.
\end{equation}
Notice that $H_1(\bw) = -\sum_{g \in \gp} w_g \log w_g$ is the Shannon entropy; and $H_0(\bw) = |\gp|$ is the max-entropy.
Corollary \ref{cor:samplecomplexityalg} gives the upper bound in terms of the R\'enyi entropy of order $2/3$ of the group distribution $\bw$.

\begin{restatable}{corollary}{SamplecomplexityAlgWeighted}(Sample complexity of Algorithm \ref{alg:hypothesis_test}).
Under the same assumption of Theorem \ref{thm:weighted_upper_new} (using weighted sampling). 
With probability at least $0.99$ \revision{\footnote{\revision{Because the probability of success in the hypothesis test is not the main focus of our theoretical analysis, we leave the dependence on it to the appendix and only show the results in the main paper for a probability of error of $0.01$.}}},
comparing $\widehat{F}(\bw)$ to $(1 - \alpha) \epsilon^2/2$, one can differentiate between $ \texttt{F\_CVaR}_{\alpha}(\bw) = 0$ and $ \texttt{F\_CVaR}_{\alpha}(\bw) > \epsilon$ using at most the following number of samples

\begin{equation}
    n = O\left(  \frac{ \left(\sum_{g \in \gp} \frac{w_g^2}{v_g^2} \right)^{1/2}}{
    \epsilon^2(1 - \alpha)} + \frac{\sum_g \frac{w_g^2}{v_g}  }{\epsilon^4(1 - \alpha)^2 }\right).
    \label{eq:upper_general_w_g}
\end{equation}

Moreover, when the model auditor can control the choice of group marginal in the test dataset and set $v_g = \frac{w_g^{2/3}}{\sum_{g^*} w_{g^*}^{2/3}}$, it is only necessary to have access to $n$ samples such that
\begin{equation}
    n =   O\left( \frac{ 2^{\frac{1}{2} H_{2/3}(w)} }{\epsilon^2(1 - \alpha)} + \frac{ 2^{\frac{1}{3} H_{2/3}(w)} }{\epsilon^4(1 - \alpha)^2} \right).
    \label{eq:sample_complexity_main}
\end{equation}

\label{cor:samplecomplexityalg}
\end{restatable}
We show this result in Appendix \ref{sec:upper_bound_CVaR}.
\secondrev{Note that we prove the result for an arbitrary error probability $\delta \in [0,1]$, but only display the result for $\delta=0.01$ as the specific probability of error is not key to our analysis.}
However, we highlight that the dependence on the probability of error $\delta$ is not optimal in our bounds and can be automatically improved by using the median of the means technique, improving the dependence to $\log(1/\delta)$ instead of $1 / \delta$, as in~\cite[Page 8]{suresh_2015}.

\revision{The sample complexity for the hypothesis test using CVaR fairness in \eqref{eq:sample_complexity_main} depends on $(1-\alpha)$ and $\epsilon$.
When $\alpha$ is close to one, i.e., max-gap fairness is recovered, the sample complexity in \eqref{eq:sample_complexity_main} might be worse than the sample complexity for the hypothesis-test using max-gap fairness ($n = O(|\gp|/\epsilon^2)$) as shown in \cite{kearns2018preventing}.
This is a consequence of us treating $\alpha$ as a fixed constant bounded away from zero, which allows us to trade off fairness guarantees for a more feasible sample complexity.
However, Proposition \ref{prop:recover_fairness_knwing_g} shows that the sample complexity of CVaR fairness is, at most, the sample complexity of max-gap fairness since for the most strict choice of $\alpha$, we recover max-gap fairness.
}
When $\epsilon$ and $(1-\alpha)$ are constants and bounded away from zero, the dominant term in the sample complexity is the first term $O\left( \frac{ 2^{\frac{1}{2} H_{2/3}(w)}}{\epsilon^2(1 - \alpha)} \right)$. Corollary \ref{cor:samplecomplexityalg} shows that the sample complexity of performing an $\epsilon$-test CVaR fairness using Algorithm \ref{alg:hypothesis_test} with weighted sampling ($v_g \propto w_g^{2/3}$) has a smaller exponent than testing for standard multi-group fairness metrics as shown in Proposition \ref{prop:hyptest_lower_bound} (testing for standard multi-group fairness has sample complexity $O\left(|\gp|\right)$), because
\begin{equation}
   2^{\frac{1}{2} H_{2/3}(\bw)} \leq 2^{\frac{1}{2} \log (|\gp|)} =|\gp|^{\frac{1}{2}} \leq |\gp|.
\end{equation}
In the worst-case-scenario, when $\bw$ maximizes the R\'enyi entropy the sample complexity is upper bounded by $O\left(|\gp|^{\frac{1}{2}} \big/ ((1 - \alpha)\epsilon^2) \right)$ providing a smaller sample complexity than i.i.d. testing using standard multi-group fairness metrics as shown in Proposition \ref{prop:hyptest_lower_bound} --- R\'enyi entropy is maximized for $w_g = \frac{1}{|\gp|}$ for all $g \in \gp$.

\revision{
In contrast, when $\epsilon$ is sufficiently small, the term $ O \left( \frac{ 2^{\frac{1}{3} H_{2/3}(w)} }{\epsilon^4}\right)$ may be dominant.
However, this term is still preferable over the original complexity of the problem $O \left( \frac{ |\gp| }{\epsilon^2}\right)$ when $\epsilon > \frac{1}{|\gp|^{1/3}}$. 
Recall that the case of interest in this paper is when $\gp$ contains a large number of groups (e.g., millions of groups) defined by the combination of sensitive attributes, hence the case where $\epsilon > \frac{1}{|\gp|^{1/3}}$ is of major interest --- here we are not considering constants in the sample complexity.
}
Moreover, when the model auditor can control the data collection process, they can choose to sample according to Algorithm \ref{alg:hypothesis_test} to decrease the sample complexity further.
For example, when the model auditor has access to at least $2$ samples of any given group $g \in \gp$ in the population and can perform non-i.i.d. sampling, they can use attribute specific sampling.
Next, we show that by choosing $M_g$ using attribute specific sampling (Definition \ref{def:attributesampling}), we can get a bound on the probability of error for the $\epsilon$-test that is independent of $\bw$.
We show this result in Theorem \ref{thm:att_upper_new}. 

\begin{restatable}{theorem}{AttUpperNew}(Achievability with attribute specific sampling).
Let $\bw$ be such that $\max_{g} w_{g} \leq c \cdot \epsilon^4 $ for a sufficiently small constant $c$ \revision{and $\max_{g} w_{g} \leq 1 - \alpha$} ,
$L: \cZ \rightarrow \{0, 1\}$ the quality of service function, and $\averagegtg{L}$ the average quality of service.
Using attribute specific sampling in Algorithm \ref{alg:hypothesis_test}, i.e.
\begin{equation*}
    Q_g(k) = \Pr\left(\text{Ber}(\min(1, \gamma w_g))\frac{n}{\gamma} = k\right),
\end{equation*}
where $\gamma = \frac{n}{2}$.

The probability of error ($P_{\text{error}}$ in \eqref{eq:error_prob}) for Algorithm \ref{alg:hypothesis_test} to differentiate $ \texttt{F\_CVaR}_{\alpha}(\bw) = 0$ from $ \texttt{F\_CVaR}_{\alpha}(\bw) > \epsilon$ is bounded by:
\begin{equation}
        P_{\text{error}} = O \left( \frac{1}{\epsilon^4 (1 - \alpha)^2 n}\right).
\end{equation}

\label{thm:att_upper_new}
\end{restatable}

We show this result in Appendix \ref{sec:upper_bound_CVaR}. As an immediate consequence of Theorem \ref{thm:att_upper_new}, we can compute the sample complexity of testing using Algorithm \ref{alg:hypothesis_test} with attribute specific sampling.
In Corollary \ref{cor:samplecomplexityalg_att}, we show that using attribute specific sampling allows us to confidently perform an $\epsilon$-test using a number of data points that do not depend on the number of group $|\gp|$.

\begin{restatable}{corollary}{SamplecomplexityalgAtt}(Sample complexity of Algorithm \ref{alg:hypothesis_test} using attribute specific sampling).
Let $\bw$ be such that $\max_{g} w_{g} \leq c \cdot \epsilon^4 $ for a sufficiently small constant $c$ \revision{and  $\max_{g} w_{g} \leq 1 - \alpha$}.
With probability at least 
$0.99$,
comparing $\widehat{F}(\bw)$ to $(1 - \alpha) \epsilon^2/2$, one can differentiate between $ \texttt{F\_CVaR}_{\alpha}(\bw) = 0$ and $ \texttt{F\_CVaR}_{\alpha}(\bw) > \epsilon$ using at most the following number of samples $n$: 
$$n = O\left( \frac{1}{\epsilon^4 (1 - \alpha)^2}\right).$$

\label{cor:samplecomplexityalg_att}
\end{restatable}
We provide the proof for this result in Appendix \ref{sec:upper_bound_CVaR} --- again, we don't show the dependency on the error probability $\delta$, but we prove it in the appendix and provide the bound for $\delta = 0.01$.

In this section, we have shown that it is possible to perform an $\epsilon$-test for CVaR fairness reliably using $\weiupperbound$ samples from $P$ when using Algorithm \ref{alg:hypothesis_test} with weighted sampling --- which is i.i.d.
We also showed that, by using Algorithm \ref{alg:hypothesis_test} with attribute specific sampling (which is non-i.i.d.), we can reliably test for CVaR fairness with $n = \attupperbound$ samples from $P$.

\revision{Intuitively, our theoretical results are a consequence of the fact that when the model auditor can control the sampling frequency in specific groups, then they can decide to sample more {\em uniformly} from all groups. For example, when using weighted sampling with $\bv \propto \bw^{2/3}$, the groups with higher $w_g$ receive a smaller $v_g$ while the groups with smaller $w_g$ receive a higher $v_g$.
Then, if smaller groups receive higher values of $\Delta(g)$, our proposed sampling methods will capture this disparity.
On the other hand, the model auditor can decide to use attribute specific sampling, which intuitively first computes $\Delta(g)$ for the groups with higher weights $w_g$ which is a non-i.i.d. sampling strategy wich a sample complexity independent of $|\gp|$.
}

In the next section, we show that by using weighted sampling with $\bw$ uniform, it is actually necessary to use $O\left(\frac{2^{\frac{1}{2} H_{2/3}(\bw)}}{(1 - \alpha)^{\frac{1}{2}}\epsilon^2}\right) = O\left(\frac{|\gp|^{\frac{1}{2}}}{(1 - \alpha)^{\frac{1}{2}}\epsilon^2}\right)$ samples from $P$ for performing an $\epsilon$-test reliably.
Indicating that Algorithm \ref{alg:hypothesis_test} has optimal dependency on $\bw$ and $|\gp|$.

\section{Converse Result for CVaR Fairness Testing under weighted sampling}
\label{sec:ConverseCVaR}
We present next a converse result for the sample complexity of the $\epsilon$-test for CVaR fairness under weighted sampling.
Our result ensures that the performance of Algorithm \ref{alg:hypothesis_test} has optimal order when $\bw$ is a uniform distribution over all groups.
We prove this optimality by showing that no test can differentiate between $ \texttt{F\_CVaR}_{\alpha}(\bw) = 0$ and $ \texttt{F\_CVaR}_{\alpha}(\bw) \geq \epsilon $ using less than $\Omega\left( \frac{\sqrt{|\gp|}}{\epsilon^2}\right)$ i.i.d. samples --- Theorem \ref{thm:converse_CVaR}.
As a consequence of this result, we provide an operational meaning for the R\'enyi entropy of order $2/3$ as the log of the number of samples that is necessary to perform an $\epsilon$-test for CVaR reliably.

We proved a similar converse result for the max-gap fairness metrics in Proposition \ref{prop:hyptest_lower_bound}.
For max-gap fairness we were able to design two distributions $P_0$ and $P_1$ that are close together in the space of all distributions in $\cZ$ according to the Hellinger divergence but at the same time $ \texttt{F\_CVaR}_{\alpha}(\bw; P_0) = 0$ and $ \texttt{F\_CVaR}_{\alpha}(\bw; P_1) \geq \epsilon$ --- Proposition \ref{prop:hyptest_lower_bound}.
Using these distributions, we leveraged the Le Cam Lemma \cite{LeCam1973} from the binary hypothesis test literature and showed that it is necessary to have $n = \Omega(|\gp|)$ samples to perform an $\epsilon$-test for max-gap fairness reliably --- Proposition \ref{prop:hyptest_lower_bound}.

Alas, the proof strategy for deriving a converse result for max-gap fairness testing does not directly apply to CVaR testing.
The reason is that it is not possible to design distributions $P_0$ and $P_1$ as before.
In fact, we can show that if two distributions have CVaR fairness far apart, then the total variation distance between them is also far apart.
Specifically, if we design $\texttt{F\_CVaR}_{\alpha}(\bw; P_0) = 0$ and $ \texttt{F\_CVaR}_{\alpha}(\bw; P_1) \geq \epsilon$ we cannot guarantee that $\text{TV}(P_0 || P_1)$ is sufficiently small to match our achievable sample complexity bound.
Hence, the approach from Proposition \ref{prop:hyptest_lower_bound} fails.

As an alternative to the approach in Proposition \ref{prop:hyptest_lower_bound}, we take a different route and use the Generalized Le Cam Lemma for a mixture of distributions \cite{Polyanskiy19}.
Instead of showing that there exist distributions $P_0$ and $P_1$ with distant CVaR fairness metric but close distance in the space of distributions, we design a sequence of distributions $\{P_u\}_{\mathcal{U}}$ indexed by $\mathcal{U}$ such that $ \texttt{F\_CVaR}_{\alpha}(\bw; P_u) \geq \epsilon$ for all $u \in \mathcal{U}$, and also $P_0$ such that $\texttt{F\_CVaR}_{\alpha}(\bw; P_0) = 0$.
We then show that the mixture of distributions $\{P_u\}_{\mathcal{U}}$ is close to $P_0$ according to the $\chi^2$ divergence --- Lemma \ref{lem:main_chi_bound_cvar}. 

Recall that the $\chi^2$ divergence between {$P \ll Q$} (i.e., $P$ is absolutely continuous with respect to $Q$) is given by
\begin{equation}
    \chi^2(P || Q) \triangleq \EE_{Q} \left[ \left( \frac{P(Z)}{Q(Z)} - 1\right)^2\right].
\end{equation}

In Lemma \ref{lem:main_chi_bound_cvar}, we design the mixture distribution $P_u \in \hypregions_1(\epsilon)$ and show that the mixture is close to a distribution $P_0 \in \hypregions_0$ according to the $\chi^2$ divergence.

\begin{restatable}{lemma}{MainChiBoundCvar}[Close Mixture of Distributions] 
For CVaR fairness metrics $\texttt{F\_CVaR}_{\alpha}(\bw)$ with quality of service function $L$ and measuring the average quality of service using the distribution $P$ such that $(L, P) \in \{(L_{EO}, P_{EO}), (L_{SP}, P_{SP})\}$.
Consider the hypothesis regions $\hypregions_0$ \eqref{eq:hyp_reg_0} and $\hypregions_1(\epsilon)$ \eqref{eq:hyp_reg_1}.
Assume that $\bw$ is a uniform distribution over all groups, i.e., $\bw = \left(\frac{1}{|\gp|}, ..., \frac{1}{|\gp|} \right)$,
For all $\alpha \in \cvarrange$ and $\frac{\alpha  |\gp|^{1/3}}{4}\geq \epsilon > 0$  there exist a sequence of distributions $P_0 \in \hypregions_0$ and $ \{P_{u}\}_{u \in \mathcal{U}} \subset \hypregions_1(\epsilon)$ indexed by elements of the set $\mathcal{U}$ such that if $u \sim \pi$ we have that
\begin{equation}
 {\chi}^2\left( \EE_{u}\left[ P_{u}^n \right]  \biggl| \biggl|  P_0^n  \right)
    \leq e^{ \frac{ 128 (1 - \alpha) n^2 \epsilon^4}{\alpha^{4} |\gp| }} - 1.
\end{equation}
\label{lem:main_chi_bound_cvar}
\end{restatable}
We prove this result in Appendix \ref{sec:converse_CVAR_apx}. In Lemma \ref{lem:main_chi_bound_cvar}, we showed that there exists a mixture of distributions with CVaR fairness greater than $\epsilon$ that is close to a distribution $P_0$ with CVaR fairness equal to $0$.
Hence, we can use the generalized Le Cam Lemma to compute the minimum number of samples that is necessary to perform an $\epsilon$-test reliably.
In Theorem \ref{thm:converse_CVaR} we show that it is necessary to have $n = \Omega\left( \frac{\sqrt{|\gp|}}{(1 - \alpha)^{1/2} \epsilon^2} \right)$ samples to perform an $\epsilon$-test reliably.

\begin{restatable}{theorem}{ConverseCVar}(Minimum Sample Complexity).
For CVaR fairness metrics $\texttt{F\_CVaR}_{\alpha}(\bw)$ with quality of service function $L$ and measuring the average quality of service using the distribution $P$ such that $(L, P) \in \{(L_{EO}, P_{EO}), (L_{SP}, P_{SP})\}$.
Consider the hypothesis regions $\hypregions_0$ \eqref{eq:hyp_reg_0} and $\hypregions_1(\epsilon)$ \eqref{eq:hyp_reg_1}.
Assume that $\bw$ is a uniform distribution over all groups, i.e., $\bw = \left(\frac{1}{|\gp|}, ..., \frac{1}{|\gp|} \right)$.
For all $\alpha \in \cvarrange$,  $\frac{\alpha  |\gp|^{1/3}}{4}\geq \epsilon > 0$ , the minimum probability of error in the $\epsilon$-test using $n$ i.i.d. samples from $P$ to distinguish between $\texttt{F\_CVaR}_{\alpha}(\bw) = 0$ and $\texttt{F\_CVaR}_{\alpha}(\bw) \geq \epsilon$ is lower bounded by
\begin{equation}
        2\inf_{\psi} P_{\text{error}} \geq 1 - \frac{1}{2} \sqrt{e^{ \frac{ 1024 (1 - \alpha) n^2 \epsilon^4}{\alpha^{4} |\gp|  }} - 1},
\end{equation}
then $n$ samples are needed to perform an $\epsilon$-test with with probability of error smaller than $0.01$ where $n$ is such that 
\begin{equation}
        n = \Omega\left( \frac{\alpha^2\sqrt{|\gp|}}{(1 - \alpha)^{1/2} \epsilon^2} \right). 
    \end{equation}
\label{thm:converse_CVaR}
\end{restatable}
We prove this result in Appendix \ref{sec:converse_CVAR_apx} --- one more time, we provide the dependence on the error probability $\delta$ in the appendix, and the result is shown for $\delta = 0.01$. In Theorem \ref{thm:converse_CVaR}, we showed that the order of the sample complexity of the Algorithm 
\ref{alg:hypothesis_test} is optimal when using weighted sampling and $\bw$ is uniform.
The results in Lemma \ref{lem:main_chi_bound_cvar} and Theorem \ref{thm:converse_CVaR} were given in terms of the equal opportunity fairness metric.
However, our result can be trivially extended to several fairness metrics, such as statistical parity and equalized odds.

We highlight that the result in Theorem \ref{thm:converse_CVaR} together with Corollary \ref{cor:samplecomplexityalg} give an operational meaning for the R\'enyi entropy of order $2/3$ as the log of the number of samples that is necessary to perform an $\epsilon$-test for CVaR reliably.

\revision{
\section{Numerical Experiments}
\label{sec:experiments}

In this section, we empirically evaluate the performance of the proposed hypothesis test methods within a  Bernoulli model. This model allows us to examine test performance by varying three key parameters: (i) the parameter $\alpha$ in Definition~\ref{def:CVaRFairness} (Figure \ref{fig:alpha_dependency}), (ii) the entropy of the group distribution $\bw$ (Figure \ref{fig:FnFPp}), and (iii) the number of samples in the hypothesis test (Figure \ref{fig:AUC_nsamples_p}). 

Additionally, we determine the maximum number of groups for which we can conduct an $\epsilon$-test for both max-gap fairness and CVaR fairness with a probability of error smaller than $45\%$ and a fixed data budget. The maximum number of groups can be derived via the lower bound provided in Proposition \ref{prop:hyptest_lower_bound} and Theorem \ref{thm:converse_CVaR} (Figure \ref{fig:max_groups}).
For practitioners, this result aids in deciding the number of sensitive attributes (generating groups $g \in \gp$) required to ensure reliable auditing for max-gap or CVaR fairness within a fixed data budget. In this section, we take a conservative approach and assume that {\em Reliable auditing} refers to performing an $\epsilon$-test with a probability of error smaller than $45\%$; essentially, the test needs to perform slightly better than a random guess to be considered reliable. Our results can also be applied for a smaller probability of error, say $\delta$ instead of $0.45$, as we show in Appendix \ref{sec:upper_bound_CVaR}. 
This impacts our results by a factor of $\frac{1}{\delta}$ instead of $\frac{1}{0.45}$, decreasing the number of possible groups for reliable testing --- this dependence can be improved to $\log\left(\frac{1}{\delta}\right)$ by using the median of the means technique from ~\cite[Page 8]{suresh_2015}.

\subsection{Experimental Setup}
First, we define the dataset distribution ($\widetilde{P}_{Z}$) and the quality of service we aim to ensure equality across different groups in the population.
Our experiment aims to ensure statistical parity across groups (Example \ref{eg:statistical_parity}) using max-gap and CVaR fairness.

We start this section by defining the protected groups $\gp$.
\paragraph{The Demographic Groups} We use $d \in \mathbb{N}$ binary sensitive attributes --- in this experiment, we don't specify such groups; however, they can encode information about sex, HIV status, or even the occurrence of stroke.
Hence, each demographic group is given by a binary sequence of length $d$; particularly, we define the set of all protected groups as all elements of $\gp(d) = \{0, 1\}^{d}$ --- e.g., when $d = 2$ the protected groups are given by $\gp(2) = \{(0, 0); (1, 0); (0, 1); (1, 1)\}$. We denote the groups by $g \in  \gp(d)  = \{0, 1\}^{d} $ and $g^i$ indicates the i-th entry of the group denoting vector in $\gp(d)$.

We define the probability of occurrence of each group $g \in \gp(d)$ as the probability of the product of $d$ Bernoulli distributions with parameter $p \in [0,1]$ to be equal to the group sensitive attributes. Rigorously, we define $\bw = (w_1, w_2, ..., w_{2^d})$ as:
\begin{equation}
    w_g \triangleq \prod_{i = 1}^{d} \Pr\left(\text{Ber}(p) = g^i \right).
    \label{eq:experimental_w}
\end{equation}

We define the group weights as in \eqref{eq:experimental_w} for the entropy of $\bw$ to be equal to $d$ times the entropy of the Bernoulli random variable, specifically $H_{2/3}(\bw) = d H_{2/3}((p, 1-p))$.
In Figures \ref{fig:FnFPp} and \ref{fig:AUC_nsamples_p}, we vary the parameter $p$ to analyze the impact of the entropy of the group weight vector $\bw$ on the hypothesis test performance.

Now that we have defined the groups considered in our numerical analysis, we introduce the prediction dataset distribution, i.e., the distribution for $z = (g, \hat{y})$  

\paragraph{Data Distribution} Each individual is represented by a vector $z = (g, \hat{y})$ where $g$ is its group membership and $\hat{y} \in \{0, 1\}$ is the outcome we aim to ensure statistical parity.
The outcome distribution is given by
\begin{equation}
    \Pr\left(\hat{y} = 1 | G = g \right) \triangleq q_g, 
\end{equation}
where $q_g \in [0, 1]$ is the per group quality of service parameter and $\hat{y}_{ | G = g}$ is a Bernoulli random variable with parameter $q_g$.
Hence, the dataset distribution of each data point is given by
\begin{equation}
    \widetilde{P}_Z \left((G, \widehat{Y}) = (g, \hat{y})\right) = w_g q_g.
\end{equation}

Now that we have defined the groups, we aim to ensure equal treatment across all elements of $\gp(d)$. We recall the quality service function and distribution, respectively,  $L_{SP}(z)$ and $P_{SP}(z)$, used for measuring the quality of the service in each group and compute the Max-Gap and CVaR fairness for the model.

\paragraph{Max-Gap and CVaR Fairness} The quality of service function and the measure for the quality of service of interest are, respectively, $L_{SP}$ and $P_{SO}$ given by
\begin{align}
    L_{SP}(g, \hat{y}) & = \mathbb{I}_{\hat{y} = 1}(g, \hat{y})\\
    P_{SP}(g, \hat{y}) &  =  \widetilde{P}_Z (g, \hat{y}).
\end{align}

Therefore, the fairness gap per group $g \in \gp$ is given by $\Delta(g)$ as follows
\begin{equation}
    \Delta(g) = \left|  q_g -  \sum_{g \in \gp} q_g w_g  \right|,
\end{equation}

which implies that the Max-Gap and CVaR fairness are given by
\begin{align}
\fairness &= \max_{g \in \gp} \Delta(g),\\
\texttt{F\_CVaR}_{\alpha}(\bw) &= \frac{1}{1 - \alpha}  \max_{\substack{\sum_{g \in Q} w_g \leq 1 - \alpha \\ Q \subset \gp}}\sum_{g \in Q} w_g \Delta(g).
\end{align}

\paragraph{Hypothesis Testing} We aim to test if all groups are treated similarly ($\fairness = 0$ or $\texttt{F\_CVaR}_{\alpha}(\bw) = 0$) or if there is a group that is unfairly treated ($\fairness \geq \epsilon$ or $\texttt{F\_CVaR}_{\alpha}(\bw) \geq \epsilon$), i.e., we aim to perform an $\epsilon$-test.
However, in general, we assume we don't have access to the data distribution; hence, to simulate this scenario, we perform the $\epsilon$-test using samples using (i) weighted sampling and (ii) attribute specific sampling.
We perform the \emph{hypothesis test for CVaR fairness} using the Algorithm \ref{alg:hypothesis_test} that we propose.
We use weighted sampling (Definition \ref{def:weightedsampling}) with $\bv = \bw$ and $\bv = \bw^{2/3}$, and attribute specific sampling.
We adopt a \emph{baseline testing for max-gap fairness} that compares the empirical average performance with the empirical average per group.
Specifically, when $n$ samples are available $\{z_i\}_{i = 1}^n$
each group has $M_g$ samples such that $\sum_g M_g  = n$, the baseline method outputs that $\fairness \geq \epsilon$ when
\begin{equation}
    \max_{\substack{g \in \gp \\ M_g > 0}} \left|   \frac{\sum_{i = 1}^{M_g} L_{SP}(z_i)}{M_g}   -   \frac{\sum_{g \in \gp} \sum_{i = 1}^{M_g} L_{SP}(z_i)}{n} \right| \geq \epsilon,
    \label{eq:maxgaptest}
\end{equation}
and outputs that $\fairness = 0$ otherwise. When performing the baseline test for max-gap fairness, we sample i.i.d. from the data distribution that is equivalent to weighted sampling with $\bv = \bw$.

\paragraph{Disparate Quality of Service} To generate disparate quality of service across groups, i.e., $\Delta(g) > 0$ for some $g \in \gp$ we take the following approach.

\begin{approach}
We set the group quality parameter $q_g = 0.05$ for $20\%$ of groups that we choose uniformly at random and $q_g = 0.5$ for the other groups.
Specifically, we sample $\lfloor 0.2 |\gp| \rfloor$ groups $\{g_1, ..., g_{\lfloor 0.2 |\gp| \rfloor}\}$ and define $q_g = 0.05$ for $g \in \{g_1, ..., g_{\lfloor 0.2 |\gp| \rfloor}\}$ and $q_g = 0.5$ for all other groups.
\label{approach3}
\end{approach}

Now that we have laid the background for our experiments, we will show our numerical results in the next section.

\subsection{Experimental Results}

\paragraph{On the dependence on $\alpha$} Figure \ref{fig:alpha_dependency} shows how the probability of error in the $\epsilon$-test for CVaR fairness (in terms if error probability \eqref{eq:error_prob}) using Algorithm \ref{alg:hypothesis_test} depends on the choice of the $\alpha$ parameter in CVaR fairness Definition \ref{def:CVaRFairness}.
The per group quality of service was defined using Approach \ref{approach3}, i.e.,  we set $20\%$ of all groups chosen uniformly at random to received a quality of service equals to $q_g = 0.05$ while the other groups received $q_g = 0.5$.

Figure \ref{fig:alpha_dependency} indicates that when $\alpha$ increases, the probability of error when performing an $\epsilon$-test also increases.
However, in the low data regime, this behavior changes. 
In this case, the errors that stem from the limited sample size dominate the probability of error, as we can conclude from the unchanged probability of error when $n = 256$ across both weighted sampling strategies and also attribute specific sampling Figures \ref{fig:alpha_dependency} (a), (b), and (c).
The error dominance from the small sample size $n = 256$ is even higher in attribute specific sampling (Figures \ref{fig:alpha_dependency} (c)).

Additionally, we also observe that the dependency on $\alpha$ is more severe in attribute-specific sampling in comparison with weighted sampling using $\bv \propto \bw^{2/3}$ or $\bv = \bw$ --- which was expected from our error probability bound.
Corollary \ref{cor:samplecomplexityalg} shows that the number of samples necessary to reliably perform an $\epsilon$-test for CVaR fairness using weighted sampling depended on $\frac{1}{(1 - \alpha)}$, while Corollary \ref{cor:samplecomplexityalg_att} shows that the number of samples depends on $\frac{1}{(1 - \alpha)^2}$ when using attribute specific sampling, implying that $\alpha$ has a more severe impact on the error probability of attribute specific sampling when compared with weighted sampling.

\begin{figure}[t]
    \centering
    \begin{subfigure}[b]{0.32\linewidth}
        \includegraphics[width=\textwidth]{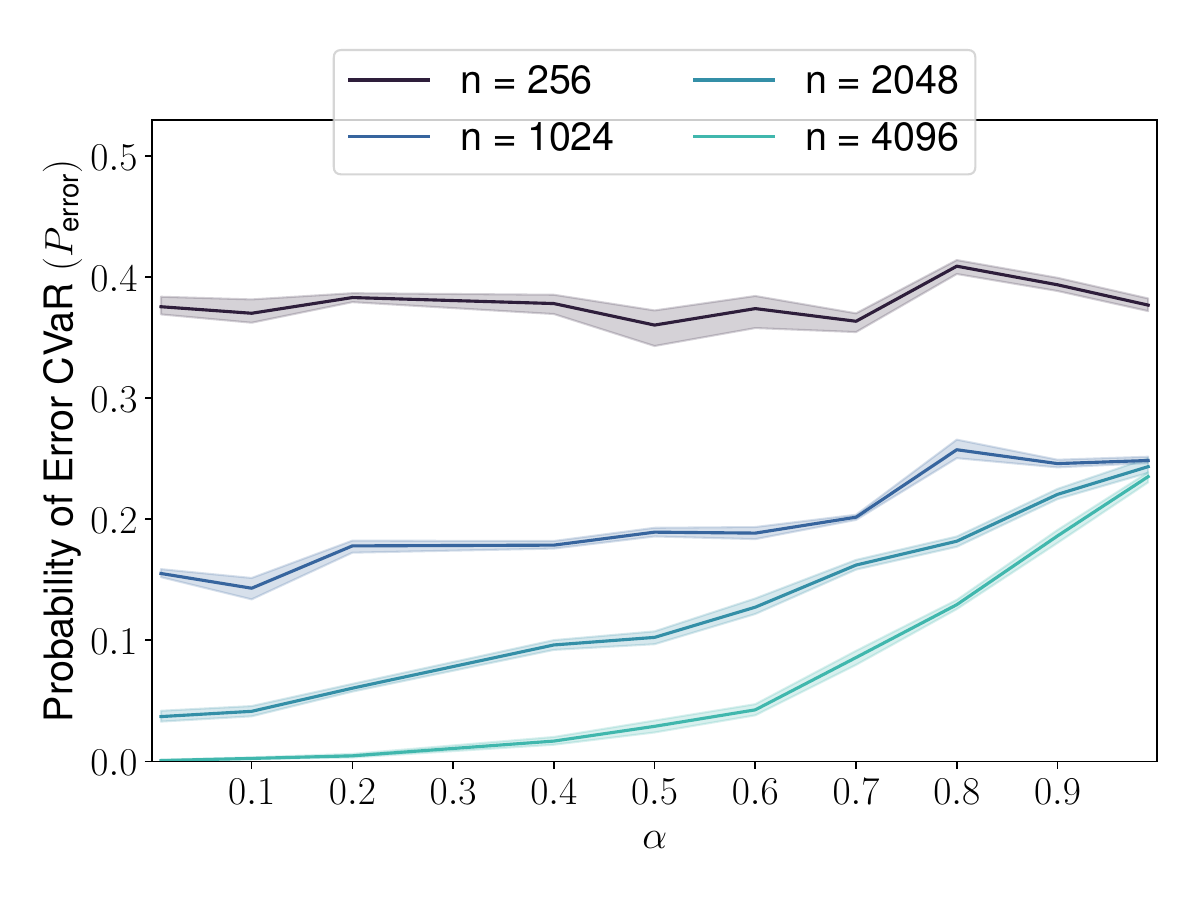}
        \caption{$\bw$ Weighted}
    \end{subfigure}
    \begin{subfigure}[b]{0.32\linewidth}
        \includegraphics[width=\textwidth]{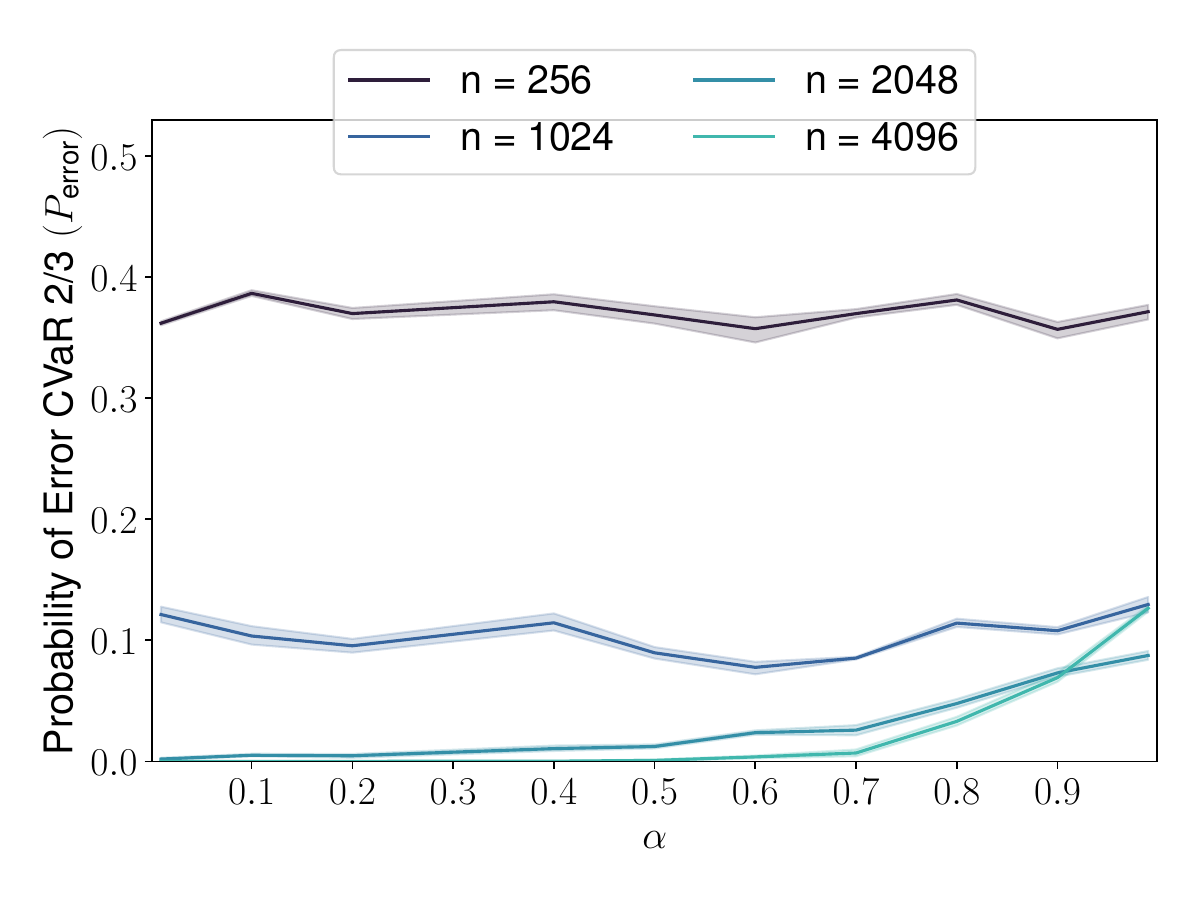}
        \caption{$\bw^{2/3}$ Weighted}
    \end{subfigure}
    \begin{subfigure}[b]{0.32\linewidth}
        \includegraphics[width=\textwidth]{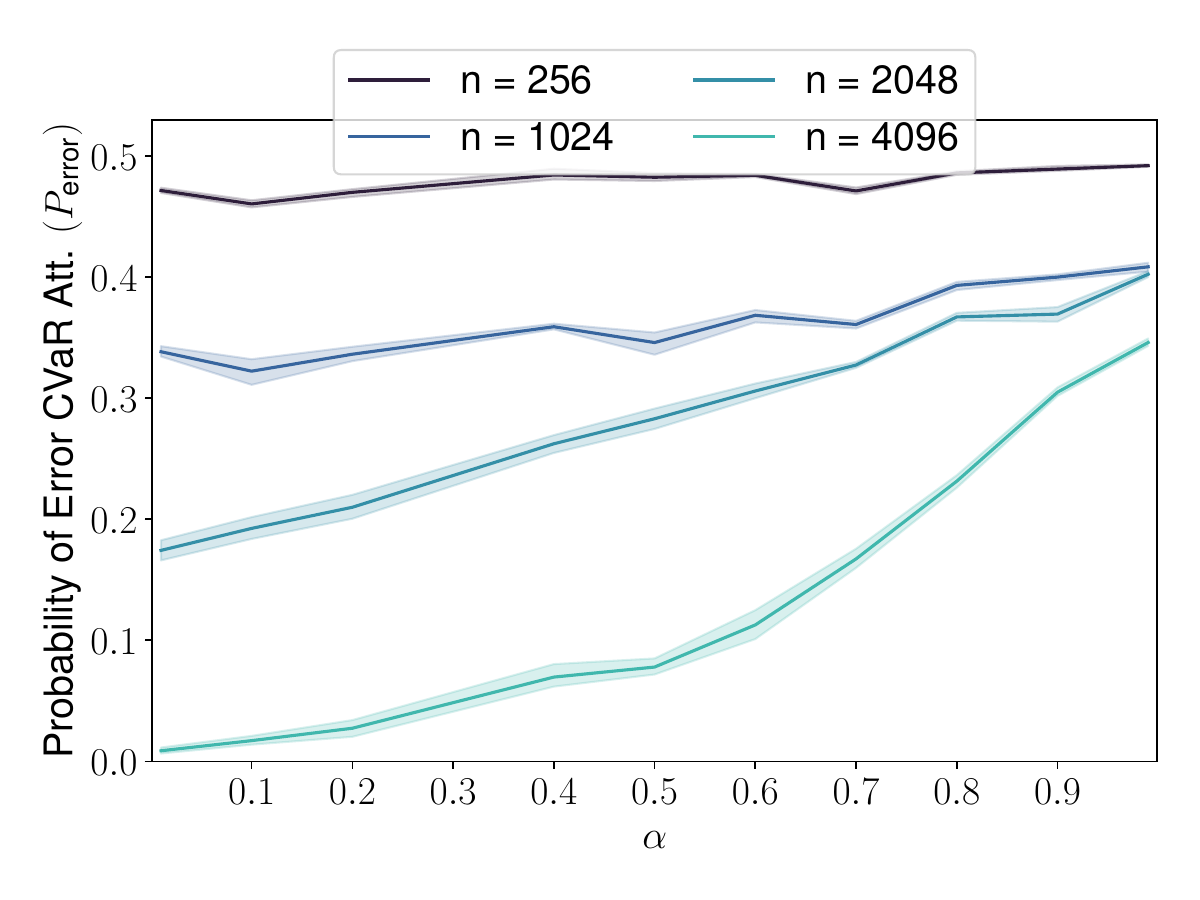}
        \caption{Attribute Specific}
    \end{subfigure}
     \caption{
     \revision{
     Probability of error in  $\epsilon$-test vs. $\alpha$ for different number of sample sizes.
  We use $\epsilon = 0.4 {{\times}} \texttt{F\_CVaR}_{\alpha}(\bw)$ in Algorithm \ref{alg:hypothesis_test} --- we exactly compute CVaR fairness to evaluate the probability of error.
  We use $d = 9$ sensitive attributes leading to $512$ groups, and the group distribution parameter is taken to be $p = 0.2$.
  The probability of error was computed using $1000$ realizations with the Monte Carlo method, and confidence intervals are plotted using Bootstrap from Seaborn \cite{Waskom2021} with $95\%$ confidence.}}
  \label{fig:alpha_dependency}
\end{figure}

\paragraph{False positive vs. false negative rate for different values of entropy}
Figure \ref{fig:FnFPp} shows how the false positive rate depends on the false negative rate for different entropy values achieved by varying the group distribution parameter $p \in \{0.05, 0.1, 0.5\}$).
To compute the false positive rate at a given false negative rate, we tune $\epsilon$ in Algorithm \ref{alg:hypothesis_test} such that the false negative rate is given by the x-axis and use this $\epsilon$ to test if $\texttt{F\_CVaR}_{\alpha}(\bw) \geq \epsilon$.
We select the threshold in the baseline test for max-gap \eqref{eq:maxgaptest} similarly to test if $\fairness \geq \epsilon$.

Figure \ref{fig:FnFPp} indicates that weighted sampling for CVaR testing has a more favorable trade-off between false negatives and positives when compared with max-gap fairness testing, specifically for lower entropy values.
Notably, the attribute specific sampling method (Figure \ref{fig:FnFPp} (d)) exhibits inferior or comparable performance compared to weighted sampling with group weights $\bv = \bw$ (Figure~\ref{fig:FnFPp}~(b)) or $\bv \propto \bw^{2/3}$ (Figure \ref{fig:FnFPp} (c)) across most entropy values obtained by varying the parameter $p$. However, under the condition of uniformly distributed group weights (i.e., $p = 0.5)$, where entropy is maximized, attribute specific sampling yields the best trade-off between false negative and positive rates, indicating that attribute specific sampling has a better performance when the entropy $\bw$ is higher --- this result is also supported by Figure \ref{fig:AUC_nsamples_p}.

Weighted sampling with $\bv \propto \bw^{2/3}$ (Figure \ref{fig:FnFPp} (c)) has the most favorable trade-off between false negative and positive rates for smaller entropies of $\bw$ --- as is suggested by the direct dependence of its sample complexity on the $H_{2/3}(\bw)$ as we shown in Corollary \ref{cor:samplecomplexityalg}. 
At maximum entropy ($p = 0.5$), weighted sampling and the baseline test for max-gap fairness exhibit similar performance, but attribute specific sampling clearly outperforms both weighted sampling and the baseline test for max-gap fairness.

Remarkably, weighted sampling with $\bv \propto \bw^{2/3}$ achieves false positive rates smaller than $10\%$ at all false negative rates larger than $20\%$ when $p = 0.05$.
Therefore, we conclude that weighted sampling and attribute specific sampling, when used for CVaR fairness testing, exhibit a better or comparable false negative vs. positive trade-off than the baseline method for max-gap fairness testing \eqref{eq:maxgaptest} across all tested values for the parameter $p$ that controls the group weights $\bw$ entropy.
Moreover, we observe that weighted sampling with $\bv \propto \bw^{2/3}$ is more flexible, achieving reasonable performance for different $\bw$ entropy values; however, its performance is degraded when entropy is maximized.
Attribute specific sampling complements weighted sampling by achieving a more favorable when the entropy of $\bw$ is maximized.

\begin{figure}[t]
  \centering
  \begin{subfigure}{0.45\linewidth}
    \includegraphics[width=\linewidth]{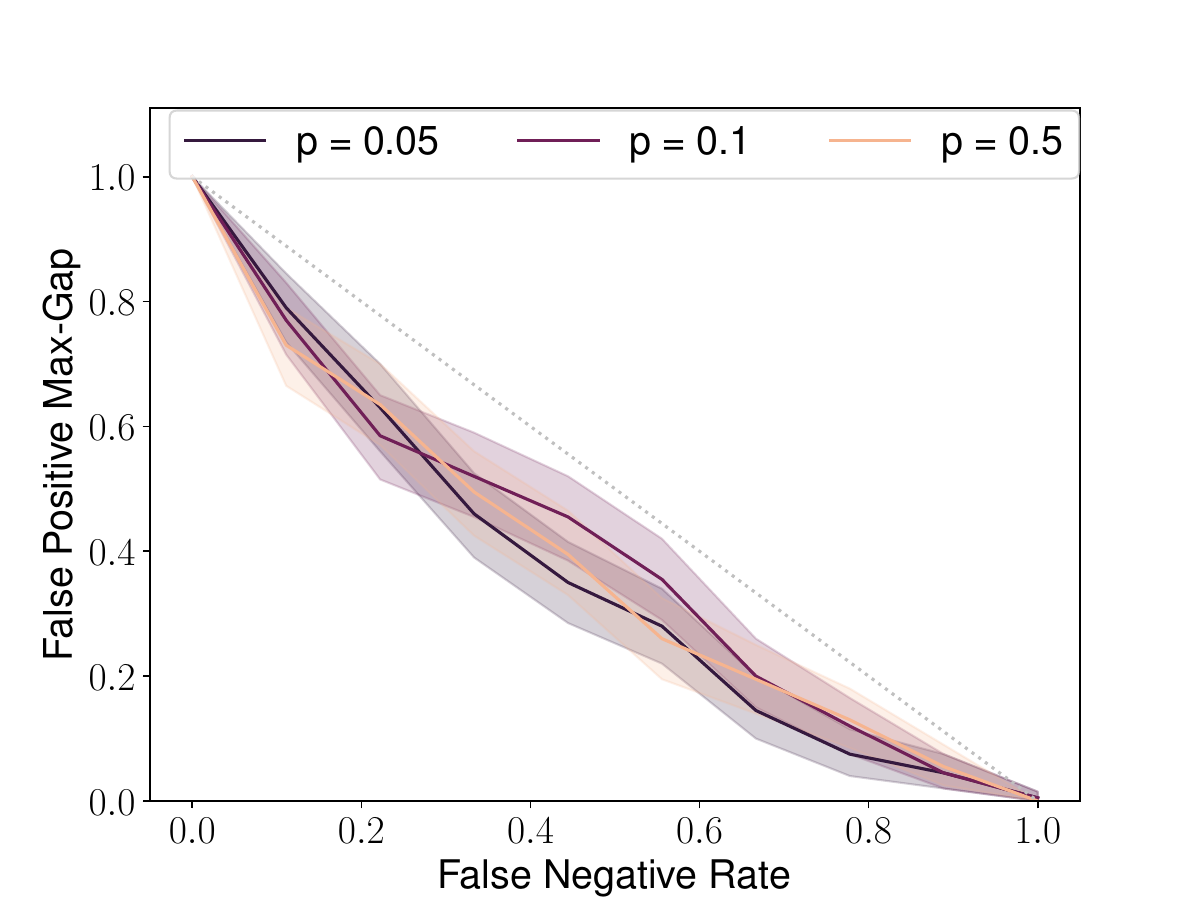}
    \caption{Max-Gap}
  \end{subfigure}
  \begin{subfigure}{0.45\linewidth}
    \includegraphics[width=\linewidth]{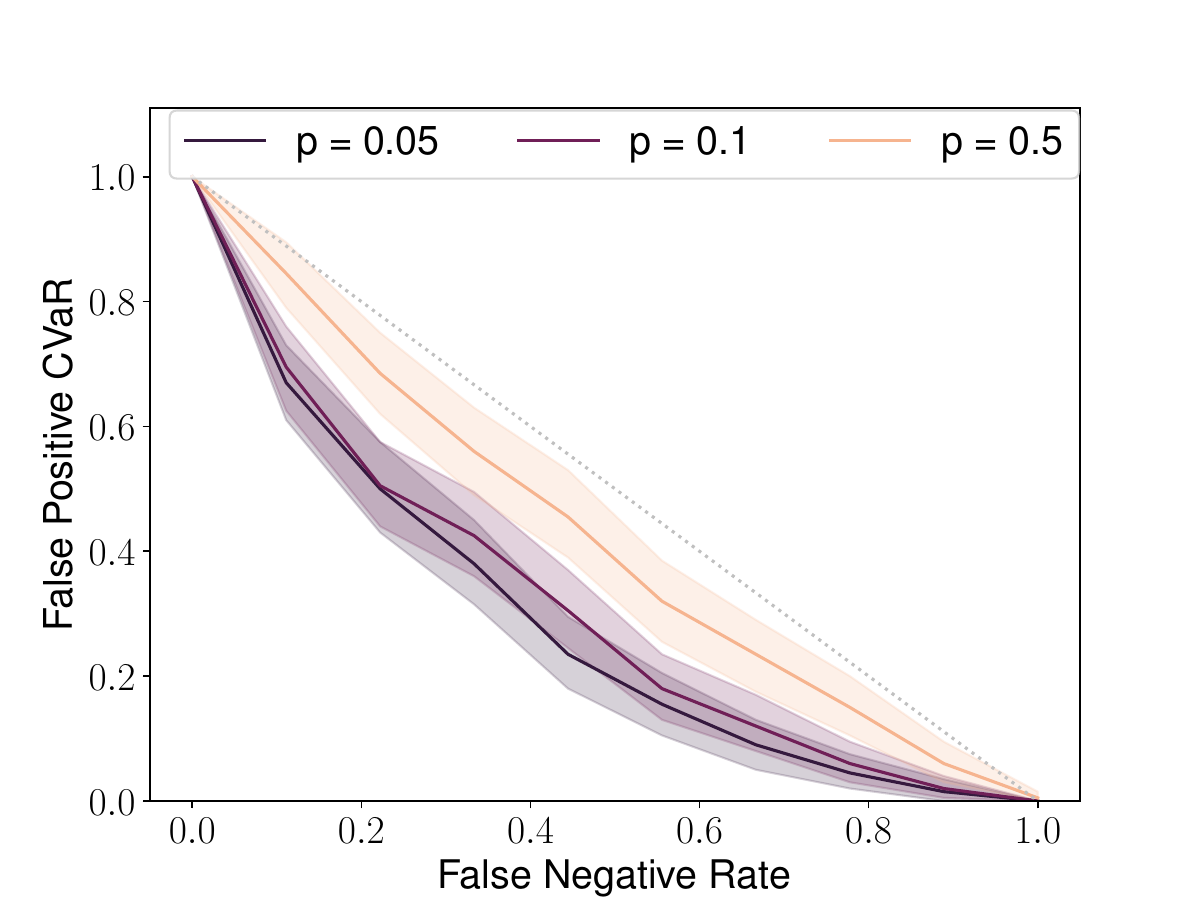}
    \caption{$\bw$ Weighted}
  \end{subfigure}
  \hfill
  \begin{subfigure}{0.45\linewidth}
    \includegraphics[width=\linewidth]{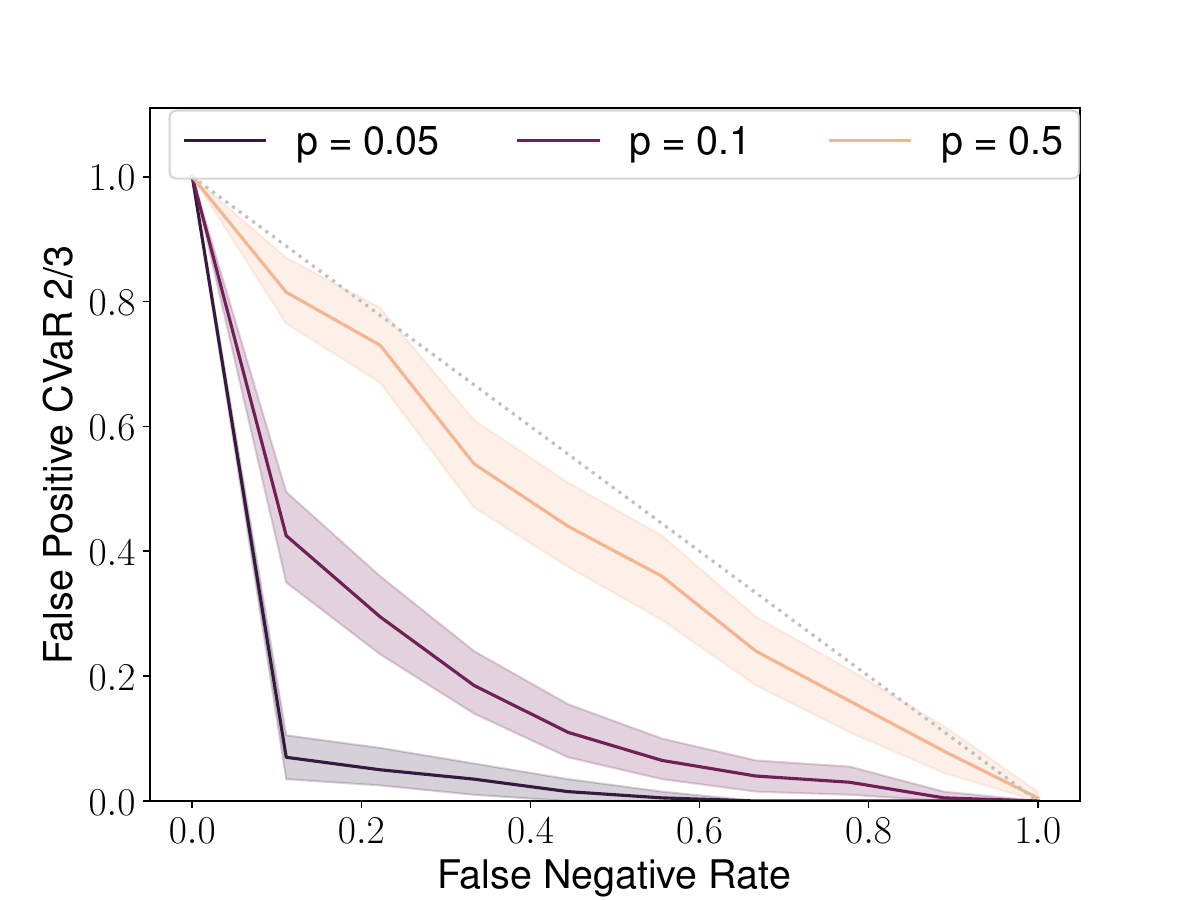}
    \caption{$\bw^{2/3}$ Weighted}
  \end{subfigure}
  \begin{subfigure}{0.45\linewidth}
    \includegraphics[width=\linewidth]{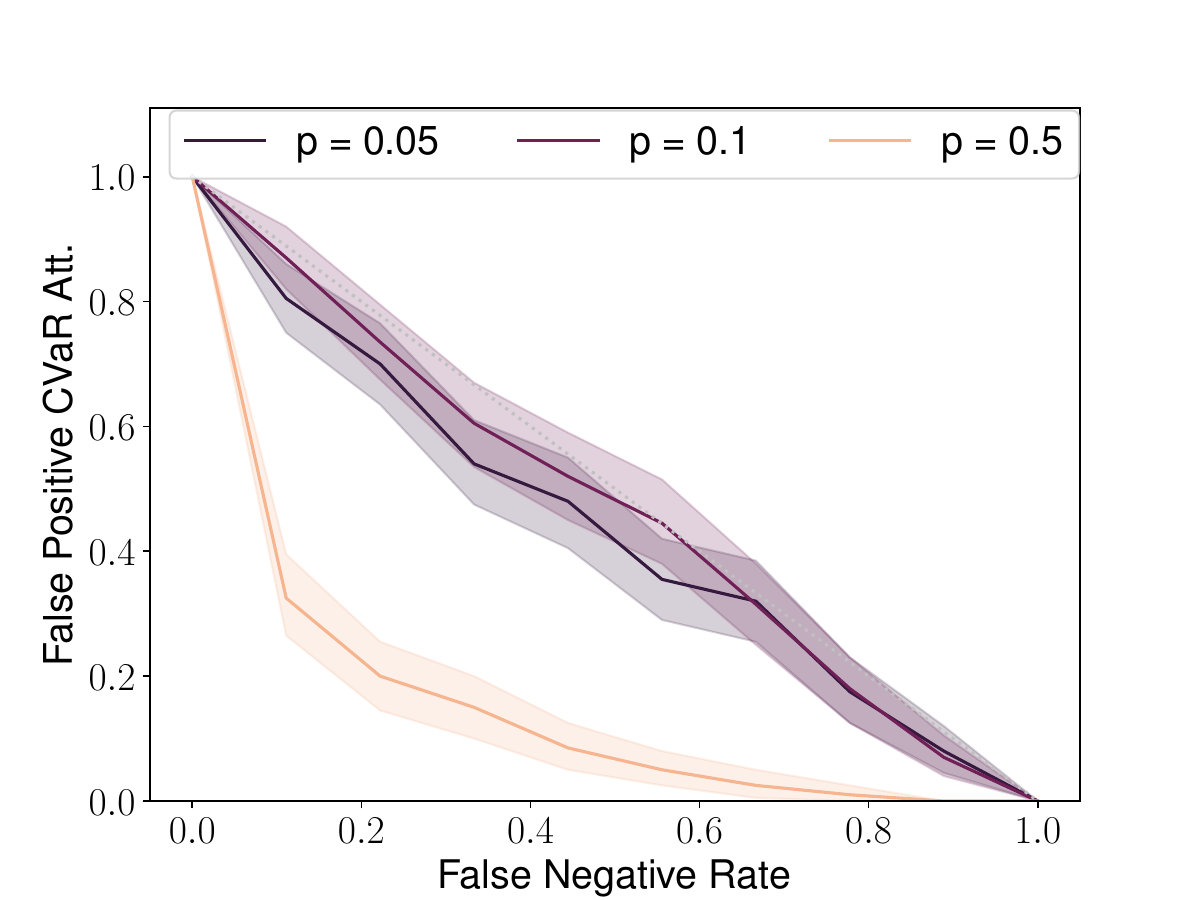}
    \caption{Attribute Specific}
  \end{subfigure}
  \caption{
  \revision{
  False positive rate at a given false negative rate for the hypothesis tests described in the paper. We use $d = 10$ sensitive groups, generating $1024$ groups and under a data budget of $n = 512$ Bernoulli realizations.
  We use Approach \ref{approach3} to distribute the per group quality of services.
  The false positive rate was computed using $1000$ realizations in the Monte Carlo method, and confidence intervals are plotted using Bootstrap from Seaborn \cite{Waskom2021} with $95\%$ confidence.
  }}
  \label{fig:FnFPp}
\end{figure}

\paragraph{The dependence on number of samples}
Figure \ref{fig:AUC_nsamples_p} shows how the area under the false negative versus positive curve (AUC) changes as the number of samples increases for the $\epsilon$-test in two scenarios with varying entropy for $\bw$. 
First, Figure \ref{fig:AUC_nsamples_p} shows the AUC for the baseline method \eqref{eq:maxgaptest} when performing an $\epsilon$-test for max-gap fairness (Figure \ref{fig:AUC_nsamples_p} (a)).
Second, it shows the AUC when performing an $\epsilon$-test for CVaR fairness using Algorithm \ref{alg:hypothesis_test} with weighted sampling with $\bv = \bw$ (Figure \ref{fig:AUC_nsamples_p} (b)), weighted sampling with $\bv \propto \bw^{2/3}$ (Figure \ref{fig:AUC_nsamples_p} (c)), and attribute-specific sampling (Figure \ref{fig:AUC_nsamples_p} (d)).
The per group quality of service is distributed using Approach \ref{approach3}.

First, we note that Figure \ref{fig:AUC_nsamples_p} shows that the performance of the baseline method for max-gap testing doesn't achieve reasonable performance across all entropy values and data budgets tested, i.e., AUC is always larger than $0.3$ for max-gap testing.
Meanwhile, when using Algorithm \ref{alg:hypothesis_test} with any of the proposed sampling methods for CVaR testing, we observe that their AUC is smaller than $0.2$ with just $n = 300$ samples, indicating the clear advantage of CVaR fairness testing in comparison with max-gap fairness.

Additionally, Figure \ref{fig:AUC_nsamples_p} shows that AUC decreases faster for CVaR fairness testing when using weighted sampling in comparison with the attribute specific sampling or the baseline method for max-gap fairness testing in most of the tested scenarios.
As we also observed in Figure~\ref{fig:alpha_dependency}, weighted sampling with $\bv \propto \bw^{2/3}$ has a flexible performance, i.e., it performs well under differentvalues of the entropy of $\bw$ and data budgets ($n$).

Figure \ref{fig:AUC_nsamples_p} also shows that using Algorithm \ref{alg:hypothesis_test} with attribute specific sampling has almost perfect performance ($\text{AUC} \approx 0$) when the groups are uniformly distributed, i.e., $p = 0.5$, even with $n = 100$ samples while testing for CVaR fairness across $1024$ groups.
When the groups are uniformly distributed, no other tested method achieves such AUC value even with $ n = 1500$ samples.
As we also observed from Figure \ref{fig:alpha_dependency}, these results indicates that when $\bw$ is close to uniform CVaR testing using attribute specific sampling offers the best performance across all tested groups while weighted sampling with $\bv \propto \bw^{2/3}$ is the best-performing sampling method when the entropy of $\bw$ is small.
In all cases, the baseline method \eqref{eq:maxgaptest} for max-gap fairness testing doesn't achieve reasonable performance.

\begin{figure}[t]
  \centering
  \begin{subfigure}{0.45\linewidth}
    \includegraphics[width=\linewidth]{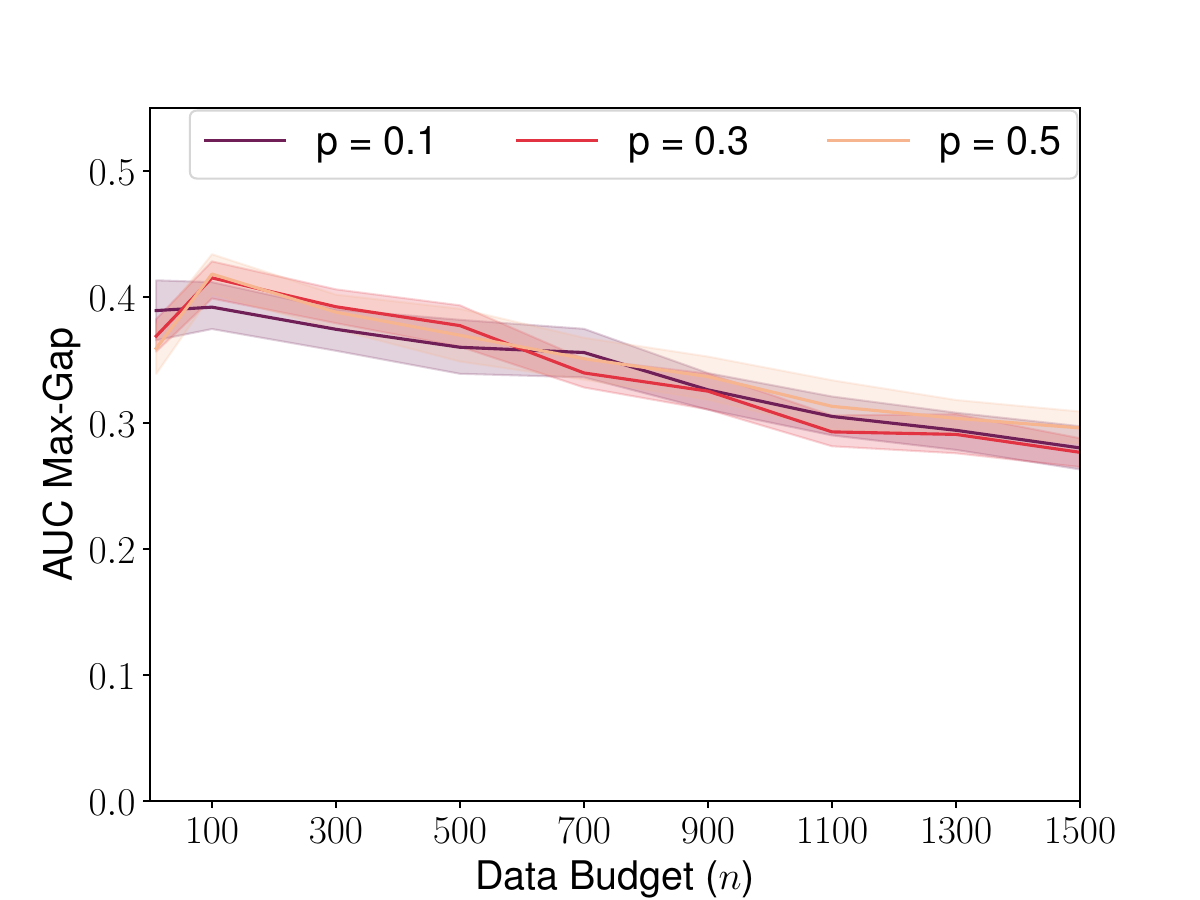}
    \caption{Max-Gap}
  \end{subfigure}
  \begin{subfigure}{0.45\linewidth}
    \includegraphics[width=\linewidth]{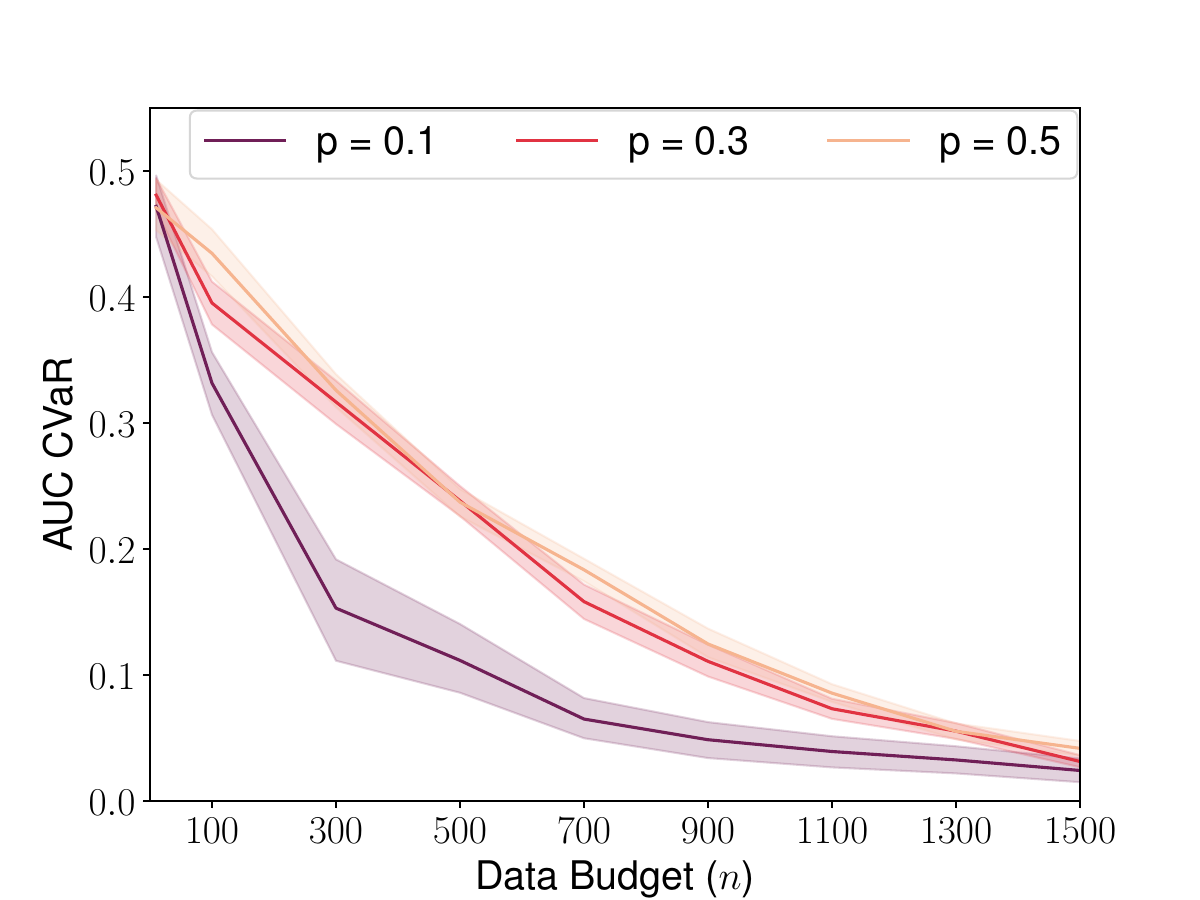}
    \caption{$\bw$ Weighted}
  \end{subfigure}
  \begin{subfigure}{0.45\linewidth}
    \includegraphics[width=\linewidth]{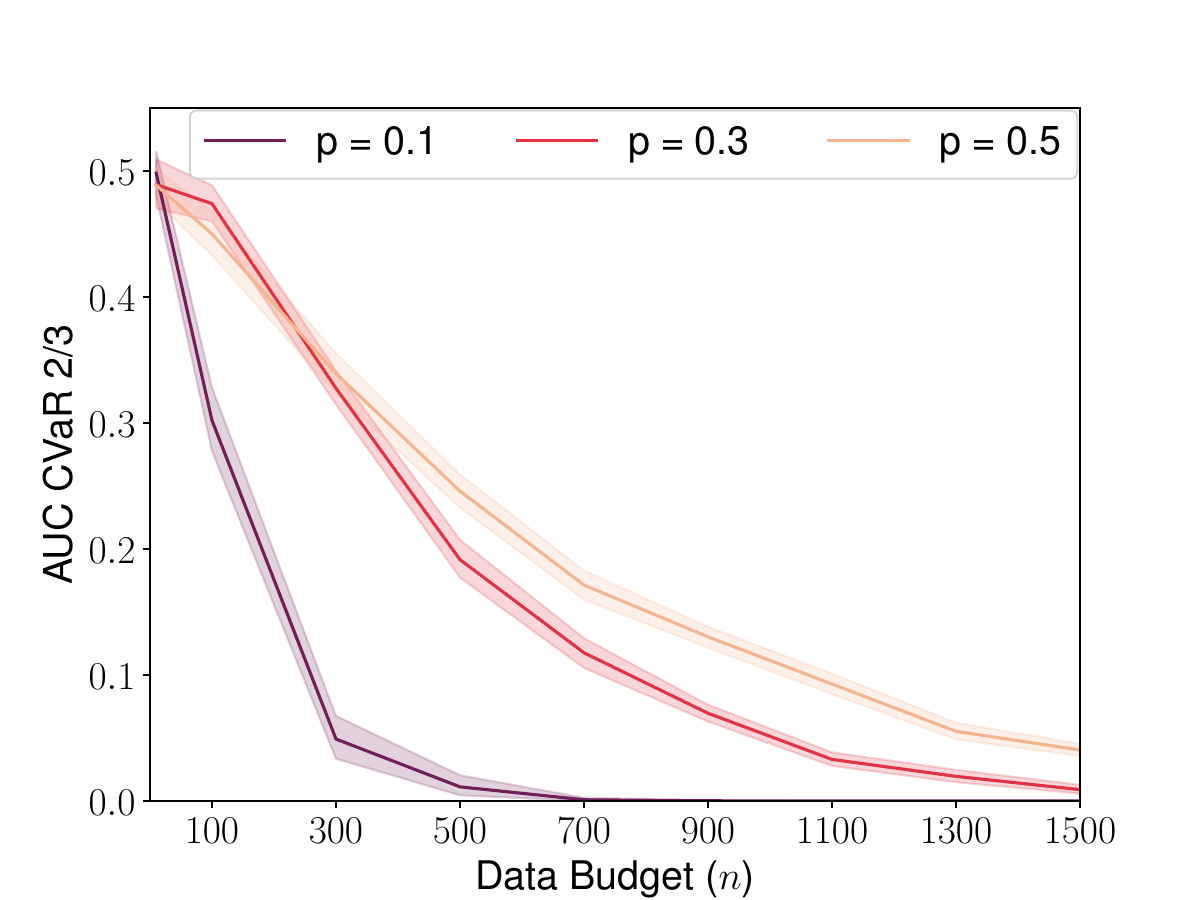}
    \caption{$\bw^{2/3}$ Weighted}
  \end{subfigure}
  \begin{subfigure}{0.45\linewidth}
    \includegraphics[width=\linewidth]{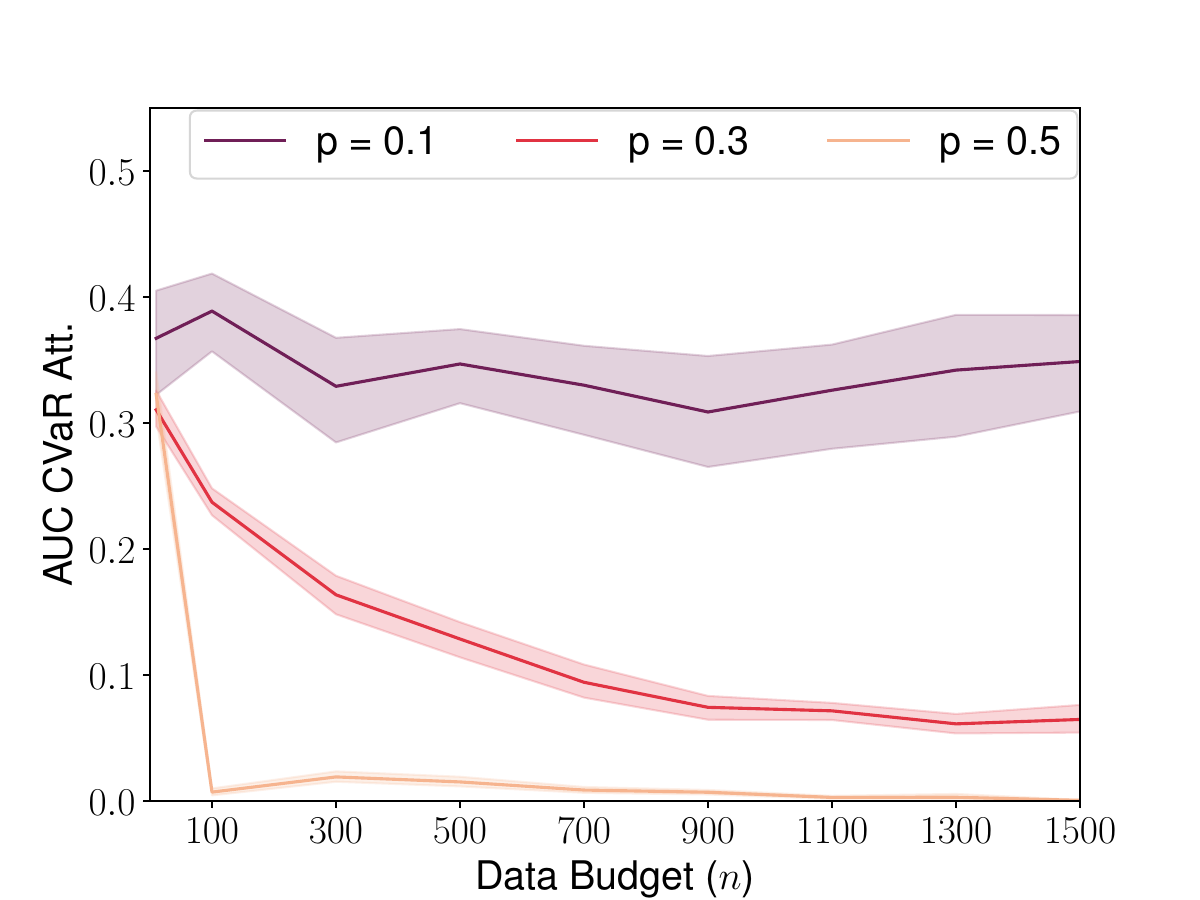}
    \caption{Attribute Specific}
  \end{subfigure}

  \caption{
  \revision{
  Area under the false negative vs. false positive curve (AUC) versus the number of samples used to perform the $\epsilon$-test.
  We vary the entropy by changing the group distribution Bernoulli parameter $p$ with a fixed number of sensitive attributes $d = 10$, generating $1024$ groups and data budget in $x$-axis.
  We use $100$ samples in the Monte Carlo method and repeat this process $20$ times to estimate AUC and plot confidence intervals using Bootstrap from Seaborn \cite{Waskom2021} with $95\%$ confidence.
  }}
  \label{fig:AUC_nsamples_p}
\end{figure}

\paragraph{Maximum $\gp$ for Reliably Auditing} Figure \ref{fig:max_groups} uses the lower bound on the error probability from Proposition \ref{prop:hyptest_lower_bound} and Theorem \ref{thm:converse_CVaR}.
In both cases, we first showed a lower bound for the minimum probability of error in a $\epsilon$-test that uses i.i.d. samples from the data distribution; then we computed the maximum number of samples to make the lower bound smaller than a fixed error probability. 
Here, we take a different approach, fixed an $\epsilon$ threshold and the number of samples used to perform the $\epsilon$-test; we found the maximum number of groups that can be used while ensuring that there exists a decision function that achieves at least an error probability of $45\%$ (almost a random guess).
Specifically, we denote the maximum number of groups that can be used for reliably (error probability smaller than $45\%$) test for max-gap fairness as $G_{\texttt{max-gap}}(n, \epsilon)$ and for CVaR fairness as $G_{\texttt{CVaR}}(n, \epsilon, \alpha)$ and define then next.

In Proposition \ref{prop:hyptest_lower_bound}, we showed that the minimum probability of error is lower bounded by \eqref{eq:prob_error_max_gap}, and if we take it to be $0.45$ we have that the number of groups $|\gp|$ is such that
\begin{equation}
    2\inf_{\psi_{\epsilon}} P_{\text{error}} = 0.9 \geq 1 - \left[2\left(1 - \left(1 - \frac{2\epsilon^2}{|\gp|}\right)^{n}\right)\right]^{1/2}.
    \label{eq:prob_error_max_gap}
\end{equation}
Then, the maximum number of groups we can have while ensuring that the probability of error in testing for max-gap fairness is at most $45\%$ is given by
\begin{equation}
    G_{\texttt{max-gap}}(n, \epsilon) \triangleq \argmax_{|\gp|} \left\{ |\gp| \in \mathbb{N} \Big| 0.9 \geq 1 - \left[2\left(1 - \left(1 - \frac{2\epsilon^2}{|\gp|}\right)^{n}\right)\right]^{1/2} \right\}.
    \label{eq:max_gap_max_groups}
\end{equation}

Similarly, using the lower bound for the probability of error when performing an $\epsilon$-test for CVaR fairness proved in Theorem \ref{thm:converse_CVaR}, the maximum number of groups to ensure that the probability of error is at most $45\%$ is given by
\begin{equation}
        G_{\texttt{CVaR}}(n, \epsilon, \alpha) \triangleq \argmax_{|\gp|} \left\{ |\gp| \in \mathbb{N} \Big| 0.9 \geq 1 - \frac{1}{2} \sqrt{e^{ \frac{ 1024 (1 - \alpha) n^2 \epsilon^4}{\alpha^{4} |\gp|  }} - 1} \right\}.
        \label{eq:cvar_max_groups}
\end{equation}

We show $G_{\texttt{max-gap}}(n, \epsilon)$ and $G_{\texttt{CVaR}}(n, \epsilon, \alpha)$ in Figure \ref{fig:max_groups}.

Remarkably, when $50k$ samples are available, the maximum number of groups a model can have while it is still possible to perform an $0.1$-test using max-gap fairness with a probability of error of at most $45\%$ (almost a random guess) is $ \lfloor 2^{17.6} \rfloor$ (Figure \ref{fig:max_groups} (a)) while, under the same scenario, it is possible to perform an $0.1$-test using CVaR fairness with $ \lfloor 2^{25.2} \rfloor$ using $\alpha = 0.9$ (Figure \ref{fig:max_groups} (b)).
In practice, this result implies that it is impossible (without further assumption on the dataset distribution) to determine if a machine learning model that was deployed using $18$ binary sensitive attributes has a max-gap fairness bigger than $0.1$.
On the other hand, using CVaR fairness testing, it is possible to reliably test if a machine-learning model that was deployed using $25$ binary sensitive attributes has CVaR fairness bigger than $0.1$ when we set $\alpha = 0.9$ (the number of sensitive attributes grows with $\alpha$).
Therefore, our bounds indicate that CVaR fairness allows practitioners to deploy machine learning models with more sensitive attributes while reliably ensuring that the model is fair according to CVaR fairness.

\begin{figure}[t]
  \centering
  \begin{subfigure}{0.45\textwidth}
    \includegraphics[width=\linewidth]{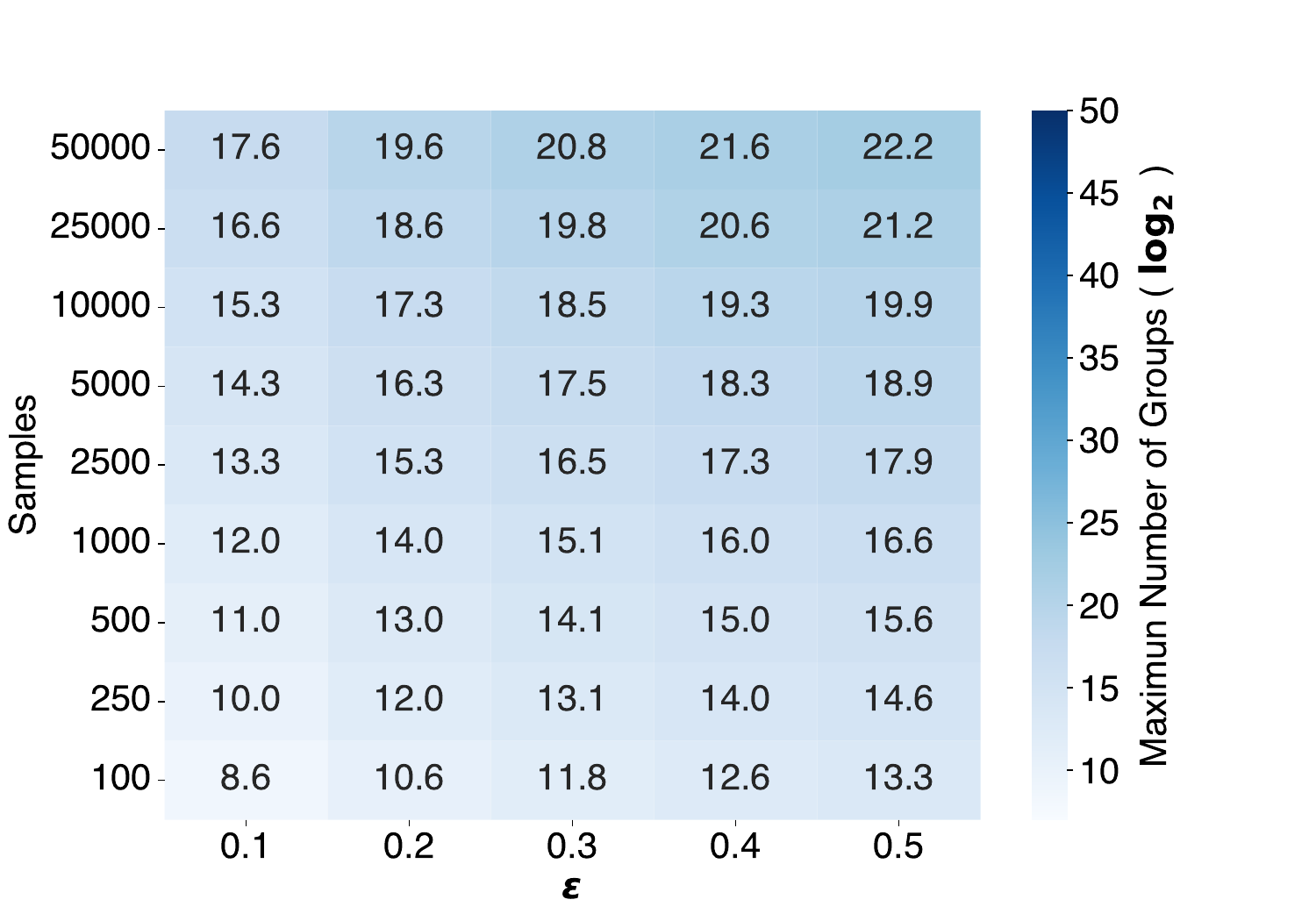}
    \caption{Max-Gap}
  \end{subfigure}
  \begin{subfigure}{0.45\textwidth}
    \includegraphics[width=\linewidth]{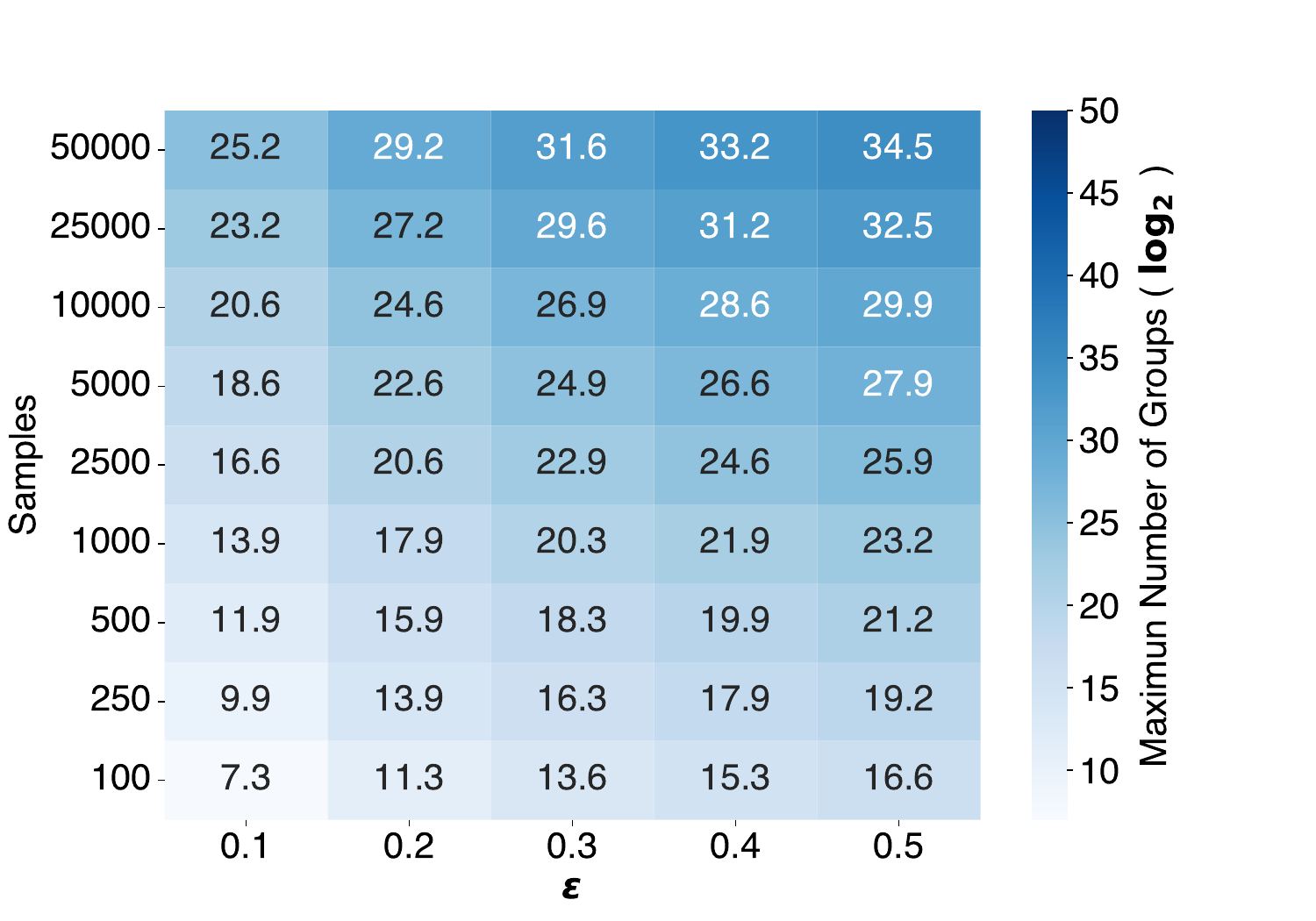}
    \caption{CVaR ($\alpha = 0.9$)}
  \end{subfigure}
  \caption{
  \revision{
  The maximum number of groups $\gp$ (z-axis) for each choice of threshold $\epsilon$ (x-axis) and sample size (y-axis) to ensure that the probability of error in the hypothesis test from Definition \ref{def:Hypothesis_Test} is smaller than $45\%$.
  Figure (a) shows the maximum number of groups when testing if max-gap fairness is bigger than $\epsilon$.
  Figures (b) the maximum number of groups when testing if CVaR fairness is bigger than $\epsilon$ using an $\alpha$ value of $0.9$.
  The maximum number of groups was computed using \eqref{eq:max_gap_max_groups} for max-gap and \eqref{eq:cvar_max_groups} for CVaR fairness.
  }}
  \label{fig:max_groups}
\end{figure}

\paragraph{Concluding remarks from experiments}
The experiments indicate that testing for CVaR fairness has a more favorable trade-off between sample complexity and the number of demographic groups in comparison with max-gap fairness, as our theoretical results demonstrate.
Therefore, when the data budget to audit a model for fairness is not large enough (smaller than the number of groups) for max-gap fairness testing, we advocate for practitioners to use CVaR fairness with the largest possible $\alpha$ that still maintains the sample complexity under their data budget --- result showed in Figure \ref{fig:max_groups}.
Moreover, when the group weight vector $\bw$ has high R\'enyi entropy, our results indicate that attribute specific sampling is more efficient in testing for CVaR fairness --- as our theoretical results show given that the sample complexity of attribute specific does not increase with the entropy of $\bw$.
When the entropy of $\bw$ is not too large, using weighted sampling with $\bv \propto \bw^{2/3}$ yields an accurate $\epsilon$-test even under a small data budget as shown in Figure \ref{fig:AUC_nsamples_p} --- this is also expected from our theoretical results in Theorem \ref{thm:weighted_upper_general}.
Finally, when only i.i.d. sampling from the data distribution is available, i.e., weighted sampling with $\bv = \bw$, our numerical results also indicate that CVaR fairness testing has a more favorable sample complexity than the baseline max-gap fairness testing.

}

\section{Concluding remarks}
\emph{Conclusion.} This paper proposes a multi-group fairness metric called conditional value-at-risk Fairness (CVaR fairness) inspired by conditional value-at-risk in finance.
In fair machine learning, we are interested in testing if a given fairness metric is greater than a fixed threshold or equal to zero.
We show that the usual multi-group fairness metrics (we call it max-gap fairness) have a high sample complexity, making testing for multi-group fairness infeasible.
This motivates CVaR fairness, a metric that allows practitioners to relax fairness guarantees to decrease the sample complexity for reliably testing for multi-group fairness.
We prove that it is possible to recover max-gap fairness from CVaR fairness by controlling one parameter in CVaR fairness.
Next, we propose algorithms that test if the CVaR fairness is bigger than a threshold or equal to zero.
We show that the proposed algorithms can reliably perform the hypothesis test with a more favorable sample complexity using an i.i.d. sample strategy (weighted sampling), improving the previous sample complexity in at least a square root factor.
Moreover, we provided a non-i.i.d. sampling strategy (attribute specific sampling) that, paired with our proposed algorithm, can test for CVaR fairness with a test dataset size independent of the number of groups.
Finally, we also provided a converse result on testing for CVaR fairness using weighted sampling, showing that the order of the sample complexity of the proposed algorithm is optimal under some assumptions.
Our results show that by using CVaR fairness and the proposed algorithms, ML practitioners can reliably test for multi-group fairness by trading off fairness guarantees for a more favorable sample complexity.

\secondrev{\emph{Ethical Considerations.} We propose CVaR fairness as a relaxation of max-gap fairness. CVaR fairness achieves a more favorable sample complexity relative to max-gap fairness by providing a weaker fairness guarantee. We highlight that the parameter $\alpha$ captures this trade-off and can be tuned to fit both a fairness requirement and a sample size constraint. We recommend that practitioners choose this parameter responsibly and be mindful that small values of $\alpha$ may decrease fairness guarantees.}

\emph{Future Work.} We hope that future work explores the empirical implications of our testing proposal to demonstrate the trade-off between sample complexity and fairness guarantees from CVaR fairness.
Additionally, by shifting from max-gap fairness to CVaR fairness, we are able to gradually improve sample complexity but this could come at the cost of possibly overlooking performance gaps for uniquely rare and unsampled groups.
While we believe this should be rare in practice and our framework is tunable, it would be valuable for future work to empirically examine this on multiple real-world use cases.
Moreover, our results only hold for binary group-specific losses, i.e., $L \in \{0, 1\}$. 
Hence, developing theoretical guarantees for multi-group fairness notions with non-binary $L$, such as multicalibration is another direction for future investigation.
\revision{Finally, we point out that our results only hold for a finite (although as large as desired) number of groups. 
However, providing theoretical guarantees for testing for multi-group fairness in the presence of uncountable continuous groups (e.g., human height) is also of interest in the community \cite{Johnson18-multicalibration, kearns2018preventing}.
We highlight that CVaR fairness (Definition \ref{def:CVaRFairness}) can be straightforwardly generalized for continuous groups. However, the theoretical guarantees we provide don't automatically extend to handle such a definition.
For this reason, we leave this contribution as future work.
}

\section{Acknowledgetments}

The authors thank Flavien Prost (Google Research) and Alexandru Tifrea (ETHZ) for their helpful comments and discussions on an earlier version of this paper. We also thank the anonymous reviewers for their careful contribution.

\bibliographystyle{IEEEtran}
\bibliography{citations.bib}
\clearpage

\appendices

 \secondrev{
\section{Table of Notation}

\begin{table}[htb!]
\begin{center}
\caption{\secondrev{Table of Notation}}
\label{table}
\secondrev{
\setlength{\tabcolsep}{3pt}
\begin{tabular}{c|p{200pt}}
\hline
\hline 
\textbf{Symbol}& \textbf{Meaning} \\
\hline
\hline
$\Delta^{d}$ & $d$-dimensional simplex\\
$ x $&  feature vector \\
$ y $&  true label \\
$ \hat{y} $ &  predicted label \\
$ \gp $ &  set of all protected groups\\
$ g $&  specific protected group\\
$ z $ &  tuple $z = (x, g, y, \hat{y})$ \\
$ \widetilde{P}_Z $ &  distribution of samples $z$\\
$ n $& data budget \\
$ \mathcal{D}_{\text{audit}} $ &  audit data $\mathcal{D}_{\text{audit}} = \{z_i\}_{i= 0}^{n}$ \\
$ \bw = (w_1, ..., w_{|\gp|})$ &  stochastic vector of group distribution (also called group weight vector)\\
$ \bv = (v_1, ..., v_{|\gp|})$ &  stochastic vector that represents group marginal from the audit dataset\\
$ \epsilon $ & hypothesis test threshold \\
$ P_{\text{error}} $ & probability of error in the hypothesis test\\
$ \gtg{L} $ & quality of service function \\
$ \averagegtg{L} $ & average quality of service \\
$\Delta(g; P, L, \bw)$ & per group fairness gap \\
$\fairness$ & max-gap fairness \\
$\alpha$ & group percentile threshold \\
$\texttt{F\_CVaR}_{\alpha}$ & CVaR fairness \\
$\widehat{F}(\bw)$ & CVaR fairness estimator \\
\hline
\hline
\end{tabular}
}
\label{tab1}
\end{center}
\end{table}

}
\section{ Lower Bounds for Hypothesis Testing for Max-Gap Fairness}
\label{sec:lower_bd_max_gap}

This section shows the lower bounds for the hypothesis test using the standard max-gap fairness metrics in Definition \ref{def:multi_group_fairness}.
Particularly, we focus on Equal Opportunity fairness. The results here were discussed in Section \ref{sec:max-gap_fairness}.

Recall that the Hellinger divergence between two distributions $P, Q$ on $\cZ$ is given by 
    \begin{align*}
     \textsc{H}^2(P || Q)  &= \sum_{z \in \cZ} \left(  \sqrt{p(z)} - \sqrt{q(z)}\right)^2.
    \end{align*}
\begin{lemma}[Close Distributions] 
    For max-gap fairness metrics $\fairness$ (e.g., equal opportunity in Example \ref{eg:equal_opportunity} and statistical parity in Example \ref{eg:statistical_parity}) with quality of service function $L$ and measuring the average quality of service using the distribution $P$ such that $(L, P) \in \{(L_{EO}, P_{EO}), (L_{SP}, P_{SP})\}$.
    If $\epsilon \leq 0.5$ there exists a distribution $P_0 \in \hypregions_0$ and $P_1 \in \hypregions_1(\epsilon)$ such that:
    \begin{equation}
        \text{H}^2(P_0 || P_1) \leq 2 - 2\left(1 - \frac{2\epsilon^2}{|\gp|}\right).
    \end{equation}
    \label{lem:CloseDistributions_apx}
\end{lemma}

\begin{proof}
We will prove the desired result for equal opportunity described in Example \ref{eg:equal_opportunity}.
The result for other fairness metrics, such as statistical parity in Example \ref{eg:statistical_parity}, and equal odds \cite{hardt2016equality}, is analogous.
We start by assuming that $$P_0(Z|(\hat{Y}, Y, G)) = P_0(X|(\hat{Y}, Y, G)) = P_1(Z|(\hat{Y}, Y, G)) = P_1(X|(\hat{Y}, Y, G)),$$ i.e., the conditional distribution of $X$ from $P_0$ and $P_1$ are equal
\begin{equation}
    P_0(X|(\hat{Y}, Y, G)) = P_1(X|(\hat{Y}, Y, G)).
    \label{eq:bound_1_in_lem_1}
\end{equation}
Now, from \eqref{eq:bound_1_in_lem_1}, we conclude that it is only necessary to bound the Hellinger divergence for the marginal distributions of $(\hat{Y}, Y, G)$, because
\begin{align*}
     \textsc{H}^2(P_0 || P_1)  & = \EE_Q\left[\left(  \sqrt{\frac{P_1(Z)}{P_0(Z)}} - 1\right)^2\right] = \EE\left[\left(  \sqrt{\frac{P_1(\hat{Y}, Y, G)}{P_0(\hat{Y}, Y, G)}} - 1\right)^2\right]. 
\end{align*}
Next, we compute the divergence of the marginals. 
\begin{align}
    & \ \ \textsc{H}^2(P_0 || P_1)\\
     &= \sum_{ \substack{ g \in \gp\\ y \in \{0, 1\} \\ \hat{y} \in \{0, 1\}}} \biggl(  \sqrt{P_0(\hat{Y} = \hat{y}, Y = y ,G = g)}   - \sqrt{P_1(\hat{Y} = \hat{y}, Y = y ,G = g)} \biggl)^2 \\
     \nonumber
     & = \sum_{ \substack{ g \in \gp\\ y \in \{0, 1\} \\ \hat{y} \in \{0, 1\}}} \biggl(  \sqrt{P_0(\hat{Y} = \hat{y} \mid Y = y ,G = g)P_0(Y = y \mid G = g)P_0(G = g)}  \\
     & - \sqrt{P_1(\hat{Y} = \hat{y} \mid Y = y , G = g)P_1(Y = y \mid G = g)P_1(G = g)} \biggl)^2 \\
     \nonumber
    & = \sum_{ \substack{ g \in \gp\\ \hat{y} \in \{0, 1\}}} \biggl(  \sqrt{P_0(\hat{Y} = \hat{y} \mid Y = 1 , G = g)P_0(G = g)}  \\
    &- \sqrt{P_1(\hat{Y} = \hat{y} \mid Y = 1 , G = g)P_1(G = g)} \biggl)^2 \\
    \nonumber
    & = \sum_{ \substack{ g \in \gp\\ \hat{y} \in \{0, 1\}}}   P_0(\hat{Y} = \hat{y} \mid Y = 1 , G = g)w_g + P_1(\hat{Y} = \hat{y} \mid Y = 1 , G = g){w}_g\\ 
    & - 2w_g\sqrt{P_1(\hat{Y} = \hat{y} \mid Y = 1 , G = g)P_0(\hat{Y} = \hat{y} \mid Y = 1 , G = g)}  \\
    \nonumber
    &= 2  - 2\sum_{ \substack{ g \in \gp}} w_g\biggl(\sqrt{P_1(\hat{Y} = 1 \mid Y = 1 , G = g)P_0(\hat{Y} = 1 \mid Y = 1 , G = g)} \\
    \label{eq:bound_2_in_lem_1}
    & + \sqrt{P_1(\hat{Y} = 0 \mid Y = 1 , G = g)P_0(\hat{Y} = 0 \mid Y = 1 , G = g)}\biggl)
\end{align}

Now, take $w_g = \frac{1}{|\gp|}$. Also, let $P_0$ and $P_1$ be :
\begin{equation}
    P_0(\hat{Y} = \hat{y} \mid Y = 1 , G = g) = \frac{1}{2} \ \ \forall g \in \gp.
    \label{eq:bound_3_in_lem_1}
\end{equation}
\begin{align}
    P_1(\hat{Y} &= \hat{y} \mid Y = 1 , G = g) = \frac{1}{2} \ \ \forall g \in \gp \setminus \{g_1\}.\\
    P_1(\hat{Y} &= \hat{y} \mid Y = 1 , G = g_1) = \frac{1}{2} + \epsilon \frac{|\gp|}{|\gp|-1}.
    \label{eq:bound_4_in_lem_1}
\end{align}
Note that $P_0 \in \hypregions_0$ and $P_1 \in \hypregions_1(\epsilon)$.

Therefore, by plug in the distributions in \eqref{eq:bound_3_in_lem_1} and \eqref{eq:bound_4_in_lem_1} in the equality from \eqref{eq:bound_2_in_lem_1}, we have that
\begin{align*}
  \twocol{  & \ \  }\textsc{H}^2(P_0 || P_1) \twocol{\\ }
      &= 2  - 2\biggl( 1 + \frac{1}{|\gp|}\biggl(\sqrt{\frac{1}{2}\left(\frac{1}{2} + \frac{|\gp|}{|\gp|-1}\epsilon\right)} + \sqrt{\frac{1}{2}\left(\frac{1}{2} - \frac{|\gp|}{|\gp|-1}\epsilon\right)} - 1\biggl)\biggl) \\
    & \leq 2  - 2\biggl( 1 + \frac{1}{|\gp|}\biggl(\sqrt{\frac{1}{2}\left(\frac{1}{2} + \epsilon\right)} + \sqrt{\frac{1}{2}\left(\frac{1}{2} - \epsilon\right)} - 1\biggl)\biggl) \\
    &\leq 2 - 2 \left( 1 - \frac{2\epsilon^2}{|\gp|}\right),
\end{align*}
which completes the proof.
\end{proof}

Recall that the total variation divergence $\textsc{TV}( . || .)$ takes a pair of distributions $P$ and $Q$ and returns
\begin{equation}
    \textsc{TV}( P || Q) \triangleq \frac{1}{2} \sum_{z \in \cZ} \left| P(z) - Q(z)\right|.
\end{equation}

\HypTestLowerBound*

\begin{proof}
We will prove the desired result for equal opportunity described in Example \ref{eg:equal_opportunity}.
The result for other fairness metrics, such as statistical parity in Example \ref{eg:statistical_parity}, and equal odds \cite{hardt2016equality}, is analogous. 

We start with $n$ samples from two distributions $P_0 \in \hypregions_0$ and $P_1 \in \hypregions_1(\epsilon)$. 
Recall that the probability of error is given by
\begin{equation}
    P_{\text{error}}(\psi) = \frac{1}{2}(P_0 (\psi = 1) + P_1 (\psi = 0) ).
\end{equation}

Therefore, by choosing the best decision function $\psi$, by the Le Cam lemma \cite{LeCam1973}, and the fact that for all distributions $P$ and $Q$
\begin{equation*}
    \textsc{TV}(P || Q) \leq \textsc{H}(P || Q),
\end{equation*}
we conclude that
    \begin{align*}
         \inf_{\psi}  2 P_{\text{error}}(\psi) &=  1 - \textsc{TV}(P_0^n || P_1^n) \\
        & \geq  1 - \sqrt{\textsc{H}^2(P_0^n || P_1^n)}\\
        & \geq  1 - \biggl(2 \left(1 - \left(1 - \frac{1}{2}\textsc{H}^2(P_0 || P_1)\right)^n\right)\biggl)^{1/2}
    \end{align*}
    From Lemma ~\ref{lem:CloseDistributions_apx} we have that there exists $P_0 \in \mathcal{P}_0$ and $P_1 \in \mathcal{P}_1(\epsilon)$ such that 
    \begin{equation}
        \text{H}^2(P_0 || P_1) \leq 2 - 2\left(1 - \frac{2\epsilon^2}{|\gp|}\right), 
    \end{equation}
    and $\bw$ uniform.
    
    Taking the $P_0 \in \mathcal{P}_0$ and $P_1 \in \mathcal{P}_1(\epsilon)$ from Lemma ~\ref{lem:CloseDistributions_apx}, we have that
    \begin{align}
    \twocol{& \ \ } \inf_{\psi}  2 P_{\text{error}}(\psi) \twocol{\\}
        &\geq 1 - \biggl(2 \left(1 - \left(1 - \frac{1}{2}\textsc{H}^2(P_0 || P_1)\right)^n\right)\biggl)^{1/2} \nonumber \\
        & \geq 1 - \biggl(2 \left(1 - \left(1 - \frac{2\epsilon^2}{|\gp|}\right)^n\right)\biggl)^{1/2}   \label{eq:bound_2_in_cor_1}
    \end{align}
We now provide the lower bound on sample complexity. We start to differentiating between the distributions $P_0 \in \mathcal{P}_0$ and $P_1 \in \mathcal{P}_1(\epsilon)$ from Lemma ~\ref{lem:CloseDistributions_apx}.
We want to ensure that 
\begin{equation}
    \delta \geq  \inf_{\psi}  P_{e}(\psi).
    \label{eq:bound_1_in_cor_1}
\end{equation}

Next, we compute how large $n$ needs to be to ensure that \eqref{eq:bound_2_in_cor_1} holds. 
    \begin{align}
       &2\delta \geq 1 - \biggl(2 \left(1 - \left(1 - \frac{2\epsilon^2}{|\gp|}\right)^n\right)\biggl)^{1/2}\\
        &\Rightarrow 1 - \frac{(1 - 2\delta)^2}{2} \geq \left(1 - \frac{2\epsilon^2}{|\gp|}\right)^n\\
        &\Rightarrow \ln{ \left(1 - \frac{(1 - 2\delta)^2}{2}\right)} \geq n \ln\left(1 - \frac{2\epsilon^2}{|\gp|}\right)\\
        &\Rightarrow \ln{ \left(1 - \frac{(1 - 2\delta)^2}{2}\right)} \geq n \frac{-2\epsilon^2}{-2\epsilon^2 + |\gp|}\\
        &\Rightarrow n  \geq  \ln{ \left(1 - \frac{(1 - 2\delta)^2}{2}\right)}\frac{-2\epsilon^2 + |\gp|}{-2\epsilon^2}\\
        &\Rightarrow n  =  \Omega\left(\frac{|\gp|}{\epsilon^2}\right)
    \end{align}

Hence, we have the desired result. 
\end{proof}

\section{Properties of CVaR Fairness}
\label{sec:properties_CVaR}
In this section, we prove the properties from CVaR fairness discussed in Section \ref{sec:CVaR}.

\begin{lemma}[$\fairness$ is bounded] 
For all distributions $P$ with support on $\cZ$, if $L: \cZ \rightarrow \{0, 1\}$ then
\begin{equation}
    0 \leq \fairness (P, L, \averagegtg{L}) \leq 1.
\end{equation}
\label{lem:apx_bounded_CVaR}
\end{lemma}
\begin{proof}
    Recall that
    \begin{align}
        \Delta(g) &= \left| \EE_P[L(Z) | G = g] - \averagegtg{L} \right|\\
        &= \left| \EE_P[L(Z) | G = g] - \sum_{g \in \gp} w_g \EE_P[L(Z) | G = g] \right|.
    \end{align}
    Where $\averagegtg{L} \in [0, 1]$ and $\EE_P[L(Z) | G = g] \in [0, 1]$.
    Therefore, 
    \begin{align}
        \Delta(g) \leq 1 \Rightarrow \fairness = \max_{g} \Delta(g) &\leq 1.
        \label{eq:bound_1_in_prop_3}
    \end{align}
\end{proof}

\CVarLowerbound*

\begin{proof}
    From the definition of CVaR fairness we know that:
    \begin{align*}
        \texttt{F\_CVaR}_{\alpha}(\bw) &= \frac{1}{(1 - \alpha)}\max_{\sum_{g \in Q \subset \gp} w_g \leq (1 - \alpha)} \sum_{g \in Q} w_g \Delta(g) \geq 0
    \end{align*}
    because from the definition of $\Delta(g)$, $\Delta(g) \geq 0$ for all $g \in \gp$ and $w_g \geq 0$ for all $g \in \gp$.

    Moreover, we have that:
    \begin{align*}
         \texttt{F\_CVaR}_{\alpha}(\bw) &= \frac{1}{(1 - \alpha)} \max_{\sum_{g \in Q \subset \gp} w_g \leq (1 - \alpha)} \sum_{g \in Q} w_g \Delta(g)\\
         & \leq\frac{1}{(1 - \alpha)} \max_{\sum_{g \in Q \subset \gp} w_g \leq (1 - \alpha)}\sum_{g \in Q} w_g \max_{g \in \gp}\Delta(g)\\
         &= \max_{g \in \gp}\Delta(g) \max_{\sum_{g \in Q \subset \gp} w_g \leq (1 - \alpha)}\frac{1}{(1 - \alpha)} \sum_{g \in Q} w_g\\
         & \leq \max_{g \in \gp}\Delta(g)\\
         & = \fairness(P, L, \averagegtg{L}).
    \end{align*}
    
\end{proof}

\RecoverFairnessKnwingG*

\begin{proof}
    From proposition \ref{prop:CVaRLowerBound}, we already know that 
    \begin{equation*}
        \texttt{F\_CVaR}_{\alpha}(\bw) \leq \fairness(L, \averagegtg{L}).
    \end{equation*}
    However, by taking $\alpha = 1 - w_{g^*}$, the set $Q = \{ w_{g^*}\}$ for some $g^* \in \argmax_g \Delta(g)$ achieves the maximum
    \begin{align*}
        \frac{1}{(1 - \alpha)} \sum_{g \in Q} w_g \Delta(g) &= \frac{1}{w_{g^*}} w_{g^*} \Delta(g^*)\\
        &=\Delta(g^*)\\
        &= \fairness(L, \averagegtg{L}).
    \end{align*}
    Moreover, $\sum_{g \in Q \subset \gp} w_g  = w_{g^*} \leq w_{g^*} = 1 - \alpha$. Hence, the desired result follows from the fact that $Q$ achieves the maximum because
    \begin{align}
        \fairness(L, \averagegtg{L}) &= \frac{1}{(1 - \alpha)} \sum_{g \in Q} w_g \Delta(g) \\
        &\leq \frac{1}{(1 - \alpha)}  \max_{\sum_{g \in Q \subset \gp} w_g  
         \leq  (1 - \alpha)} \sum_{g \in Q} w_g \Delta(g) \\
         & = \texttt{F\_CVaR}_{\alpha}(\bw) \leq \fairness(L, \averagegtg{L}).
    \end{align}

    Note that if $\bw = \left( \frac{1}{|\gp|}, ..., \frac{1}{|\gp|}\right)$, then $\alpha = 1 - w_{g^*} = 1 - \frac{1}{|\gp|}$.
    
\end{proof}

\section{Upper Bound for CVaR Testing}
\label{sec:upper_bound_CVaR}

This section proves the results in Section \ref{sec::testing}.

\begin{lemma}
\label{lem:apx_weight_bounds}
Under weighted sampling, if $nv_g \leq 1$ and $n \geq 2$,
\[
\Pr(M_{ {g}} \geq 1) \geq  n v_g/e,
\]
and 
\[
\Pr(M_{ {g}} \geq 2) \geq n^2 v^2_g / 4e.
\]
\end{lemma}
\begin{proof}
Since  $nv_g \leq 1$,
\[
\Pr(M_{ {s}} \geq 1) \geq \Pr(M_{ {s}} = 1)=  n v_g (1-v_g)^{n-1} \geq n v_g/e,
\]
Similarly $nv_g \leq 1$,
\[
\Pr(M_{ {s}} \geq 2) \geq \Pr(M_{ {s}} = 2)=  n(n-1)/2 (v_g)^2 (1-v_g)^{n-2} \geq n^2 v^2_s / 4e,
\]
where the last inequality follows from $n \geq 2$.
\end{proof}

\begin{lemma} [Computing Average of First Order Estimation] Define $\frac{\sum_{i = 1}^{M_g} \gtg{L}(z_i)}{M_g} = 0$ if $M_g = 0$ and recall that $z = (x, g, y)$.
If $L: \cZ \rightarrow \{0,1\}$, then
\begin{equation}
    \EE \left[ \frac{\sum_{i = 1}^{M_g} \gtg{L}(z_i)}{M_g}\right] = \EE [\gtg{L}(z) \mid G = g] \Pr(M_g \geq 1)
\end{equation}
\label{lem::apx_avg-first-order}
\end{lemma}

\begin{proof}
By the conditional law of expectation,
\begin{align*}
\EE\left[\frac{\sum_{i = 1}^{M_g} \gtg{L}(z_i)}{M_{ {g}}}\right] &=  \EE\left[\left.\EE\left[\frac{\sum_{i = 1}^{M_g} \gtg{L}(z_i)}{M_{ {g}}}\right|M_{ {g}} = m_g \right]\right].
\end{align*}

If $m_{ {g}} = 0$, by definition,
\begin{equation}
    \left.\EE\left[\frac{\sum_{i = 1}^{M_g} \gtg{L}(z_i)}{M_{ {g}}}\right|M_{ {g}} = m_{ {g}}\right] =0.
\end{equation}

If $m_{ {g}} \geq 1$, then
\begin{align}
\twocol{& }\EE\left[\left.\frac{\sum_{i = 1}^{M_g} \gtg{L}(z_i)}{M_{ {g}}}\right|M_{ {g}} = m_{ {g}} \right] \twocol{\\}
= & \sum_{i = 1}^{m_g} \frac{\EE\left[\left. \gtg{L}(z_i)\right|M_{ {g}} = m_{ {g}}\right]}{m_{ {g}}} \\
= & \sum_{i = 1}^{m_g} \frac{\EE\left[\left. \gtg{L}(z)\right| G = g\right]}{m_{ {g}}}\\
= & \EE\left[\left. \gtg{L}(z)\right| G = g\right]
\end{align}

Combining the above two equations yields
\begin{equation*}
\left.\EE\left[\frac{\sum_{i = 1}^{M_g} \gtg{L}(z_i)}{M_{ {g}}}\right|M_{ {g}} = m_g \right] = \EE\left[\left. \gtg{L}(z)\right| G = g\right] \mathbb{I}[M_{ {g}} \geq 1].
\end{equation*}
Taking the expectation on both sides w.r.t. $M_{{g}}$ yields the lemma.
\end{proof}

\begin{lemma} [Computing Variance of First Order Estimation] Define $\frac{\sum_{i = 1}^{M_g} \gtg{L}(z_i)}{M_g} = 0$ for $M_g = 0$ and recall that $z = (x, g, y)$. 
If $L: \cZ \rightarrow \{0,1\}$, then
\begin{align*}
    \text{Var} \left[ \frac{\sum_{i = 1}^{M_g} \gtg{L}(z_i)}{M_g} \right] \leq 2\Pr(M_g \geq 1)
\end{align*}

\label{lem::apx_var-first-order}
\end{lemma}

\begin{proof}
    By the law of total variance,
\begin{align}
   \twocol{\quad &=} \text{Var}\left[\frac{\sum_{i = 1}^{M_g} \gtg{L}(z_i)}{M_g} \right] \twocol{\\
    &} =\EE\left[ \text{Var} \left[ \left.  \frac{\sum_{i = 1}^{M_g} \gtg{L}(z_i)}{M_g} \right| M_{ {g}} \right] \right] \twocol{\\
    &} + \text{Var}\left[\left.\EE\left[\frac{\sum_{i = 1}^{M_g} \gtg{L}(z_i)}{M_g} \right| M_{ {g}} \right] \right].  
   \label{eq:var00}
\end{align}

We first compute the second term in the RHS of the above equation ~\eqref{eq:var00}.
By Lemma \ref{lem::apx_avg-first-order} we have that:
\begin{align}
\twocol{\nonumber
    &}
    \text{Var}\left[\left.\EE\left[\frac{\sum_{i = 1}^{M_g} \gtg{L}(z_i)}{M_g} \right| M_{ {g}} \right] \right] \twocol{\\}
    &= \text{Var}\left[  \EE\left[\left. \gtg{L}(z)\right| G = g\right] \mathbb{I}[M_{ {g}} \geq 1] \right] \\
    &=  \EE\left[\left. \gtg{L}(z)\right| G = g\right]^2\text{Var}\left[ \mathbb{I}[M_{ {g}} \geq 1] \right] \\
    & \leq  \EE\left[\left. \gtg{L}(z)\right| G = g\right]^2\Pr\left(M_{ {g}} \geq 1\right)\\
    \label{eq:var01}
    & \leq  \Pr\left( M_{ {g}} \geq 1 \right)
\end{align}

Now, let's bound the first term in the first term in the RHS of ~\eqref{eq:var00}. Let's start by bounding $ \text{Var} \left[ \left.  \frac{\sum_{i = 1}^{M_g} \gtg{L}(z_i)}{M_g} \right| M_{ {g}} = m_g \right]$.

If $m_g = 0$ then:
\begin{equation*}
     \text{Var} \left[ \left.  \frac{\sum_{i = 1}^{M_g} \gtg{L}(z_i)}{M_g} \right| M_{ {g}} \right] = 0
\end{equation*}

If $m_g \geq 1$, then by Popoviciu's inequality on variances:

\begin{align*}
  \twocol{  &} \text{Var} \left[ \left.  \frac{\sum_{i = 1}^{M_g} \gtg{L}(z_i)}{M_g} \right| M_{ {g}} \right] \twocol{\\}
    & \leq \frac{1}{4} \Bigg( \max_{z_i} \left( \frac{\sum_{i = 1}^{M_g} \gtg{L}(z_i)}{M_g} \right)  - \min_{z_i} \left( \frac{\sum_{i = 1}^{M_g} \gtg{L}(z_i)}{M_g} \right) \Bigg)^2 \\
    & \leq \frac{1}{4}( 0 - (-1)1)^2 = \frac{1}{4},
\end{align*}
where the maximum is achieved when all $z_i$ are such that $\gtg{L}(z_i) = 1$ and the minimum is achieved when $\gtg{L}(z_i) = 0$ --- note that the bound holds when $M_g \neq 0$.

By the two last equations, we conclude that:
\begin{align}
    & \text{Var} \left[ \left.  \frac{\sum_{i = 1}^{M_g} \gtg{L}(z_i)}{M_g} \right| M_{ {g}} = m_g\right] \leq \frac{1}{4} \mathbb{I}[m_{ {g}} \geq 1]
\end{align}

and by taking expectations on both sides:
\begin{align}
    \EE\left[\text{Var} \left[ \left.  \frac{\sum_{i = 1}^{M_g} \gtg{L}(z_i)}{M_g} \right| M_{ {g}} = m_g\right]\right] \leq \frac{1}{4} \Pr(m_{ {g}} \geq 1)
    \label{eq:var02}
\end{align}

Combining equation \ref{eq:var00}, \ref{eq:var01}, and \ref{eq:var02} we conclude that:
\begin{align}
    \text{Var}\left[\frac{\sum_{i = 1}^{M_g} \gtg{L}(z_i)}{M_g} \right] \leq \left(1 + \frac{1}{4}\right) \Pr(m_{ {g}} \geq 1) \leq 2 \Pr(m_{ {g}} \geq 1)
\end{align}

\end{proof}

\begin{lemma}[Computing Average of Second Order Estimation]  Define $\frac{\sum_{i = 1}^{M_g} L(z_i)}{M_g}\frac{\sum_{i = 1}^{M_g} L(z_i)-1}{M_g-1} = 0$ if $M_g  = 0, 1$, and recall that $z = (x, g, y)$. If $L: \cZ \rightarrow \{0,1\}$, then
\begin{align*}
    \twocol{\quad&}\EE \left[ \frac{\sum_{i = 1}^{M_g} L(z_i)}{M_g}\frac{\sum_{i = 1}^{M_g} L(z_i)-1}{M_g-1}\right] \twocol{\\} 
    = & \EE [L(z) \mid G = g]^2 \Pr(M_g \geq 2)
\end{align*}
\label{lem::apx_avg-second-order}
\end{lemma}

\begin{proof}
    By the conditional law of expectation, 
    \begin{align}
       \twocol{ \quad&}\EE \left[ \frac{\sum_{i = 1}^{M_g} L(z_i)}{M_g}\frac{\sum_{i = 1}^{M_g} L(z_i)-1}{M_g-1}\right] \twocol{\\ }
    =& \EE \left[ \EE \left[ \left. \frac{\sum_{i = 1}^{M_g} L(z_i)}{M_g}\frac{\sum_{i = 1}^{M_g} L(z_i)-1}{M_g-1} \right| M_g = m_g\right]\right] \\
    \end{align}

    If $m_g \in \{0, 1\}$ then: 
    \begin{equation}
        \EE \left[ \left. \frac{\sum_{i = 1}^{M_g} L(z_i)}{M_g}\frac{\sum_{i = 1}^{M_g} L(z_i)-1}{M_g-1} \right| M_g = m_g\right] = 0
        \label{eq:avg_estimator_first_eq}
    \end{equation}

    If $m_g \geq 2$ then: 
    \begin{align}
        \quad & \EE \left[ \left. \frac{\sum_{i = 1}^{M_g} L(z_i)}{M_g}\frac{\sum_{i = 1}^{M_g} L(z_i)-1}{M_g-1} \right| M_g = m_g\right] \\
        =&  \EE \left[ \left. \frac{\sum_{i = 1}^{m_g} L(z_i)}{m_g}\frac{\sum_{i = 1}^{m_g} L(z_i)-1}{m_g-1} \right| M_g = m_g\right] \\ 
        =&  \EE \left[ \left. \frac{ \left(\sum_{i = 1}^{m_g} L(z_i) \right)^2 - \sum_{i = 1}^{m_g} L(z_i)}{m_g(m_g-1)} \right| M_g = m_g\right] \\
        =&  \EE \left[ \left. \frac{ \sum_{i = 1}^{m_g} L(z_i)^2 + \sum_{i \ne j} L(z_i)L(z_j)  - \sum_{i = 1}^{m_g} L(z_i)}{m_g(m_g-1)} \right| M_g = m_g\right] \\
        =&   \frac{ m_g \EE \left[L(z)^2  - L(z)|G = g\right] +m_g(m_g - 1)  \EE \left[L(z)|G = g\right]^2  }{m_g(m_g-1)}  \\
        \label{eq:avg_estimator_second_eq}
        = &  \EE \left[L(z)|G = g\right]^2,
    \end{align}
    where the last inequality comes from the fact that $\EE \left[L(z)^2  - L(z)|G = g\right] = 0$ because $L(z) \in \{0, 1\}$ therefore $L(z)^2 - L(z) = 0$ for all $z$.

    Combining \eqref{eq:avg_estimator_first_eq} and \eqref{eq:avg_estimator_second_eq} yields
    
    \begin{align}
      \twocol{  \quad & }\EE \left[ \left. \frac{\sum_{i = 1}^{M_g} L(z_i)}{M_g}\frac{\sum_{i = 1}^{M_g} L(z_i)-1}{M_g-1} \right| M_g = m_g\right]\twocol{ \\}
        = & \EE \left[L(z)|G = g\right]^2 \mathbb{I}[M_{ {g}} \geq 2].
    \end{align}
    Taking the expectation on both sides w.r.t. $M_{{g}}$ give us the lemma.
\end{proof}

\begin{lemma}[Computing Variance of Second Order Estimation]  Define $\frac{\sum_{i = 1}^{M_g} L(z_i)}{M_g}\frac{\sum_{i = 1}^{M_g} L(z_i)-1}{M_g-1} = 0$ if $M_g  = 0, 1$, and recall that $z = (x, g, y)$. If $L: \cZ \rightarrow \{0, 1\}$, then
    \begin{align*}
    \text{Var} \left[ \frac{\sum_{i = 1}^{M_g} L(z_i)}{M_g}\frac{\sum_{i = 1}^{M_g} L(z_i)-1}{M_g-1}\right] \leq & 2 \Pr(M_g \geq 2)
    \end{align*}
\label{lem::apx_var-second-order}
\end{lemma}

\begin{proof}
    By the law of total variance,
\begin{align}
 \twocol{   \nonumber
   \quad &=} \text{Var}\left[\frac{\sum_{i = 1}^{M_g} L(z_i)}{M_g}\frac{\sum_{i = 1}^{M_g} L(z_i)-1}{M_g-1}\right] \twocol{\\}  
    & =\EE\left[ \text{Var} \left[ \left.  \frac{\sum_{i = 1}^{M_g} L(z_i)}{M_g}\frac{\sum_{i = 1}^{M_g} L(z_i)-1}{M_g-1} \right| M_{ {g}} \right] \right]  \nonumber \\
    & + \text{Var}\left[\left.\EE\left[\frac{\sum_{i = 1}^{M_g} L(z_i)}{M_g}\frac{\sum_{i = 1}^{M_g} L(z_i)-1}{M_g-1} \right| M_{ {g}} \right] \right].
     \label{eq:bdvar-1}
\end{align}

We first compute the second term in the RHS of the above equation ~\eqref{eq:bdvar-1}.
By Lemma \ref{lem::apx_avg-second-order} we have that:
\begin{align}
\nonumber
    &\text{Var}\left[\left.\EE\left[\frac{\sum_{i = 1}^{M_g} L(z_i)}{M_g}\frac{\sum_{i = 1}^{M_g} L(z_i)-1}{M_g-1} \right| M_{ {g}} \right] \right] \\
    \label{eq:bdVar0}
    &= \text{Var}\Bigg[ \Bigg( \frac{\EE \left[L(z)^2  - L(z)|G = g\right]}{(m_g - 1)} + \EE \left[L(z)|G = g\right]^2  \Bigg)\mathbb{I}[m_{ {g}} \geq 2] \Bigg] 
\end{align}

Note that $\EE \left[L(z)^2  - L(z)|G = g\right] = 0$.

Now, we can bound Eq. \ref{eq:bdVar0} by
\begin{align}
\nonumber
       &= \text{Var}\Bigg[ \Bigg( \frac{\EE \left[L(z)^2  - L(z)|G = g\right]}{(m_g - 1)}          \twocol{\\
    &} + \EE \left[L(z)|G = g\right]^2  \Bigg)\mathbb{I}[m_{ {g}} \geq 2] \Bigg] \\
    &  = \text{Var}\left[ \left(  \EE \left[L(z)|G = g\right]^2 \right)\mathbb{I}[m_{ {g}} \geq 2]\right]\\
    & = \EE\left[ \left( \EE \left[L(z)|G = g\right]^2 \right)^2\mathbb{I}[m_{ {g}} \geq 2]\right] \twocol{ \\
    \nonumber
    &} - \EE\left[ \left( \EE \left[L(z)|G = g\right]^2 \right)\mathbb{I}[m_{ {g}} \geq 2]\right]^2\\
    & \leq \EE\left[ \left( \EE \left[L(z)|G = g\right]^2 \right)^2\mathbb{I}[m_{ {g}} \geq 2]\right]\\
    \nonumber
     & \leq \norm{\left( \EE \left[L(z)|G = g\right]^2 \right)^2}_{\infty}\EE\left[\mathbb{I}[m_{ {g}} \geq 2]\right]\\
     \label{eq:bdVar3}
     & \leq \Pr(m_{ {g}} \geq 2)
\end{align}
Where the last inequality comes from the fact that $L \in \{0, 1\}$.

Now, let's bound the first term in the first term in the RHS of ~\eqref{eq:bdvar-1}. Let's start by bounding $ \text{Var} \left[ \left.  \frac{\sum_{i = 1}^{M_g} L(z_i)}{M_g}\frac{\sum_{i = 1}^{M_g} L(z_i)-1}{M_g-1} \right| M_{ {g}} = m_g \right]$.

If $m_g \in \{0, 1\}$ then:
\begin{equation*}
     \text{Var} \left[ \left.  \frac{\sum_{i = 1}^{M_g} L(z_i)}{M_g}\frac{\sum_{i = 1}^{M_g} L(z_i)-1}{M_g-1} \right| M_{ {g}} = m_g\right] = 0
\end{equation*}

If $m_g \geq 2$ then, by Popoviciu's inequality on variances:
\begin{align*}
    \twocol{&} \text{Var} \left[ \left.  \frac{\sum_{i = 1}^{M_g} L(z_i)}{M_g}\frac{\sum_{i = 1}^{M_g} L(z_i)-1}{M_g-1} \right| M_{ {g}} = m_g\right]\twocol{\\}
    & \leq \frac{1}{4} \Bigg( \max_{z_i} \left(\frac{\sum_{i = 1}^{M_g} L(z_i)}{M_g} \frac{\sum_{i = 1}^{M_g} L(z_i)-1}{M_g-1}\right) \\
    & - \min_{z_i} \left(\frac{\sum_{i = 1}^{M_g} L(z_i)}{M_g} \frac{\sum_{i = 1}^{M_g} L(z_i)-1}{M_g-1}\right) \Bigg)^2 \\
    & \leq \frac{1}{4}( 1 - 0)^2 = \frac{1}{4}
\end{align*}

By the two last equations, we conclude that:
\begin{align}
    & \text{Var} \left[ \left.  \frac{\sum_{i = 1}^{M_g} L(z_i)}{M_g}\frac{\sum_{i = 1}^{M_g} L(z_i)-1}{M_g-1} \right| M_{ {g}} = m_g\right] \twocol{\\
    &} \leq \frac{1}{4} \mathbb{I}[m_{ {g}} \geq 2]
\end{align}
By taking averages on both sides, we get:
\begin{align}
    \twocol{\nonumber
&} \EE\left[ \text{Var} \left[ \left.  \frac{\sum_{i = 1}^{M_g} L(z_i)}{M_g}\frac{\sum_{i = 1}^{M_g} L(z_i)-1}{M_g-1} \right| M_{ {g}} \right] \right] \twocol{\\
    & }\leq \frac{1}{4} \Pr(M_{ {g}} \geq 2).    \label{eq::bdVar2}
\end{align}

Bounding the RHS of \eqref{eq:bdvar-1} by using the result in \eqref{eq:bdVar3} and \eqref{eq::bdVar2} we conclude that
\begin{align}
 \twocol{  &} \text{Var}\left[\frac{\sum_{i = 1}^{M_g} L(z_i)}{M_g}\frac{\sum_{i = 1}^{M_g} L(z_i)-1}{M_g-1}\right] \twocol{\\}
    & \leq \frac{5}{4}\Pr(M_{ {g}} \geq 2) \\
    & \leq 2\Pr(M_{ {g}} \geq 2).
\end{align}
\end{proof}

\begin{lemma}
Denote $E_g = \EE \left[L(z)|G = g\right]$, if 
$L: \cZ \rightarrow \{0, 1\}$, then
\begin{align}
\sum_{g \in \gp} w_g E_g^2 - \averagegtg{L}^2  \geq (1 - \alpha) (\texttt{F\_CVaR}_{\alpha}(w))^2. \label{ps-bound-quant}
\end{align}
\label{lem:apx_CVaRUpperBound}
\end{lemma}

\begin{proof}
The proof follows from algebraic manipulation as follows.
\begin{align}
\twocol{\nonumber
     &} (1 - \alpha) \left(\texttt{F\_CVaR}_{\alpha}(w)\right)^2  + \averagegtg{L}^2 \twocol{\\
     \nonumber}
     = & (1 - \alpha)  \left(\max_{\sum w_g \leq 1-\alpha} \sum \frac{ w_g}{(1 - \alpha)} \Delta(g)\right)^2  + \averagegtg{L}^2\\
     \nonumber
     \leq & (1 - \alpha)  \left(\sum_{g \in Q_{\alpha}} \frac{\sqrt{w_g}}{\sqrt{(1 - \alpha)}} \frac{\sqrt{w_g}\Delta(g)}{\sqrt{(1 - \alpha)}}\right)^2  + \averagegtg{L}^2\\
     \label{eq:cauchy}
     \leq & (1 - \alpha)  \left(\sum_{g \in Q_{\alpha}} \frac{w_g}{{(1 - \alpha)}}\right) \left( \sum_{g \in Q_{\alpha}} \frac{{w_g}\Delta(g)^2}{{(1 - \alpha)}}\right)  + \averagegtg{L}^2\\
     \nonumber
     \leq &  \sum_{g \in Q_{\alpha}} {{w_g}\Delta(g)^2} + \averagegtg{L}^2\\
     \nonumber
     = & \sum_{g \in Q_{\alpha}} {{w_g}(E_g - \averagegtg{L})^2} + \averagegtg{L}^2\\
     \nonumber
     \leq & \sum_{g \in \gp} {{w_g}(E_g - \averagegtg{L})^2} + \averagegtg{L}^2\\
     \nonumber
     = & \sum_{g \in \gp} {w_g}E_g^2 - 2\averagegtg{L}\left(\sum_{g \in \gp} {w_g}E_g - \averagegtg{L}\right)\\
     \nonumber
     = & \sum_{g \in \gp} {w_g}E_g^2,
\end{align}
where the inequality in \ref{eq:cauchy} follows by invoking Cauchy-Schwartz inequality.
\end{proof}

\begin{lemma}
\label{lem:main_upper_general} If $L: \cZ \rightarrow \{0, 1\}$ is the quality of service function and $\averagegtg{L}$ be the average quality of service and $w_g \leq 1 - \alpha$ for all $g \in \gp$, then

(a) If $\texttt{F\_CVaR}_{\alpha}(w) = 0$, then 
\[
 \Pr \left(\widehat{F} > (1 - \alpha) \epsilon^2/2 \right)  \leq  \sum_g \frac{64 w_g^2}{(1 - \alpha)^2 \epsilon^4P\left[M_{ {g}} \geq 2\right]} + \frac{128}{(1 - \alpha)^2 \epsilon^4 } \sum_g \frac{w^2_g}{P\left[M_{ {g}} \geq 1\right]}.
\]

Similarly, 

(b) if $\texttt{F\_CVaR}_{\alpha}(w) \geq  \epsilon$, then 
\[
 \Pr\left(\widehat{F} < (1 - \alpha) \epsilon^2/2 \right)  \leq  \sum_g \frac{64 w_g^2}{(1 - \alpha)^2 \epsilon^4P\left[M_{ {g}} \geq 2\right]} + \frac{128}{(1 - \alpha)^2 \epsilon^4 } \sum_g \frac{w^2_g}{P\left[M_{ {g}} \geq 1\right]}.
\]
\end{lemma}
\begin{proof}

Denote $E_g = \EE \left[L(z)|G = g\right]$.

Recall that the estimator is given by
\begin{align}
    \widehat{F} &= \sum_{ {g} \in \gp} \frac{w_g}{P\left[M_{ {s}} \geq 2\right]} \frac{\sum_{i = 1}^{M_g} L(z_i)}{M_g}\frac{(\sum_{i = 1}^{M_g} L(z_i)-1)}{M_g-1} - \left( \sum_{ {g} \in \gp} \frac{w_g}{P\left[M_{ {s}} \geq 1\right]} \frac{\sum_{j = 1}^{M_g} \gtg{L}(z_j)}{M_{ {g}}} \right)^2
\end{align}

We first show the result under the null hypothesis, \revision{i.e., when $\texttt{F\_CVaR}_{\alpha}(w) = 0$. Note that this condition implies that $\left(\sum_{g \in \gp} w_g E_g^2 - \averagegtg{L}^2 \right) = 0$, because as all $w_g \leq 1-\alpha$ we have that if $E_g \neq \averagegtg{L}$ for some $g \in \gp$, then $\Delta(g) > 0$ and $\texttt{F\_CVaR}_{\alpha}(w) \geq \frac{w_g \Delta(g)}{1 - \alpha} > 0$.
}
\begin{align}
&\Pr(\widehat{F} > (1 - \alpha) \epsilon^2/2 ) \\
& = \Pr\left(\widehat{F} - \left(\sum_{g \in \gp} w_g E_g^2 - \averagegtg{L}^2 \right) > (1 - \alpha) \epsilon^2/2 \right) \\
& \leq  \Pr\Bigg(\sum_{ {g} \in \gp} \frac{w_g \sum_{i = 1}^{M_g} L(z_i) (\sum_{i = 1}^{M_g} L(z_i)-1)}{P\left[M_{ {s}} \geq 2\right] M_g (M_g - 1)}  - \sum_{g \in \gp} w_g E_g^2> \frac{(1 - \alpha) \epsilon^2}{4} \Bigg) \nonumber\\
& + \Pr\left(\averagegtg{L}^2 - \left( \sum_{ {g} \in \gp} \frac{w_g}{P\left[M_{ {s}} \geq 1\right]} \frac{\sum_{j = 1}^{M_g} \gtg{L}(z_j)}{M_{ {g}}} \right)^2 > \frac{(1 - \alpha) \epsilon^2}{4}  \right) \\
& \leq  \Pr\Bigg(\sum_{ {g} \in \gp} \frac{w_g \sum_{i = 1}^{M_g} L(z_i) (\sum_{i = 1}^{M_g} L(z_i)-1)}{P\left[M_{ {s}} \geq 2\right] M_g (M_g - 1)} - \sum_{g \in \gp} w_g E_g^2> \frac{(1 - \alpha) \epsilon^2}{4} \Bigg) \nonumber\\
& + \Pr\left(\left( \sum_{ {g} \in \gp} \frac{w_g}{P\left[M_{ {s}} \geq 1\right]} \frac{\sum_{j = 1}^{M_g} \gtg{L}(z_j)}{M_{ {g}}}  \right) < \averagegtg{L} - \frac{(1 - \alpha) \epsilon^2}{8} \right) ,
\end{align}
\revision{
where the last inequality follows by observing that $a^2 - b^2 > c$ implies that $a^2 - c>  b^2$ which implies that $\sqrt{a^2 - c}>  b$ because $b$ is positive.
Then, we conclude by showing that
\begin{align}
     & \sqrt{a^2 - c} \leq a - \frac{c}{2}\\
   \iff &  a^2 - c \leq \left(a - \frac{c}{2}\right)^2 \\
  \iff   & a^2 - c \leq a^2 - ac + \frac{c^4}{4} \\
  \iff   & - c \leq - ac + \frac{c^4}{4} \\
  \iff   & c(a - 1)  \leq \frac{c^4}{4},
\end{align}
the last inequality holds because $a \leq 1$ and $c > 0$.
}

Next, we bound each of the terms on the right-hand side of the above equation using Chebyshev's inequality. 

First, define $\widehat{E}_1$ such that:
\begin{equation}
    \widehat{E}_1 = \sum_{ {g} \in \gp} \frac{w_g \sum_{i = 1}^{M_g} L(z_i) (\sum_{i = 1}^{M_g} L(z_i)-1)}{P\left[M_{ {g}} \geq 2\right] M_g (M_g - 1)} - \sum_{g \in \gp} w_g E_g^2
\end{equation}

By Lemma~\ref{lem::apx_avg-second-order},
\[
\EE[\widehat{E}_1] = 0,
\]

and by Lemma~\ref{lem::apx_var-second-order}
\begin{align*}
\text{Var}(\widehat{E}_1)
& \leq \sum_g  \frac{2 w_g^2}{P\left[M_{ {g}} \geq 2\right]}.
\end{align*}
Hence by Chebyshev's inequality,
\begin{align}
  \twocol{&} \Pr\left(\widehat{E}_1 >  (1 - \alpha) \epsilon^2/4\right) \twocol{\\}
    & \leq \Pr\left(\widehat{E}_1 - \EE[\widehat{E}_1]>  (1 - \alpha) \epsilon^2/4\right)\\
    & \leq \Pr\left(\left| \widehat{E}_1 - \EE[\widehat{E}_1] \right|>  (1 - \alpha) \epsilon^2/4\right)\\
    & \leq \sum_g \frac{32 w_g^2}{(1 - \alpha)^2 \epsilon^4 P\left[M_{ {g}} \geq 2\right]}
\end{align}

Additionally, denote $\widehat{E}_2$ by:
\begin{equation}
    \widehat{E}_2 = \sum_{ {g} \in \gp} \frac{w_g}{P\left[M_{ {g}} \geq 1\right]} \frac{\sum_{j = 1}^{M_g} \gtg{L}(z_j)}{M_{ {g}}}  - \averagegtg{L}
\end{equation}
Similar to the previous argument, we conclude that:
\[
\EE[\widehat{E}_2] = 0,
\]
and by Lemma~\ref{lem::apx_var-first-order}
\begin{align*}
\text{Var}(\widehat{E}_1)
& \leq \sum_g  \frac{2w_g^2}{P\left[M_{ {g}} \geq 1\right]}.
\end{align*}

Hence, by applying Chebyshev's inequality,
\begin{align*}
    \twocol{&}\Pr\left(\widehat{E}_2 < - (1 - \alpha) \epsilon^2/8\right) \twocol{\\}
=&   \Pr\left(\EE[\widehat{E}_2] - \widehat{E}_2 > (1 - \alpha) \epsilon^2/8\right)\\
\leq &   \Pr\left(\left|\EE[\widehat{E}_2] - \widehat{E}_2 \right|> (1 - \alpha) \epsilon^2/8\right)\\
\leq & \frac{128}{(1 - \alpha)^2 \epsilon^4 } \sum_g \frac{w^2_g}{P\left[M_{ {g}} \geq 1\right]}
\end{align*}

Hence,
\begin{align*}
\Pr(\widehat{F} > (1 - \alpha) \epsilon^2/2 ) 
& \leq \sum_g \frac{32 w_g^2}{(1 - \alpha)^2 \epsilon^4P\left[M_{ {g}} \geq 2\right]} + \frac{128}{(1 - \alpha)^2 \epsilon^4 } \sum_g \frac{w^2_g}{P\left[M_{ {g}} \geq 1\right]},
\end{align*}
completing the proof.

Next, we show the result under the hypothesis $\hypothesis_1$, i.e., when $\texttt{F\_CVaR}_{\alpha} \geq \epsilon$.
We are interested in bounding:
\begin{align}
& \Pr(\widehat{F} < (1 - \alpha) \epsilon^2/2 )
\end{align}
From Lemma \ref{lem:apx_CVaRUpperBound}, we have that
\begin{align}
\nonumber
\twocol{&}\Pr(\widehat{F} < (1 - \alpha) \epsilon^2/2 )\twocol{\\
\nonumber}
& \leq \Pr\left(\widehat{F} - \left(\sum_{g \in \gp} w_g E_g^2 - \averagegtg{L}^2 \right) < - (1 - \alpha)\epsilon^2 + (1 - \alpha) \epsilon^2/2 \right) \nonumber\\
\nonumber
&\leq \Pr\left(\widehat{F} - \left(\sum_{g \in \gp} w_g E_g^2 - \averagegtg{L}^2 \right) < - (1 - \alpha)\epsilon^2/2 \right) \\
\label{eq:lem_8_eq_1}
&\leq \Pr\left(-\widehat{F} + \left(\sum_{g \in \gp} w_g E_g^2 - \averagegtg{L}^2 \right) >  (1 - \alpha)\epsilon^2/2 \right)
\end{align}

Using the bound in \eqref{eq:lem_8_eq_1}, and the same procedure we used in for bounding the probability of error when  $\texttt{F\_CVaR}_{\alpha} = 0$, we can show that under $\texttt{F\_CVaR}_{\alpha} \geq \epsilon$
\begin{align*}
\Pr(\widehat{F} > (1 - \alpha) \epsilon^2/2 ) 
& \leq \sum_g \frac{32 w_g^2}{(1 - \alpha)^2 \epsilon^4P\left[M_{ {g}} \geq 2\right]} + \frac{128}{(1 - \alpha)^2 \epsilon^4 } \sum_g \frac{w^2_g}{P\left[M_{ {g}} \geq 1\right]},
\end{align*}

\end{proof}

\WeightedUpperGeneral*

\begin{proof}
    Using Algorithm \ref{alg:hypothesis_test}, we have that:
    \begin{align}
        P_{\text{error}} &= \Pr\left( \left.\widehat{F} \geq \frac{\alpha \epsilon^2}{2} \right| \texttt{F\_CVaR}_{\alpha}(\bw) = 0 \right) \Pr\left( \texttt{F\_CVaR}_{\alpha}(\bw) = 0 \right) \\ 
        &+ \Pr\left( \left.\widehat{F} < \frac{\alpha \epsilon^2}{2} \right| \texttt{F\_CVaR}_{\alpha}(\bw) \geq \epsilon \right)\Pr\left( \texttt{F\_CVaR}_{\alpha}(\bw)\geq \epsilon\right)
        \label{eq:thm_1_bd_1}
    \end{align}
    
    then, by Lemma \ref{lem:main_upper_general}, we can bound the RHS \eqref{eq:thm_1_bd_1} and conclude that
    \begin{align}
        P_{\text{error}} \leq \sum_g \frac{32 w_g^2}{(1 - \alpha)^2 \epsilon^4P\left[M_{ {g}} \geq 2\right]} + \frac{128}{(1 - \alpha)^2 \epsilon^4 } \sum_g \frac{w^2_g}{P\left[M_{ {g}} \geq 1\right]}
        \label{eq:thm_1_bd_2}
    \end{align}

    By using weighted sampling we can lower bound the probabilities $P\left[M_{{g}} \geq 2\right]$ and $P\left[M_{{g}} \geq 1\right]$ as in Lemma \ref{lem:apx_weight_bounds}. 
    Using this lower bound, we can upper bound the RHS of \eqref{eq:thm_1_bd_2} and conclude that
    \begin{align}
        P_{\text{error}} &\leq \sum_g \frac{128 e w_g^2}{(1 - \alpha)^2 \epsilon^4n^2 (v_g)^2} + \frac{128e}{(1 - \alpha)^2 \epsilon^4 } \sum_g \frac{w^2_g}{n v_g}
\end{align}

Hence, we have the desired result
\begin{equation}
     P_{\text{error}}  =  O\left(\sum_g \frac{  w_g^2}{(1 - \alpha)^2 \epsilon^4n^2 (v_g)^2} + \frac{1}{(1 - \alpha)^2 \epsilon^4 } \sum_g \frac{w^2_g}{n v_g}\right).
\end{equation}
\end{proof}

\SamplecomplexityAlgWeighted*
\begin{proof}
The bound on $n$ for general $\bv$ and $\bw$ follows from the error probability in Theorem~\ref{thm:weighted_upper_general}. We now prove the result when $v_g \propto w^{2/3}_g$. From Theorem \ref{thm:weighted_upper_general} we have that 
    \begin{align}
        P_{\text{error}} &\leq \sum_g \frac{128 e w_g^2}{(1 - \alpha)^2 \epsilon^4n^2 (v_g)^2} + \frac{128e}{(1 - \alpha)^2 \epsilon^4 } \sum_g \frac{w^2_g}{n v_g}\\
 & = 128 e \frac{\left(\sum_g w_g^{2/3}\right)^3}{(1 - \alpha)^2 \epsilon^4n^2} + \frac{128e}{(1 - \alpha)^2 \epsilon^4 } \frac{ (\sum_g w_g^{2/3}) \sum_g w^{4/3}_g}{n}\\
 & \leq 128 e \frac{\left(\sum_g w_g^{2/3}\right)^3}{(1 - \alpha)^2 \epsilon^4n^2} + \frac{128e}{(1 - \alpha)^2 \epsilon^4 } \frac{ (\sum_g w_g^{2/3})}{n}
\end{align}

only need to ensure that
\begin{equation}
    \delta \geq 128 e \frac{\left(\sum_g w_g^{2/3}\right)^3}{(1 - \alpha)^2 \epsilon^4n^2} + \frac{128e}{(1 - \alpha)^2 \epsilon^4 } \frac{ (\sum_g w_g^{2/3})}{n}
\end{equation}

Therefore, it is only necessary to use $n$ samples to test for CvaR fairness reliably, where $n$ is such that
\begin{align}
    n &= O \left(\frac{ \left(\sum_g w_g^{2/3}\right)^{3/2}}{{\delta}\epsilon^2(1 - \alpha)} + \frac{ \left(\sum_g w_g^{2/3}\right)}{{\delta}\epsilon^4(1 - \alpha)^2}  \right) \\
    \iff n &= O\left( \frac{ 2^{\frac{3}{2} \log(\sum_g w_g^{2/3})} }{{\delta}\epsilon^2(1 - \alpha)} +  \frac{2^{\log(\sum_g w_g^{2/3})} }{{\delta}\epsilon^4(1 - \alpha)^2}
    \right)\\
    \iff n &= O\left( \frac{ 2^{\frac{1}{2} H_{2/3}(w)} }{{\delta}\epsilon^2(1 - \alpha)} +\frac{ 2^{\frac{1}{3} H_{2/3}(w)} }{{\delta}\epsilon^4(1 - \alpha)^2} \right)
\end{align}
Setting $\delta = 0.01$ yields the result.

\end{proof}

\AttUpperNew*

\begin{proof}
Let $\delta = 0.01$.
    As in the proof of Theorem \ref{thm:weighted_upper_general}, we have that
    \begin{align}
        P_{\text{error}} \leq \sum_g \frac{128 w_g^2}{(1 - \alpha)^2 \epsilon^4P\left[M_{ {g}} \geq 2\right]} + \frac{128}{(1 - \alpha)^2 \epsilon^4 } \sum_g \frac{w^2_g}{P\left[M_{ {g}} \geq 1\right]}
        \label{eq:thm_2_bd_1}
    \end{align}
    Using attribute specific sampling with $\gamma = \frac{n}{2}$ we conclude that $P\left[M_{ {g}} \geq 1\right] = P\left[M_{ {g}} \geq 2\right] = \min(1, \gamma w_g)$.
    Therefore, we can bound the probability of error in \eqref{eq:thm_2_bd_1} as:

    \begin{align}
        P_{\text{error}} & \leq \sum_g \frac{128 w_g^2}{(1 - \alpha)^2 \epsilon^4 \min(1, \gamma w_g)} + \frac{128}{(1 - \alpha)^2 \epsilon^4 } \sum_g \frac{w^2_g}{ \min(1, \gamma w_g)} \\
        &= \sum_g \frac{256 w_g}{(1 - \alpha)^2 \epsilon^4 n} + \frac{256}{(1 - \alpha)^2 \epsilon^4 } \sum_g \frac{w_g}{ n} \\ 
        & \leq \frac{256}{(1 - \alpha)^2 \epsilon^4 n} 
    \end{align}
    Therefore, we conclude that
    \begin{equation}
        P_{\text{error}} = O \left( \frac{1}{\epsilon^4 (1 - \alpha)^2 n}\right).
    \end{equation}
\end{proof}

\SamplecomplexityalgAtt*
\begin{proof}
Let $\delta = 0.01$.    From Theorem \ref{thm:att_upper_new}, we just need to ensure that:
    \begin{equation}
         P_{\text{error}} \leq \frac{256}{(1 - \alpha)^2 \epsilon^4 n} \leq \delta
    \end{equation}
    Therefore, 
    \begin{equation}
         \frac{256}{(1 - \alpha)^2 \epsilon^4 \delta} \leq n
    \end{equation}
Concluding the proof.     
\end{proof}

\section{Converse of Testing for CVaR under weighted sampling}
\label{sec:converse_CVAR_apx}
\MainChiBoundCvar*

\begin{proof}
We will prove the desired result for equal opportunity described in Example \ref{eg:equal_opportunity}.
The result for other fairness metrics, such as statistical parity in Example \ref{eg:statistical_parity}, and equal odds \cite{hardt2016equality}, is analogous.

Let $\bw$ be the uniform distribution over all groups. Note that under this definition $\bv = \bw$ for any $\eta$. 
Let $Q \subset \gp$ be such that $|Q| = \lfloor(1-\alpha)|\gp|\rfloor$.
We define the set $\mathcal{U}(Q)$ to be
\begin{equation}
    \mathcal{U}(Q) \triangleq \{-1, 1\}^{|Q|}.
\end{equation}
$\mathcal{U}(Q)$ is the set of vector with entries equal to $+1$ or $-1$. 
We will use $\mathcal{U}(Q)$ to perturb the distributions of groups $g \in Q$. Denote the group distribution vector $\bw = (w_1, ..., w_{|\gp|})$.
And define
\begin{align}
    P_0(G = g) &= w_g,\\
    P_u(G = g) &= w_g \ \ \forall u \in \mathcal{U}(Q).
\end{align}

Hence, we define the conditional distribution of $\hat{Y}$ given $G$ accordingly with $P_0$ and $P_u$ to be
\begin{align}
\label{eq:null_hyp_dist}
    P_0(\hat{Y} = 1  | G = g, Y = 0) &= \frac{1}{2} \ \ & \ \ \forall g \in \gp \\
\label{eq:hyp_1_dist_nQ}
    P_u(\hat{Y} = 1  | G = g, Y = 0) &= \frac{1}{2} \ \ & \forall g \notin Q\\
\label{eq:hyp_1_dist_Q}
    P_u(\hat{Y} = 1  | G = g, Y = 0) &= \frac{1}{2} + \frac{\tau \epsilon_g u_g}{w_g} & \ \ \forall g \in Q.
\end{align}
Where $u = (u_1, ..., u_{|Q|})$, $\epsilon_g$ is such that
\begin{equation}
    \epsilon_g \triangleq \frac{\epsilon w_g^{2/3}}{ \left( \sum_{g \in Q} w_g^{2/3} \right)},
    \label{eq:def_epsilon_g}
\end{equation}
and $\tau$ is given by
\begin{equation}
    \tau \triangleq \frac{1 - \alpha}{\alpha}.
\end{equation}
Notice that we need to show that the probability distributions are well defined.
To do so, it is only necessary to ensure that $\frac{\tau \epsilon_g}{w_g} \leq \frac{1}{2}$. 
\begin{equation}
    \tau \frac{\epsilon_g}{w_g} = \tau \frac{\epsilon w_g^{-1/3}}{ \left( \sum_{g' \in Q} w^{2/3}_{g'} \right)} = \tau \frac{\epsilon |\gp|^{1/3}}{\lfloor(1-\alpha)|\gp|\rfloor \frac{1}{|\gp|^{1/3}} }  \leq  \tau \frac{2\epsilon}{(1-\alpha)|\gp|^{1/3}} \leq   \frac{2\epsilon}{\alpha|\gp|^{1/3}}\leq \frac{1}{2}.
    \label{eq:prob_well_def_1}
\end{equation}

Once we have shown that $P_0$ and $P_u$ are probability distributions for all $u \in \mathcal{U}(Q)$, let's prove that 
\begin{enumerate}
    \item $\texttt{F\_CVaR}_{\alpha}(P_0, \bw) = 0$
    \item $\texttt{F\_CVaR}_{\alpha}(P_u, \bw) \geq \epsilon$ for all $u \in \mathcal{U}(Q)$.
\end{enumerate}

Recall that CVaR fairness is given by
\begin{equation}
    \texttt{F\_CVaR}_{\alpha}(P, \bw) =   \frac{1}{1 - \alpha}  \max_{\sum_{g \in A \subset \gp} w_g \leq 1 - \alpha} \sum_{g \in A} w_g \Delta(P, g), \\
\end{equation}
where $\Delta(g, P)$ is such that
\begin{equation}
    \Delta(P, g)= \left| P(\hat{Y} = 1 | G = g, Y = 0)  - P(\hat{Y} = 1 | Y = 0)\right|.
\end{equation}

\textbf{Item 1)}: Let's show that $\texttt{F\_CVaR}_{\alpha}(P_0, \bw) = 0$.
Note that for all $g \in \gp$
\begin{equation}
    P_0(\hat{Y} = 1 | Y = 0) = \sum_{g \in \gp} w_g P_0(\hat{Y} = 1 | G = g, Y = 0) = \frac{1}{2} = P_0(\hat{Y} = 1 | G = g, Y = 0),
\end{equation}
therefore, 
\begin{equation}
    \Delta(P_0, g) = 0 \ \ \forall g \in \gp,
\end{equation}
concluding that 
\begin{equation}
    \texttt{F\_CVaR}_{\alpha}(P_0, \bw) =   \frac{1}{1 - \alpha}  \max_{\sum_{g \in A \subset \gp} w_g \leq 1 - \alpha} \sum_{g \in A} w_g \Delta(P_0, g) = 0.
\end{equation}

\textbf{Item 2)}: Let's show that $\texttt{F\_CVaR}_{\alpha}(P_u, \bw) \geq \epsilon$ for all $u \in \mathcal{U}(Q)$.

Note that
\begin{equation}
     P_u(\hat{Y} = 1 | Y = 0) = \sum_{g \in \gp} w_g P_u(\hat{Y} = 1 | G = g, Y = 0) = \frac{1}{2} + \sum_{g \in Q} \tau \epsilon_g u_g.
    \label{eq:no_group_pu}
\end{equation}

Hence, we show that $\Delta(P_u, g')$ for $g' \in Q$ is
\begin{equation}
    \Delta(P_u, g') = \tau \left|  \frac{\epsilon_{g'} u_{g'}}{w_{g'}}  - \sum_{g \in Q}  \epsilon_g u_g \right|.
    \label{eq:delta_per_group_pu}
\end{equation}

Hence we can lower bound the CVaR fairness gap using \eqref{eq:delta_per_group_pu}.

\begin{align}
    \texttt{F\_CVaR}_{\alpha}(P_u, \bw) &=   \frac{1}{1 - \alpha}  \max_{\sum_{g \in A \subset \gp} w_g \leq 1 - \alpha} \sum_{g' \in A} w_{g'} \Delta(P_u, g') \\
    &\geq   \frac{1}{1 - \alpha} \sum_{g \in Q} w_g \Delta(g) \\
    &=   \frac{1}{1 - \alpha} \sum_{g' \in Q} w_{g'} \left|\frac{\epsilon_{g'} \tau u_{g'}}{w_{g'}} - \sum_{g \in Q} \epsilon_g \tau u_g \right|\\
    &=   \frac{\tau }{1 - \alpha} \sum_{{g'} \in Q}\left|{\epsilon_{g'} u_{g'}} - w_{g'}\sum_{g \in Q} \epsilon_g  u_g \right|\\
    & \geq  \frac{\tau }{1 - \alpha} \sum_{{g'} \in Q}{|\epsilon_{g'}u_{g'}|} - w_{g'}\left|\sum_{g \in Q} \epsilon_g  u_g\right| \\
    & \geq  \frac{\tau }{1 - \alpha} \sum_{{g'} \in Q}{\epsilon_{g'} } - w_{g'}\sum_{g \in Q} \epsilon_g  |u_g| \\
    & =  \frac{\tau }{1 - \alpha} \sum_{{g'} \in Q}{\epsilon_{g'}} - w_{g'}\sum_{g \in Q} \epsilon_g \\
    & =  \frac{\tau }{1 - \alpha} \sum_{{g'} \in Q}{\epsilon_{g'}} - w_{g'}\epsilon \\
    & =  \frac{\tau }{1 - \alpha} \epsilon\left(1 - \sum_{{g'} \in Q} w_{g'}\right)\\
    & \geq  \frac{\tau }{1 - \alpha} \epsilon \alpha = \epsilon.
\end{align}

We then conclude that $\texttt{F\_CVaR}_{\alpha}(P_u, \bw) \geq \epsilon$ for all $u \in \mathcal{U}(Q)$ and from item 1 we had that $\texttt{F\_CVaR}_{\alpha}(P_0, \bw) = 0$.

Now, let's show that, using these distributions, we can make the $\chi^2$ distribution be small considering the mixture of distributions.
Our main objective is to upper bound ${\chi}^2\left( \EE_{u}\left[ P_{u}^n \right]  \biggl| \biggl|  P_0^n  \right)$ by a quantity that (i) goes to zero when the number of samples $n$ increase, and (ii) increases when the number of groups $|\gp|$ increases.

To bound $\chi^2$ of a mixture, we will use the Ingster–Suslina method \cite{Ingster03}.
Hence, we assume that each $u$ is chosen at random from a distribution $\pi$.
The method ensures that
\begin{equation}
    {\chi}^2\left( \EE_{u}\left[ P_{u}^n \right]  \biggl| \biggl|  P_0^n  \right) = \EE_{u, u' \sim \pi} \left[ \prod_{i = 1}^{n} \sum_{z_i \in \cZ} \frac{P_u(z_i) P_{u'}(z_i)}{P_0(z_i)} \right] - 1.
    \label{eq:Ingster_inequality}
\end{equation}

Hence, we just need to compute $\sum_{z_i \in \cZ} \frac{P_v(z_i) P_{v'}(z_i)}{P_0(z_i)}$.

\begin{align}
    &\sum_{z_i \in \cZ} \frac{P_u(z_i) P_{u'}(z_i)}{P_0(z_i)} \\
    &= \sum_{\hat{y}, g} \frac{P_u(\hat{Y} = \hat{y}, G = g  | Y = 0 )P_{u'}(\hat{Y} = \hat{y}, G = g  | Y = 0 )}{P_0(\hat{Y} = \hat{y}, G = g  | Y = 0 )}\\
    &= \sum_{\hat{y}, g \in Q} \frac{P_u(\hat{Y} = \hat{y}, G = g  | Y = 1 )P_{u'}(\hat{Y} = \hat{y}, G = g  | Y = 0 )}{P_0(\hat{Y} = \hat{y}, G = g  | Y = 0 )}\\
    &+ \sum_{\hat{y}, g \notin Q} \frac{P_u(\hat{Y} = \hat{y}, G = g  | Y = 0 )P_{u'}(\hat{Y} = \hat{y}, G = g  | Y = 0)}{P_0(\hat{Y} = \hat{y}, G = g  | Y = 0 )}\\
    &= \sum_{g \in Q} \frac{ \left(\frac{w_g}{2} +  \tau u_g \epsilon_g\right)\left(\frac{w_g}{2} + \tau u_g^{'} \epsilon_g\right) + \left(\frac{w_g}{2} - \tau u_g \epsilon_g\right)\left(\frac{w_g}{2} -  \tau u_g^{'}\epsilon_g \right)}{\frac{w_g}{2}} + \sum_{g \notin Q} \frac{2\frac{w_g}{2}\frac{w_g}{2}}{\frac{w_g}{2}}\\
    &= \sum_{g \in Q} \frac{ \left(\frac{w_g^2}{2} + 2\tau^2 \epsilon_g^2 u_g u_g^{'} \right)}{\frac{w_g}{2}} + \sum_{g \notin Q} \frac{2\frac{w_g}{2}\frac{w_g}{2}}{\frac{w_g}{2}}\\
    \label{eq:first_step_chi_div}
    &= \sum_{g \in Q} \frac{ 4\tau^2 \epsilon_g^2 u_g u_g^{'} }{w_g} + \sum_{g \in Q} {w_g} + \sum_{g \notin Q} {w_g}{}\\
    &= \sum_{g \in Q} \frac{ 4\tau^2 \epsilon_g^2 u_g u_g^{'} }{w_g} + 1.
\end{align}

Now, using Ingster–Suslina inequality \eqref{eq:Ingster_inequality} and the bound in \eqref{eq:first_step_chi_div}, we have
\begin{align}
     {\chi}^2\left( \EE_{u}\left[ P_{u}^n \right]  \biggl| \biggl|  P_0^n  \right)  &= \EE_{u, u' \sim \pi} \left[ \prod_{i = 1}^{n} \sum_{z_i \in \cZ} \frac{P_u(z_i) P_{u'}(z_i)}{P_0(z_i)} \right] - 1\\
    &= \EE_{u, u' \sim \pi} \left[ \left(\sum_{g \in Q} \frac{ 4\tau^2 \epsilon_g^2 u_g u_g^{'} }{w_g} + 1 \right)^n \right] - 1 \\
    &\leq \EE_{u, u' \sim \pi} \left[ e^{n\sum_{g \in Q} \frac{ 4\tau^2 \epsilon_g^2 u_g u_g^{'}}{w_g}} \right] - 1 \\
    \label{eq:subgausian_v}
    &\leq \EE_{u \sim \pi} \EE_{u' \sim \pi} \left[ e^{n\sum_{g \in Q} \frac{ 4\tau^2 \epsilon_g^2 u_g u_g^{'}}{w_g}} \right] - 1.
\end{align}

We note that $u$ and $u'$ are $1$-subgasian, hence, we can bound \eqref{eq:subgausian_v} by
\begin{align}
    \label{eq:ingster_suslina}
    \EE_{u \sim \pi} \EE_{u' \sim \pi} \left[ e^{n\sum_{g \in Q} \frac{ 4\tau^2 \epsilon_g^2 u_g u_g^{'}}{w_g}} \right] - 1 &= \EE_{u \sim \pi} \EE_{u' \sim \pi} \left[\prod_{g \in Q} e^{n \frac{ 4\tau^2 \epsilon_g^2 u_g u_g^{'}}{w_g}} \right] - 1 \\
    \label{eq:bound_subgausian}
    &\leq \EE_{u \sim \pi}   \left[e^{n^2 \sum_{g \in Q} \frac{  16\tau^4 \epsilon_g^4 u_g^2}{w_g^2}} \right] - 1 \\
\label{eq:bound_after_subgausian}
    &\leq e^{n^2 \sum_{g \in Q} \frac{16\tau^4 \epsilon_g^4}{w_g^2}} - 1,
\end{align}
\revision{
where \eqref{eq:ingster_suslina} comes from \eqref{eq:subgausian_v}, \eqref{eq:bound_subgausian} comes from the fact that $u'$ is $1$-subgausian, and \eqref{eq:bound_subgausian} comes from the fact that $|u_g| = 1$.
}

Now, by plugging in $\epsilon_g$ \eqref{eq:def_epsilon_g} in \eqref{eq:bound_after_subgausian} we conclude that
\begin{align}
    \EE_{u \sim \pi} \EE_{u' \sim \pi} \left[ e^{n\sum_{g \in Q} \frac{ 4\tau^2 \epsilon_g^2 u_g u_g^{'}}{w_g}} \right] - 1
    &\leq e^{n^2 \sum_{g \in Q} \frac{16\tau^4 \epsilon_g^4}{w_g^2}} - 1 \leq e^{ \frac{ 16\tau^4 n^2 \epsilon^4}{\left(\sum_{g \in Q} w_g^{2/3} \right)^{3} }} - 1
    \label{eq:last_inequality}
\end{align}

But, from the definition of $Q$, i.e., $\sum_{g \in Q} w_g^{2/3} = \lfloor(1-\alpha)|\gp|\rfloor \frac{1}{|\gp|^{1/3}}  \geq \frac{3(1 - \alpha)}{8} \sum_{g \in \gp} w_g^{2/3}$.
Hence, plugging \eqref{eq:last_inequality} in \eqref{eq:last_inequality} we have that
\begin{align}
    \EE_{u \sim \pi} \EE_{u' \sim \pi} \left[ e^{n\sum_{g \in Q} \frac{ 4\tau^2 \epsilon_g^2 u_g u_g^{'}}{w_g}} \right] - 1
    &\leq e^{n^2 \sum_{g \in Q} \frac{16\tau^4 \epsilon_g^4}{w_g^2}} - 1 \leq e^{ \frac{ 1024\tau^4 n^2 \epsilon^4}{\left( (1 - \alpha) \sum_{g \in \gp} w_g^{2/3} \right)^{3} }} - 1.
    \label{eq:combining_everything}
\end{align}

Concluding the proof of the lemma by combining the inequality in \eqref{eq:subgausian_v} with \eqref{eq:combining_everything} and getting
\begin{align}
    {\chi}^2\left( \EE_{u}\left[ P_{u}^n \right]  \biggl| \biggl|  P_0^n  \right)  &\leq e^{ \frac{ 1024\tau^4 n^2 \epsilon^4}{\left( (1 - \alpha) \sum_{g \in \gp} w_g^{2/3} \right)^{3} }} - 1 \\
    &\leq e^{ \frac{ 1024 (1 - \alpha) n^2 \epsilon^4}{\alpha^{4} \left( \sum_{g \in \gp} w_g^{2/3} \right)^{3} }} - 1. 
\end{align}
\end{proof}

\ConverseCVar*

\begin{proof}
    From Lemma \ref{lem:main_chi_bound_cvar} we know that there exist a sequence of distributions $P_0 \in \hypregions_0(\epsilon)$ and $ \{P_{u}\}_{u \in \mathcal{U}} \in \hypregions_1(\epsilon)$ indexed by elements of a set $\mathcal{U}$ such that if $u \in \mathcal{U}$ we have that
    \begin{equation}
     {\chi}^2\left( \EE_{u}\left[ P_{u}^n \right]  \biggl| \biggl|  P_0^n  \right)
        \leq e^{ \frac{ 1024 (1 - \alpha) n^2 \epsilon^4}{\alpha^{4} |\gp| }} - 1.
        \label{eq:recal_lemma_chi_bound}
    \end{equation}

    On the other hand, by the generalized Le Cam lemma \cite{Polyanskiy19}, we have that
    \begin{align}
        2\inf_{\psi} P_{\text{error}} & \geq 1 - \text{TV}\left( \EE_{u}\left[ P_{u}^n \right]  \biggl| \biggl|  P_0^n  \right) \\
        \label{eq:prob_error_to_chi}
        & \geq 1- \frac{1}{2} \sqrt{{\chi}^2\left( \EE_{u}\left[ P_{u}^n \right]  \biggl| \biggl|  P_0^n  \right)}
    \end{align}

    Plugging \eqref{eq:recal_lemma_chi_bound} in \eqref{eq:prob_error_to_chi}, we have that:
    \begin{equation}
        2\inf_{\psi} P_{\text{error}} \geq 1 - \frac{1}{2} \sqrt{e^{ \frac{ 1024 (1 - \alpha) n^2 \epsilon^4}{\alpha^{4} |\gp|  }} - 1}.
    \end{equation}
    Therefore, if the probability of error is smaller than $0.01$ we need to ensure that: 
    \begin{equation}
       0.02 \geq  1 - \frac{1}{2} \sqrt{e^{ \frac{ 1024 (1 - \alpha) n^2 \epsilon^4}{\alpha^{4}|\gp|  }} - 1}.
    \end{equation}    
    From where we conclude that
    \begin{align}
        n &= \Omega\left( \frac{\alpha^2 \sqrt{|\gp| } }{(1 - \alpha)^{1/2} \epsilon^2} \right) ,
    \end{align}
    which is the desired result.
\end{proof}

\secondrev{
\begin{lemma}[$ \texttt{F\_CVaR}_{\alpha}$ vs. $\text{CVaR}_{\beta}$] Let $\sum_{g \in Q_{\alpha}} w_g = 1- \alpha$, and any two groups $g \neq g'$ are such that $\Delta(g) \neq \Delta(g')$. 
Then, we can define $\beta = 1 - \alpha$ and conclude that 
\begin{equation}
    \text{CVaR}_{\beta}(\Delta(g)) = \texttt{F\_CVaR}_{\alpha}(\bw; P, L, \averagegtg{L})
\end{equation}
\label{apx_cvar_fcvar}
\end{lemma}

\begin{proof}
First, notice that $g \in Q_{\alpha}$ if and only if $\Delta(g) \geq q$
Hence, taking $\beta =  1 - \alpha $ we conclude that $q_{\beta} = \min_{g\in Q_{\alpha}} \Delta(g)$, then

\begin{align}
    \text{CVaR}_{\beta}(\Delta(g)) &= \EE[\Delta(g) | \Delta(g) \geq q_{\beta}] \\
    &= \frac{\sum_{g; \Delta(g) \geq \beta}w_g\Delta(g)}{\sum_{g; \Delta(g) \geq \beta}w_g}\\
    &= \frac{\sum_{g \in Q_{\alpha}} w_g\Delta(g)}{\sum_{g \in Q_{\alpha}}w_g}\\
    &= \frac{\sum_{g \in Q_{\alpha}} w_g\Delta(g)}{1 - \alpha}\\
    &=  \texttt{F\_CVaR}_{\alpha}(\bw; P, L, \averagegtg{L}).
\end{align}
    
\end{proof}
}
\end{document}